\documentclass[main]{article}

 \PassOptionsToPackage{numbers, compress}{natbib}
\usepackage[main, final]{neurips_2026}
 \usepackage{orcidlink}
 \usepackage{algorithm}
 \usepackage{algpseudocode}
 \usepackage{amsmath, amssymb, amsthm}

\theoremstyle{plain}
\newtheorem{theorem}{Theorem}[section]

\newtheorem{lemma}[theorem]{Lemma}

\theoremstyle{definition}

\newtheorem{assumption}[theorem]{Assumption}
\theoremstyle{remark}
\newtheorem{remark}[theorem]{Remark}

\usepackage[utf8]{inputenc} 
\usepackage[T1]{fontenc}    
\usepackage{hyperref}       
\usepackage{url}            
\usepackage{booktabs}       
\usepackage{amsfonts}       
\usepackage{nicefrac}       
\usepackage{microtype}      
\usepackage{xcolor}         
\usepackage{empheq}
\usepackage[most]{tcolorbox}
\usepackage{wrapfig}

\usepackage{multirow}
\usepackage{cancel}

\usepackage{subcaption}
\usepackage{soul}

\DeclareMathOperator{\Var}{\mathbb{V}ar}

\title{Robust-by-Design Distributional Learning from Contaminated Samples}

\author{%
  Nevena~Gligi\'c$^{1,2}$ \qquad Arya~Farahi$^{1,2}$ \\
  $^{1}$Department of Statistics and Data Science, The University of Texas at Austin, USA \\
  $^{2}$The NSF-Simons AI Institute for Cosmic Origins, USA \\
  \texttt{\{nevena.gligic, arya.farahi\}@utexas.edu}
}

\begin{document}

\maketitle

\begin{abstract}
Modern data analysis pipelines increasingly rely on large datasets assembled through retrieval, scraping, automated filtering, and post-hoc curation, where the observed sample is often a contaminated mixture rather than the target population. Standard distributional objectives either ignore contamination and converge to the wrong target, or retain only a small trusted subset and discard most of the data. We study supervised and unsupervised learning when a clean reference sample is compared with a pooled sample from $r=(1-\epsilon)q+\epsilon z$, where $q$ is the latent target distribution and $z$ is an unknown contaminant. Motivated by settings with noisy membership scores for all observations and exact labels for a small audited subset, we propose contamination-corrected maximum mean discrepancy (CC--MMD), a design-based estimator that combines proxy scores with audited residual corrections to recover the oracle MMD that would have been computed from the latent clean sample. CC--MMD requires no calibration assumption on the proxy scores and no structural assumptions on the contaminant beyond standard kernel regularity. We prove finite-population design-unbiasedness of the corrected kernel sums, establish asymptotic normality of the audit correction, and show consistency for the target discrepancy $\mathrm{MMD}^2(p,q;k)$. Empirically, CC--MMD tracks the oracle under increasing contamination, improves parameter recovery and prediction in contaminated regression, and enables generative models trained on heavily corrupted mixtures to recover the clean target distribution. These results show that contaminated data need not be discarded or treated as the estimand. With sparse trusted labels and abundant noisy supervision, distributional objectives can be made robust by design.
\end{abstract}

\vspace{-5mm}
\section{Introduction}
\vspace{-2mm}

Ancient DNA research was reshaped by the problem of contamination \cite{hofreiter2001ancient, allentoft2013recovering, nobelprize2022medicineadvanced}. Early attempts to recover DNA from ancient human remains repeatedly encountered modern human DNA introduced during excavation, handling, and laboratory processing, and later work showed that such contamination can bias downstream genetic analyses toward the ancestry of the contaminating source \cite{racimo2016joint, peyregne2020authentict}. Consequently, paleogenomics adopted authentication protocols, clean-room standards, and explicit \textit{contamination-estimation methods} as part of the inferential pipeline \cite{orlando2021ancient}. Recent work continues to show that contamination can distort downstream procedures such as genotype imputation and PCA-based population-genetic analysis \cite{garrido2024assessing}. Modern machine learning confronts an analogous structural problem at web scale: datasets built through scraping, retrieval, automated filtering, and post-hoc curation often mix target examples with observations drawn from the wrong population, whether as out-of-domain documents in a web corpus \cite{dodge2021documenting, kreutzer2022quality}, false positives in a clinical cohort \cite{rao2025cohortdiagnostics}, or interlopers in an astronomical catalog \cite{farahi2016galaxy, pullen2016interloper}. In each case, the observed sample is not simply the population of interest but a contaminated mixture, and procedures applied naively to that sample risk learning or inferring the wrong distribution.

We study distribution matching under mixture contamination \cite{huber1965robust, huber1992robust}. The observed pool follows
\[
r = (1-\epsilon)q + \epsilon z,
\]
where $q$ denotes the target distribution and $z$ an unknown contaminating distribution. When procedures are applied directly to the observed pool target $r$, discrepancy estimation, parameter recovery, and generative learning can converge to the wrong objective even with very large samples.

Many data analysis pipelines already produce partial information about sample membership \cite{ratner2017snorkel, banda2018advances, davis2023hetdex, huang2024alchemist, wang2024human}. Automated systems assign classifier scores, retrieval scores, or heuristic quality flags to every example, and a smaller subset can often be audited through expert review or higher-fidelity measurement. This yields a common regime of abundant noisy supervision and sparse trusted labels. Existing practice often takes one of two unsatisfactory routes. Either most of the data is discarded, and only the audited subset is used, or the full contaminated pool is used, and bias toward the mixture accepted. For instance, see signal-to-noise cuts in spectroscopic surveys \cite{mukae2026enhancing}, quality-control pipelines in genomics \cite{hemstrom2024next, sealock2025tutorial}, and threshold-based event selection in high-energy physics \cite{hance2025learning}. Neither route utilizes the full potential of available information. 

This gap is especially visible in distribution-based learning and inference applications. Kernel discrepancies and related distribution-matching objectives are used to compare samples \cite{gretton2012kernel, gretton2006kernel}, fit implicit models \cite{briol2019statistical}, estimate parameters \cite{bharti2023optimally}, and train generative models \cite{li2015generative, li2017mmd}. Under contamination, naive objectives absorb the contaminant component and push the analysis toward the wrong population. Our experiments (Section~\ref{sec:experiments}) demonstrate this behavior: contamination biases discrepancy estimation, distorts regression-style parameter recovery, and degrades generative modeling when standard objectives are optimized on the contaminated pool.

We fill this gap by constructing a contamination-corrected distributional objective that combines proxy membership scores for all observations with a design-based residual correction from a small audited subset. In the kernel setting studied here, this construction recovers the oracle kernel sums in expectation under the audit design, does not require calibrated proxy scores, and supports learning objectives for inferential, predictive, and generative models. These results show that large contaminated datasets can still support valid inference and faithful generation when the learning objective targets the latent clean component rather than the observed mixture. 

\section{Data Contamination Problem and the Target Statistic}
\label{sec:contamination_problem}

This section formalizes maximum mean discrepancy (MMD) estimation under contamination, where one sample is clean, and the other is a mixture of the target distribution and an unknown contaminant distribution. Because the analyst does not know which pooled observations are genuine target samples, they instead observe noisy membership scores for all points and exact membership labels for a small uniformly random subset.

\paragraph{Problem Setup.} Suppose H\"uber’s contamination setup \citep{huber1992robust, chen2016general}. See Appendix for illustrative, real-world examples. Let $\mathcal{X}$ denote the observation space and let $k:\mathcal{X}\times\mathcal{X}\to\mathbb{R}$ be a bounded positive definite kernel. We observe two samples: a clean reference sample
\begin{equation}
X_1,\dots,X_n \overset{iid}{\sim} p,    
\end{equation}
and a pooled sample
\begin{equation}
W_1,\dots,W_N \overset{iid}{\sim} r, 
\qquad 
r=(1-\epsilon)q+\epsilon z,\quad \epsilon\in[0,1),
\end{equation}
where $q$ is the target distribution of interest and $z$ is an unknown nuisance distribution. The pooled sample mixes valid target points with contaminants. 
To formalize this, each pooled point $W_j$ has an unobserved membership indicator $I_j\in\{0,1\}$, where $I_j=1$ means that $W_j$ is genuinely drawn from target $q$, and $I_j=0$ means that it comes from distribution $z$. We do not observe $I_j$ directly for most points. Instead, we observe a proxy score $\widetilde I_j\in[0,1]$ (e.g., confidence score of a classifier), intended to approximate $I_j$; the special case $\widetilde I_j\in\{0,1\}$ corresponds to hard classification.
We make no calibration assumption on $\widetilde I_j$. In particular, we do not assume $\widetilde I(w)=\mathbb{P}(I=1\mid W=w)$. Such posterior calibration is difficult to guarantee in practice \citep{hebert2018multicalibration, vashistha2025mathcali, collina2026optimal}. The true conditional law of $I$ given $W$ is unknown, proxy scores are typically produced by learned models fitted to finite, possibly shifted data, and empirical diagnostics can only assess calibration approximately on limited labeled samples. In high-dimensional settings, verifying exact conditional calibration is generally infeasible. Our setup, therefore, treats $\widetilde I_j$ only as an imperfect proxy, not as a trusted posterior probability.

To correct this proxy, we assume access to a small audited subset $L\subset\{1,\dots,N\}$, with $|L|=m\ll N$, for which the true labels $I_j$ are observed without error. We model $L$ as a simple random sample without replacement from the pooled indices, so that $\Pr(j\in L)=\frac{m}{N}$ and $\Pr(j,\ell\in L)=\frac{m(m-1)}{N(N-1)},\quad j\neq \ell$. Throughout this section, expectations are taken with respect to the random draw of $L$, conditional on the realized pooled sample and proxy scores. On labeled points, we define the residual $\Delta_j := I_j-\widetilde I_j$. The quantity $\Delta_j$ measures the local error of the proxy score on audited examples and will be used to debias estimators constructed from $\widetilde I_j$.

In this work, we treat the noisy proxy labels or scores $\tilde I$, the audited labels $I$, and the analysis sample as fixed outputs of an upstream data collection and processing pipeline. The downstream analyst may lack the access, authority, labels, computational resources, or time required to retrain or redesign that pipeline. Our goal is therefore to enable statistically valid downstream learning and inference from these fixed outputs when only a limited audited subset is available. Appendix~\ref{app:additional-results} discusses this setup further.

\paragraph{Target Quantity and Oracle Estimator.} Our goal is to estimate the squared MMD between the clean reference distribution $p$ and the target distribution $q$,
\begin{equation}
\label{eq:mmd_target}
\mathrm{MMD}^2(p,q;k) = \mathbb{E}[k(X,X')] + \mathbb{E}[k(Y,Y')] - 2\,\mathbb{E}[k(X,Y)],
\end{equation}
where $X,X'\overset{iid}{\sim} p$ and $Y,Y'\overset{iid}{\sim} q$ are independent.

If the true membership indicators $I_1,\dots,I_N$ were observed for the pooled sample $W_1,\dots,W_N$, we could estimate the $q$-terms by restricting to the latent target points, equivalently by weighting each $W_j$ by $I_j$. Using ordered pairs with diagonals excluded, define
\begin{align}
\label{eq:oracle_sums}
A_{pp}^\star := \sum_{i\neq i'} k(X_i,X_{i'}),\quad A_{qq}^\star :&= \sum_{j\neq j'} I_j I_{j'}\,k(W_j,W_{j'}),\quad A_{pq}^\star := \sum_{i=1}^n \sum_{j=1}^N I_j\,k(X_i,W_j), \nonumber\\
B_{pp}^\star := n(n-1), \quad B_{qq}^\star :&= \sum_{j\neq j'} I_j I_{j'}, \quad B_{pq}^\star := n\sum_{j=1}^N I_j.
\end{align}
The corresponding oracle estimator is
\begin{equation}
\label{eq:mmd_oracle}
\widehat{\mathrm{MMD}}^2_\star = \frac{A_{pp}^\star}{B_{pp}^\star} + \frac{A_{qq}^\star}{B_{qq}^\star} - 2\,\frac{A_{pq}^\star}{B_{pq}^\star}.
\end{equation}
This is the oracle estimator we aim to estimate. Specifically, our aim is to recover it, in expectation over the audit design, using only the proxy scores $\widetilde I_j$ and the audited subset $L$.

\paragraph{Contamination-Corrected MMD Estimator.} We treat $\widetilde I_j$ as a noisy proxy for the unobserved indicator $I_j$ and use the audited subset to correct its error. Let
\[
\Delta_j := I_j-\widetilde I_j, \qquad j\in L,
\]
denote the observed residual on audited points. Automatically, we get algebraic decompositions
\begin{equation}
\label{eq:I_identity}
I_j = \widetilde I_j + \Delta_j,
\qquad
I_j I_{j'} = \widetilde I_j\widetilde I_{j'} + \Delta_j\widetilde I_{j'} + \widetilde I_j\Delta_{j'} + \Delta_j\Delta_{j'},
\quad j\neq j'.
\end{equation}
We estimate the residual terms by Horvitz--Thompson reweighting under the simple-random-sampling design for $L$ \citep{horvitz1952generalization}. For the $qq$ term, define
\begin{align}
\label{eq:Aqq_hat}
\widehat A_{qq}
:=
&\sum_{j\neq j'} \widetilde I_j \widetilde I_{j'}\,k(W_j,W_{j'})
+
\frac{N}{m}\sum_{j\in L}\sum_{j'\neq j} \Delta_j \widetilde I_{j'}\,k(W_j,W_{j'})
+
\frac{N}{m}\sum_{j'\in L}\sum_{j\neq j'} \widetilde I_j \Delta_{j'}\,k(W_j,W_{j'}) \nonumber\\
&\quad+
\frac{N(N-1)}{m(m-1)}
\sum_{\substack{j,j'\in L\\ j\neq j'}}
\Delta_j \Delta_{j'}\,k(W_j,W_{j'}).
\end{align}
For the $pq$ term, define
\begin{align}
\label{eq:Apq_hat}
\widehat A_{pq} := \sum_{i=1}^n\sum_{j=1}^N \widetilde I_j\,k(X_i,W_j) + \frac{N}{m}\sum_{i=1}^n\sum_{j\in L} \Delta_j\,k(X_i,W_j).
\end{align}
Let $\widehat B_{qq}$ and $\widehat B_{pq}$ be defined by replacing $k(W_j,W_{j'})$ with $1$ throughout. See Equation~\eqref{eq:Bpq_hat} in Supplementary Materials. The $pp$ component is fully observed and requires no correction. We then define the contamination-corrected estimator by plugging these corrected sums into the oracle ratio form: \vspace{-2mm}
\begin{tcolorbox}[
colframe=teal!90,
colback=teal!5]
\vspace{-2mm}
\begin{equation}
\label{eq:mmd_cc}
\widehat{\mathrm{MMD}}^2_{\mathrm{CC}} = \frac{A_{pp}^\star}{B_{pp}^\star} + \frac{\widehat A_{qq}}{\widehat B_{qq}} - 2\,\frac{\widehat A_{pq}}{\widehat B_{pq}}.    
\end{equation}
\vspace{-4mm}
\end{tcolorbox}\vspace{-1mm}
By construction, the correction leaves the directly observed $pp$ term unchanged and adjusts the $qq$ and $pq$ terms using the audited residuals. See Algorithm~\ref{alg:pp_mmd_contam} in Supplementary Materials.

\paragraph{Design-based guarantees.} We show that, under the simple random labeling design, the corrected \emph{unnormalized} kernel sums are unbiased for their oracle counterparts. This design-based recovery is analogous to the unbiased recovery of pair counts in prediction-powered inference \citep{angelopoulos2023prediction} and related Horvitz--Thompson corrections for spatial and point-process U-statistics \citep{farahi2026two}.

\begin{theorem}
\label{thm:unbiased_sums}
Let $m\!\ge\!2$ and $N$ be finite. Under Assumption~\ref{ass:srs}, conditioned on the realized reference sample $\{X_i\}_{i=1}^n$, we have $\mathbb{E}_L[\widehat A_{qq}]\!=\!A_{qq}^\star$,
$\mathbb{E}_L[\widehat B_{qq}]\!=\!B_{qq}^\star$,
$\mathbb{E}_L[\widehat A_{pq}]\!=\!A_{pq}^\star$, and
$\mathbb{E}_L[\widehat B_{pq}]\!=\!B_{pq}^\star$.
\end{theorem}

Theorem~\ref{thm:unbiased_sums} states that the corrected \emph{unnormalized} kernel sums match the oracle kernel sums \emph{in expectation over the labeled subset}, with no assumptions on the quality of $\widetilde I$.
Ratios of unbiased estimators are generally biased; hence, the ratio estimator Equation~\eqref{eq:mmd_cc} is not necessarily unbiased. The estimator requires kernel sums over $X$--$X$, $W$--$W$, and $X$--$W$ interactions, with additional terms restricted to $L$. These sums can be computed exactly (via kernel matrices) or approximately (via mini-batching or random feature maps), without changing the algebra of the correction.

\section{Asymptotic and Finite Sample Guarantees}
\label{sec:theory_pp_mmd}

We now relate the CC-MMD estimator to both the oracle finite-sample estimator and the population target $\mathrm{MMD}^2(p,q;k)$. The analysis has two layers. First, we show that the audit correction is design-unbiased for the oracle kernel sums, conditional on the realized pooled sample. Second, we show that the oracle estimator itself is consistent for, and asymptotically normal around, the population quantity $\mathrm{MMD}^2(p,q;k)$. A key takeaway is that no structural assumptions on the contaminant distribution $z$ are needed when the kernel $k$ is bounded \footnote{In this case, boundedness alone controls the relevant kernel terms under the mixture $r=(1-\epsilon)q+\epsilon z$. When $k$ is unbounded, the required regularity enters only through moment conditions under the mixture distribution.}.

These results rely on two distinct notions of unbiasedness. The first is \emph{design-unbiasedness}: conditional on the realized finite population $\{(W_j,I_j,\widetilde I_j)\}_{j=1}^N$. This justifies CC-MMD as an unbiased recovery of the oracle estimator under the labeling design. The second is \emph{statistical unbiasedness}: under the mixture-with-indicator model, the oracle estimator targets the population discrepancy $\mathrm{MMD}^2(p,q;k)$. The first statement is finite-population and design-based; the second links the finite-sample oracle to the underlying population parameter. We formalize these two steps next.

Theorem~\ref{thm:unbiased_sums} shows that the corrected kernel sums are design-unbiased for their oracle counterparts. This is the key first step, but it does not by itself imply that the ratio-form CC-MMD estimator is unbiased. To control the estimator's bias, we use a deterministic ratio expansion. This isolates the leading contribution of numerator and denominator errors and shows that the discrepancy between CC-MMD and the oracle estimator is controlled.

\begin{lemma}[Ratio linearization]
\label{lem:ratio_linearization}
Let $A>0$, $B>0$, and let $\widehat A,\widehat B$ be random variables such that $\widehat B>0$ almost surely. Then
\begin{equation}
\label{eq:ratio_identity}
\frac{\widehat A}{\widehat B} - \frac{A}{B}
=
\frac{\widehat A-A}{B}
-
\frac{A}{B^2}(\widehat B-B)
+
R,
\end{equation}
where
\begin{equation}
\label{eq:ratio_remainder_exact}
R
= -\frac{(\widehat A-A)(\widehat B-B)}{B\widehat B} + \frac{A(\widehat B-B)^2}{B^2\widehat B},
\end{equation}
and therefore
\begin{equation}
\label{eq:ratio_remainder_bound}
|R|
\le \frac{|\widehat A-A|\,|\widehat B-B|}{B\widehat B} + \frac{A}{B^2\widehat B}\,|\widehat B-B|^2.
\end{equation}
\end{lemma}

We note that each nontrivial component of CC-MMD, $\frac{\widehat A_{qq}}{\widehat B_{qq}}$ and
$\frac{\widehat A_{pq}}{\widehat B_{pq}}$, is a ratio estimator. Lemma~\ref{lem:ratio_linearization} shows that once the corrected sums and denominators concentrate around their oracle counterparts, the induced ratio error is second-order up to the explicit remainder term. In particular, combined with Theorem~\ref{thm:unbiased_sums}, it reduces the analysis of CC-MMD to controlling the finite-population fluctuations of the corrected sums and guarantees that the oracle denominators remain nondegenerate.

\paragraph{Design-based Asymptotic Normality and CLT.} The following result complements Theorem~\ref{thm:unbiased_sums}. Design-unbiasedness shows that the correction recovers the oracle sums in expectation, while the central limit theorem below quantifies the scale of the remaining design fluctuation. With Lemma~\ref{lem:ratio_linearization}, it implies that the discrepancy between CC-MMD and the oracle ratio estimator is asymptotically Gaussian at the $\sqrt m$ scale, provided the denominators remain nondegenerate.

\begin{theorem}[CLT]
\label{thm:design_clt_full}
Under Assumptions~\ref{ass:kernel}--\ref{ass:pair_clt}, conditionally on the realized samples,
\begin{equation}
\sqrt m\,
\Big(
\frac{\widehat A_{pq}-A_{pq}^\star}{nN},\;
\frac{\widehat B_{pq}-B_{pq}^\star}{nN},\;
\frac{\widehat A_{qq}-A_{qq}^\star}{N^2},\;
\frac{\widehat B_{qq}-B_{qq}^\star}{N^2}
\Big)
\Rightarrow
\mathcal{N}(0,\Sigma_{\mathrm{design}})
\end{equation}
for some finite covariance matrix $\Sigma_{\mathrm{design}}$. If Assumption~\ref{ass:denom_nondeg} also holds, the difference between $\widehat{\mathrm{MMD}}^2_{\mathrm{CC}}$ and the oracle estimator is asymptotically Gaussian at rate $\sqrt m$.
\end{theorem}

The theorem shows that, with a proportional audit budget $m/N\to\rho>0$, the audit correction does not introduce a slower asymptotic scale than the underlying empirical MMD fluctuations. A complete discussion, including separate CLTs for the linear $pq$ terms and quadratic $qq$ terms, the decomposition of the sampled-pair correction, and full proofs, is deferred to Appendix~\ref{app:design_clt_proofs}. We now transfer the design-based CLT for the corrected kernel sums into a corresponding result for the ratio-form CC-MMD estimator in Equation~\eqref{eq:mmd_cc}. Specifically, the result below shows that the same $\sqrt m$ behavior carries over to the full ratio estimator, provided the denominators remain nondegenerate.

Recall $B_{qq}^\star\!\!=\!\!\sum_{j\neq j'} I_jI_{j'}$ and $B_{pq}^\star\!=\!n\sum_{j=1}^N I_j$. Under the mixture-with-indicator model, the number of genuine target points in the pool is $N_q^\star\!:=\!\sum_{j=1}^N I_j$, so that $B_{qq}^\star\!=\!N_q^\star(N_q^\star-1)$ and $B_{pq}^\star\!=\!nN_q^\star$. Thus, if the pool contains a nonvanishing fraction of genuine $q$-points, both oracle denominators are of the natural orders $N^2$ and $nN$, respectively. Assumption~\ref{ass:denom_nondeg} makes this requirement explicit and also ensures that the corrected denominators remain positive with high probability.

\begin{theorem} \label{thm:cc_to_oracle}
Under Assumptions~\ref{ass:kernel}--\ref{ass:denom_nondeg}, condition on the realized samples $\{X_i\}_{i=1}^n$ and 
$\{(W_j,I_j,\widetilde I_j)\}_{j=1}^N$. Then, as $m\to\infty$,
\begin{equation}
\label{eq:cc_oracle_consistency}
\widehat{\mathrm{MMD}}^2_{\mathrm{CC}}-\widehat{\mathrm{MMD}}^2_\star
\;\xrightarrow{\mathbb P_L}\;0,
\end{equation}
and moreover
\begin{equation}
\label{eq:cc_oracle_clt}
\sqrt m\Bigl(\widehat{\mathrm{MMD}}^2_{\mathrm{CC}}-\widehat{\mathrm{MMD}}^2_\star\Bigr)
\;\Rightarrow\;
\mathcal N(0,V_{\mathrm{design}})
\end{equation}
for some finite variance $V_{\mathrm{design}}\ge 0$, determined by the joint design-based covariance of
\[
(\widehat A_{qq}-A_{qq}^\star,\,
\widehat B_{qq}-B_{qq}^\star,\,
\widehat A_{pq}-A_{pq}^\star,\,
\widehat B_{pq}-B_{pq}^\star).
\]
\end{theorem}

\paragraph{Population Consistency.} Theorem~\ref{thm:cc_to_oracle} is purely design-based. Conditional on the realized finite population, it shows that CC-MMD asymptotically reproduces the oracle estimator that would have been computed if the true membership labels were fully observed. The remaining discrepancy is due solely to finite audit noise, which vanishes at the standard $m^{-1/2}$ rate. To conclude consistency for the population target $\mathrm{MMD}^2(p,q;k)$, it remains to relate the oracle estimator to the population expectations in Equation~\eqref{eq:mmd_target}. 

\begin{theorem}[Consistency]
\label{thm:consistency_population}
Under Assumptions~\ref{ass:kernel}--\ref{ass:denom_nondeg},
\begin{equation}
\label{eq:cc_consistency_population}
\widehat{\mathrm{MMD}}^2_{\mathrm{CC}}
\;\xrightarrow{\mathbb P}\;
\mathrm{MMD}^2(p,q;k),
\end{equation}
where the probability is over all sources of randomness in the sample generation and audit design.
\end{theorem}

Theorem~\ref{thm:consistency_population} combines the two parts of the analysis. The design-based results show that auditing corrects the proxy labels well enough to recover the oracle estimator, while the population argument shows that the oracle estimator itself converges to the desired discrepancy between $p$ and $q$. In this sense, CC-MMD is consistent even though the target sample is contaminated and the proxy labels may be arbitrarily miscalibrated.

\begin{remark}[Asymptotic Normality.]
The assumptions developed so far are sufficient for consistency, but not by themselves for a full population-level CLT around $\mathrm{MMD}^2(p,q;k)$. To strengthen Theorem~\ref{thm:consistency_population} into asymptotic normality, one would additionally need a limit theorem for the oracle empirical MMD itself (see Corollary 16 in \citep{gretton2012kernel}), and then combine that oracle fluctuation with the design-based fluctuation from Theorem~\ref{thm:cc_to_oracle}.
\end{remark}

\begin{remark}
If $m/N\to 0$ but $m\to\infty$, the correction remains consistent for the oracle quantities, but the design variance is larger because it scales with the smaller audit size. In practice, this means slower convergence and greater sensitivity to audit noise.
\end{remark}

\begin{remark}[Choosing the audit size]
The audit size required in practice is governed primarily by the variance of the audit correction, $V_{\mathrm{design}}$, in Equation~\eqref{eq:cc_oracle_clt}. It is generally not controlled by the number of contamination mechanisms present in the pool. In this respect, the audit plays a role similar to a validation sample in prediction-powered inference \citep{angelopoulos2023prediction}: once it estimates the discrepancy between proxy scores and true labels with sufficient precision, additional audits yield diminishing returns, as observed in our audit-size ablation, Figure~\ref{fig:ablation-audit-size}.

When the contamination structure is unknown, one can begin with a pilot audit, estimate the correction variance, $V_{\mathrm{design}}$, and then choose $m$ so that $V_{\mathrm{design}}/m$ falls below a target threshold, such as the variance of the audit–only MMD estimator. When contamination structure can be modeled or simulated, the correction and MMD variances can instead be estimated through simulation to guide the audit budget.
\end{remark}

\vspace{-2mm}
\section{Experiments}
\label{sec:experiments}

The details of all experiments and additional results are in the Appendix.
The code to reproduce our experiments is publicly available at \url{https://github.com/NevenaGligic/CC-MMD}.  

{\bf Experiment 1 -- MMD Estimation.}  \vspace{-1.8mm}

\begin{wrapfigure}{r}{0.50\textwidth}
\vspace{-12pt}
\centering
\includegraphics[width=\linewidth]{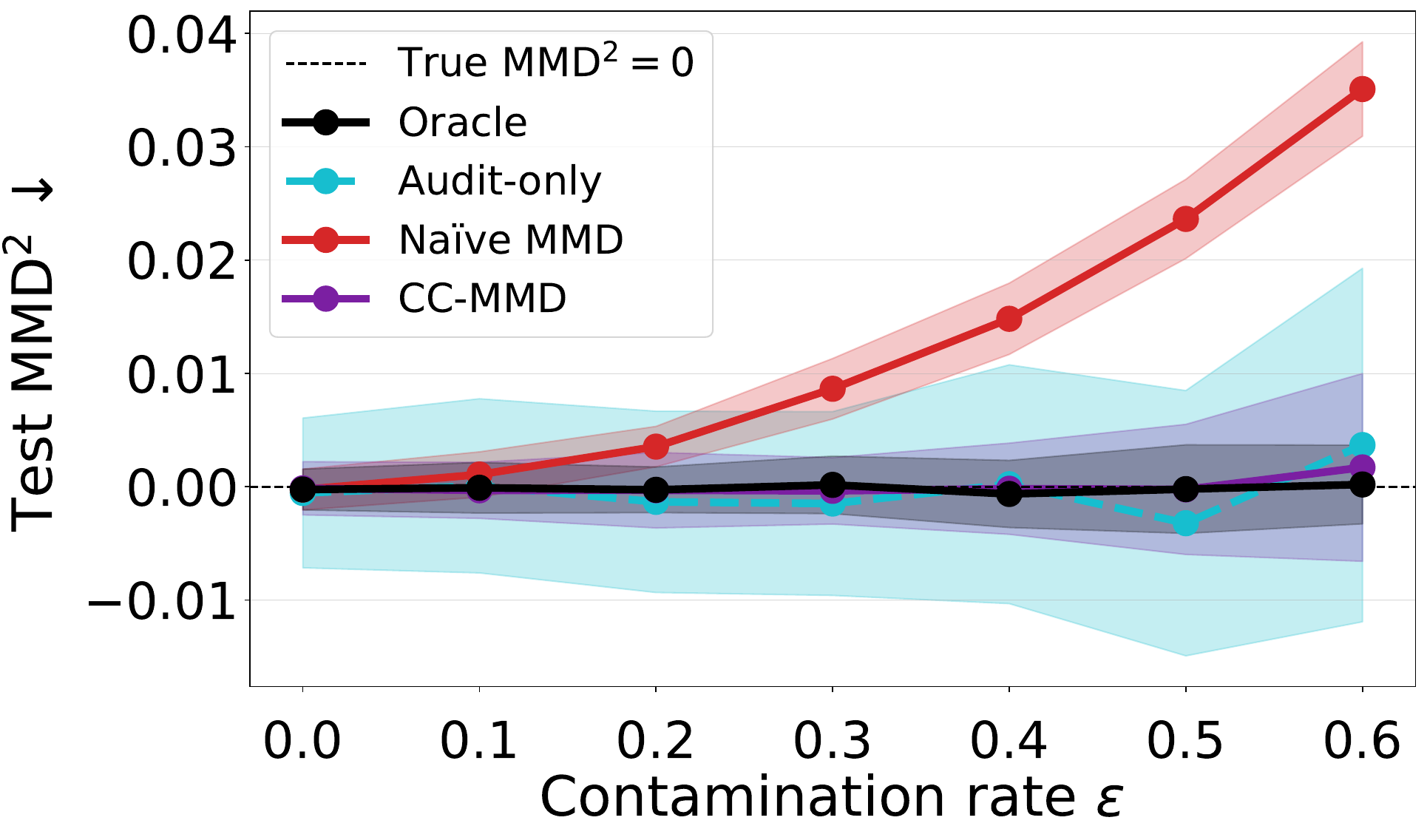}
\vspace{-15pt}
\caption{\textbf{MMD estimation under increasing contamination.} Na\"ive MMD becomes increasingly biased as $\epsilon$ grows, whereas CC--MMD tracks the oracle and is less variable than the audit-only estimator.}
\label{fig:sim_experiments}
\vspace{-10pt}
\end{wrapfigure}
Figure~\ref{fig:sim_experiments} evaluates the proposed estimator across increasing contamination levels in a setting where the reference and target distributions coincide, so the population target is $\mathrm{MMD}^2(p,q)=0$. Thus, departures from zero reflect contamination-induced or finite-sample bias. The oracle estimator, which uses the true latent membership indicators, remains centered near zero and serves as the benchmark. In contrast, the na\"ive estimator increasingly overestimates the discrepancy as $\epsilon$ grows, reflecting bias from treating the contaminated pool as the target sample. The audit-only estimator is unbiased but has a higher variance due to a smaller sample size. CC--MMD closely follows the oracle throughout the contamination sweep (Theorems~\ref{thm:cc_to_oracle} and \ref{thm:consistency_population}), indicating that the audit-based residual correction removes most contamination-induced bias while retaining the efficiency benefits of using the full pool. 

\newpage

{\bf Experiment 2 -- Regression.} \vspace{-1.8mm}

\begin{wraptable}{r}{0.58\textwidth}
\vspace{-11pt}
\centering
\footnotesize
\setlength{\tabcolsep}{3pt}
\begin{tabular}{lccc}
  \toprule
  Method & Intercept & Slope & $\ell_2$ error \\
  \midrule
  True        & 1.000 & 2.000 & -- \\
  MSE         & $0.621 \pm 0.854$ & $0.923 \pm 0.634$ & $1.449 \pm 0.573$ \\
  Huber       & $0.859 \pm 0.670$ & $1.319 \pm 0.654$ & $0.888 \pm 0.756$ \\
  Tukey       & $0.924 \pm 0.622$ & $1.452 \pm 0.725$ & $0.733 \pm 0.825$ \\
  Na\"ive MMD & $0.966 \pm 0.231$ & $2.146 \pm 0.136$ & $0.259 \pm 0.164$ \\
  Audit-only  & $0.988 \pm 0.164$ & $1.994 \pm 0.106$ & $0.169 \pm 0.098$ \\
  CC--MMD     & $\mathbf{1.000 \pm 0.105}$ & $\mathbf{2.001 \pm 0.072}$ & $\mathbf{0.110 \pm 0.063}$ \\
  \bottomrule
  \end{tabular}
\vspace{-5pt}
\caption{\textbf{Regression task benchmarking.} Mean $\pm$ s.d.\ over 200 independent realizations; best in bold. CC--MMD is centered on the true parameters and has the smallest spread of any method on both, attaining $35\%$ lower $\ell_2$ error than audit-only, the strongest baseline.}
\label{tab:regression}
\vspace{-10pt}
\end{wraptable}
Table~\ref{tab:regression} evaluates parameter recovery in a contaminated linear regression problem over 200 independent realizations. Clustered outliers induce bias in least-squares estimation, pulling the fitted line toward the contaminant clusters and away from the ground truth, so MSE substantially underestimates the slope. Robust regression does not meaningfully mitigate this effect: Huber regression and Tukey's biweight reduce this bias but do not remove it, and remain highly variable across realizations. The uncorrected MMD objective is less sensitive to outliers, consistent with known robustness properties of MMD under Huber contamination \citep{cherief2022finite}, but it still targets the contaminated mixture and overestimates the slope. An audit-only baseline avoids contamination entirely but discards most of the sample, at a cost in precision.

In contrast, CC--MMD uses the proxy scores and audited labels to correct the distribution-matching objective, and is both the least biased and the least variable method: its mean intercept and slope are centered on the true values, with the smallest standard deviation of any method on both parameters. Its mean $\ell_2$ error is about 35\% lower than that of audit-only, the strongest baseline, and six to thirteen times lower than MSE, Huber and Tukey. The margin over audit-only is what the corrected objective recovers from the unaudited observations that audit-only discards. This holds even though the proxy is weak (AUROC $0.66 \pm 0.11$), reflecting that the audit correction requires no assumption on proxy quality (Theorem~\ref{thm:unbiased_sums}). These results illustrate the distinction between stability and correction. Bounded-kernel MMD can limit the influence of contamination \citep{cherief2022finite, alquier2024universal, briol2019statistical}, explaining the relative performance of na\"ive MMD, but with fixed non-negligible $\epsilon$, the contamination-induced estimand shift persists asymptotically and can be relatively large. CC--MMD removes most of this bias while retaining the distribution-matching advantages of MMD. Individual realizations and the full distribution of estimates are shown in Appendix~\ref{sec:appendix_linear_regression}, Figures~\ref{fig:regression_multi} and~\ref{fig:regression_200}.

\begin{figure}[th]
  \centering
    \includegraphics[width=0.48\linewidth]{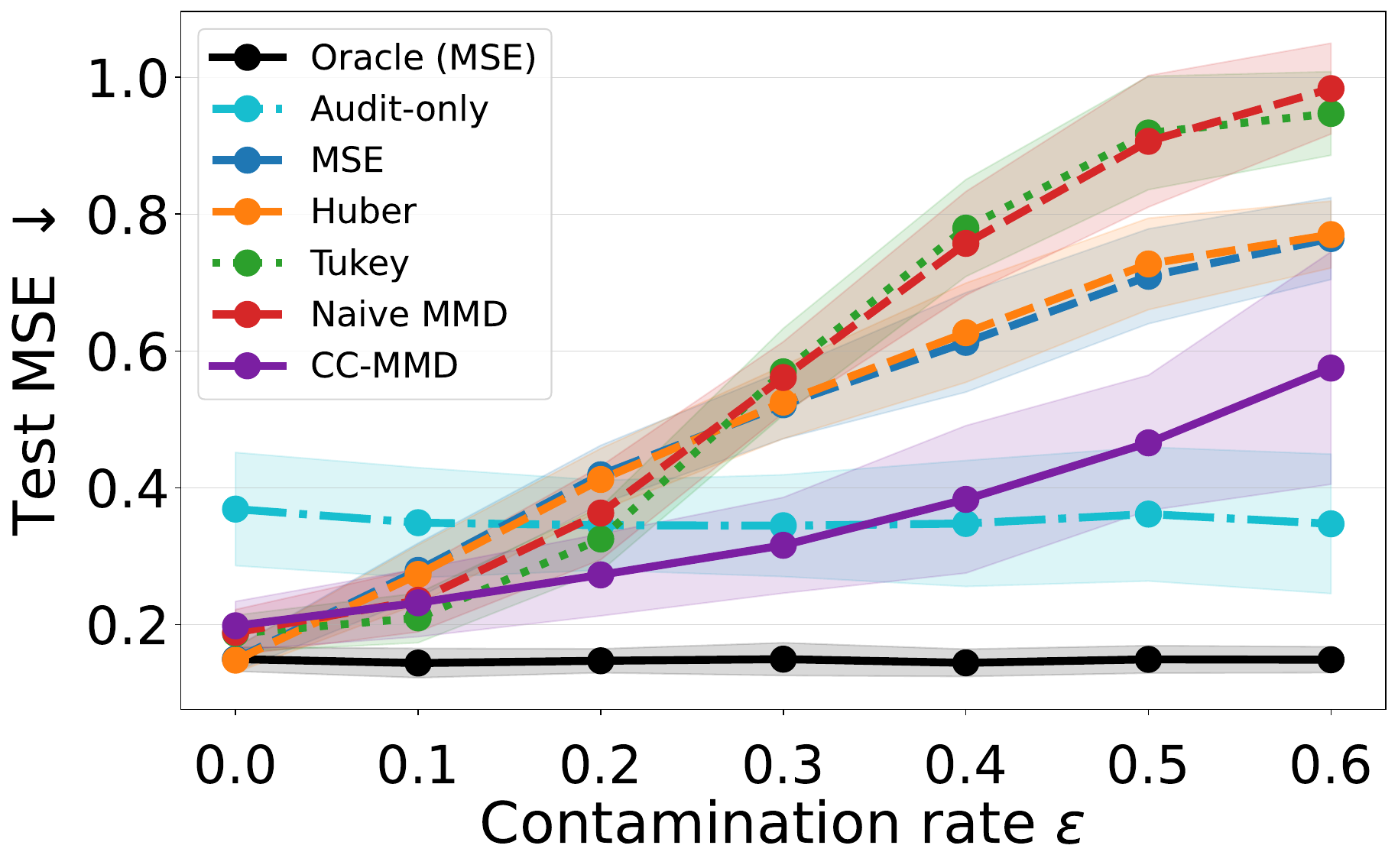}
    \includegraphics[width=0.48\linewidth]{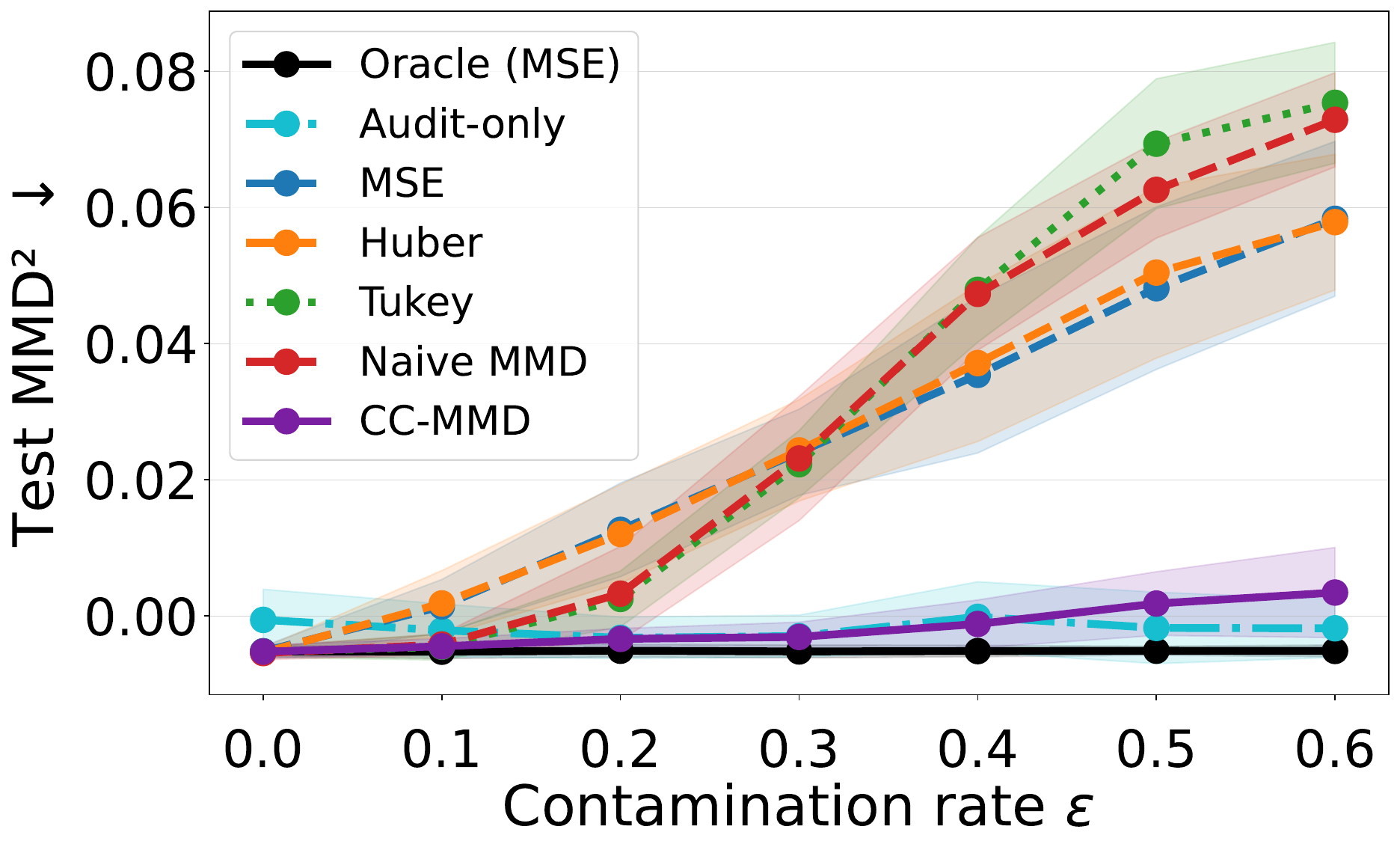}
    \vspace{-1mm}
  \caption{\textbf{Prediction task benchmarking.} Shaded regions denote $\pm 1$ standard deviation over 20 seeds. {\bf Left:} Test MSE vs.\ $\epsilon$.
      CC-MMD is the best-performing method trained on full data for $\epsilon > 0.1$, achieving up to 39\% lower MSE than the next-best baseline. {\bf Right:} Test MMD$^2$ between predicted and true output distributions vs.\ $\epsilon$. CC-MMD maintains near-zero divergence across all contamination levels, matching the Oracle, while other methods trained on full data diverge monotonically.}
      \vspace{-3mm}
  \label{fig:exp4_quant}
\end{figure}


{\bf Experiment 3 -- Prediction (Concrete Compressive Strength).}  \vspace{-1.8mm}

Figure~\ref{fig:exp4_quant} evaluates CC--MMD in a supervised prediction task on the UCI Concrete Compressive Strength dataset~\citep{yeh1998modeling}. We construct a natural contamination setting by treating 28-day cured concrete as the target population, $q$, and early-cured concrete ($\leq 7$ days) as contaminants, $z$. The curing-age variable is removed from the features, so the contaminating samples are not trivially identifiable. Unlike the synthetic outliers in Experiment~2, these contaminants are not extreme residual outliers: they occupy overlapping feature support but induce a systematic downward shift in the conditional relationship between material composition and compressive strength. This makes the experiment a harder and more realistic test of whether the method can correct distributional contamination rather than merely resist large residuals.

For each contamination level $\epsilon \in \{0.0,0.1,\ldots,0.6\}$, we train MLP regressors on a contaminated dataset containing 300 samples, fixed, from the target population and the rest from the contaminating population. The performance is tested on a clean held-out test set with 125 samples from $q$. The audit budget is $20\%$ of the contaminated training dataset. All methods use the same architecture and optimizer. Implementation details are provided in Appendix~\ref{sec:exp_concrete_full}. We compare the oracle model trained on the true inliers, audit-only training, MSE, Huber, Tukey's biweight, na\"ive MMD, and CC--MMD.

At $\epsilon=0$, CC--MMD is competitive with the oracle and standard losses, indicating that the correction introduces no meaningful degradation when the pool is clean. As contamination increases, however, the baselines deteriorate sharply. MSE and Huber reach test errors of approximately $0.76$ and $0.77$ at $\epsilon=0.6$, respectively, showing that Huber robustness does not help when contaminants have moderate residuals. Tukey's biweight performs even worse at high contamination, reaching approximately $0.95$, because the contaminating distribution inflates the residual scale and causes unreliable downweighting. Na\"ive MMD also fails, reaching approximately $0.98$. 
In contrast, CC--MMD is the best-performing non-oracle method which uses full training data for every $\epsilon > 0.1$. At $\epsilon=0.5$, it obtains test MSE $0.47 \pm 0.10$, a $34\%$ reduction relative to the next-best baseline; at $\epsilon=0.6$, it obtains $0.58 \pm 0.17$, a $25\%$ reduction. The gains are the largest at intermediate contamination levels, where CC--MMD improves over MSE by up to $39\%$. The right panel of Figure~\ref{fig:exp4_quant} shows the same pattern distributionally: CC--MMD maintains test MMD$^2$ near the oracle level, while all uncorrected methods which use full training data diverge as contamination increases.

Figure~\ref{fig:exp4_scatter} shows that, at $\epsilon=0.5$, MSE, Huber, Tukey, and na\"ive MMD underpredict high-strength concrete because early-cured contaminants pull predictions downward. In contrast, CC--MMD tracks the oracle and stays close to the identity line. Figure~\ref{fig:exp4_qq} confirms this pattern: CC--MMD recovers the clean marginal output distribution, while the baselines compress the predicted range. These results show that robustness alone does not correct contamination; CC--MMD uses audit information to adjust the distribution-matching objective and recover clean-target performance.

{\bf Experiment 4 -- Image Generation.}  \vspace{-1.8mm}

We evaluate CC--MMD on image generative modeling with contaminated data. The target is clean MNIST~\citep{deng2012mnist}, and the contaminating distributions are K-MNIST~\citep{clanuwat2018deep} and Omniglot~\citep{lake2015human}. We construct a contaminated training pool with $N=10{,}000$ images and severe contamination, $\epsilon=0.60$. We audit $m=1{,}000$ images ($10\%$ of the pool) and evaluate three generator architectures: inverse autoregressive flow (IAF, \cite{kingma2016improved}), variational autoencoders (VAE, \cite{kingma2013auto}), generative moment matching networks (GMMN, \cite{li2015generative}). For each architecture, CC--MMD is compared against na\"ive MMD and the corresponding architecture's standard training objective: sliced Wasserstein distance (SWD, \cite{rabin2011wasserstein}) for IAF; ELBO with MSE and Huber reconstruction for VAE; adversarial training for GMMN (GAN, \cite{goodfellow2020generative}). All comparisons fix the architecture, initialization, optimizer, and training budget. Implementation details are provided in Appendix~\ref{sec:appendix_generative}. We evaluate generated samples against 10,000 clean MNIST test images using test MMD$^2$ and Fréchet Inception Distance (FID). See Figure~\ref{fig:ccmmd_contam_metrics} for the results. 

\begin{figure}[htp]
    \centering
    \includegraphics[width=0.49\linewidth]{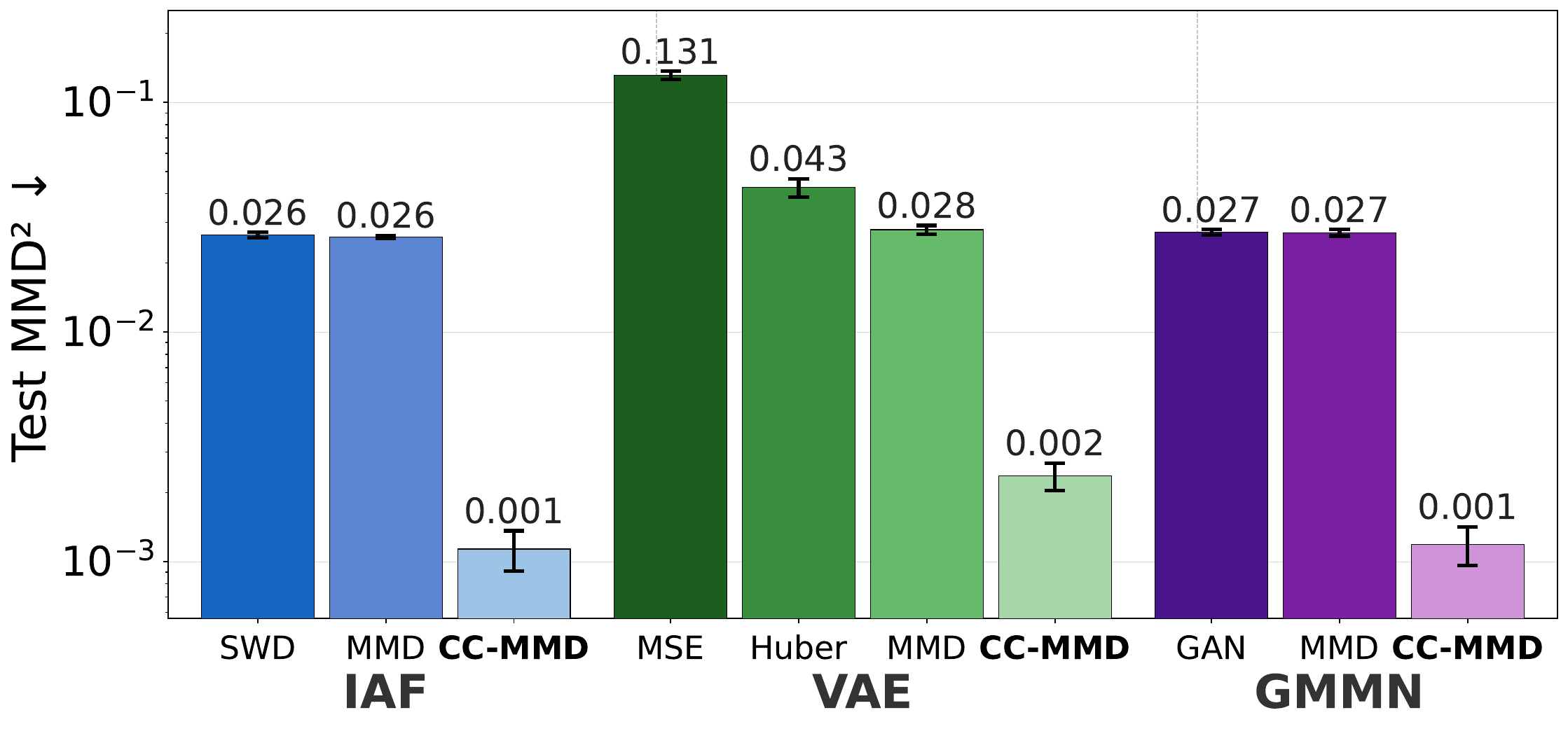}
    \hfill
    \includegraphics[width=0.49\linewidth]{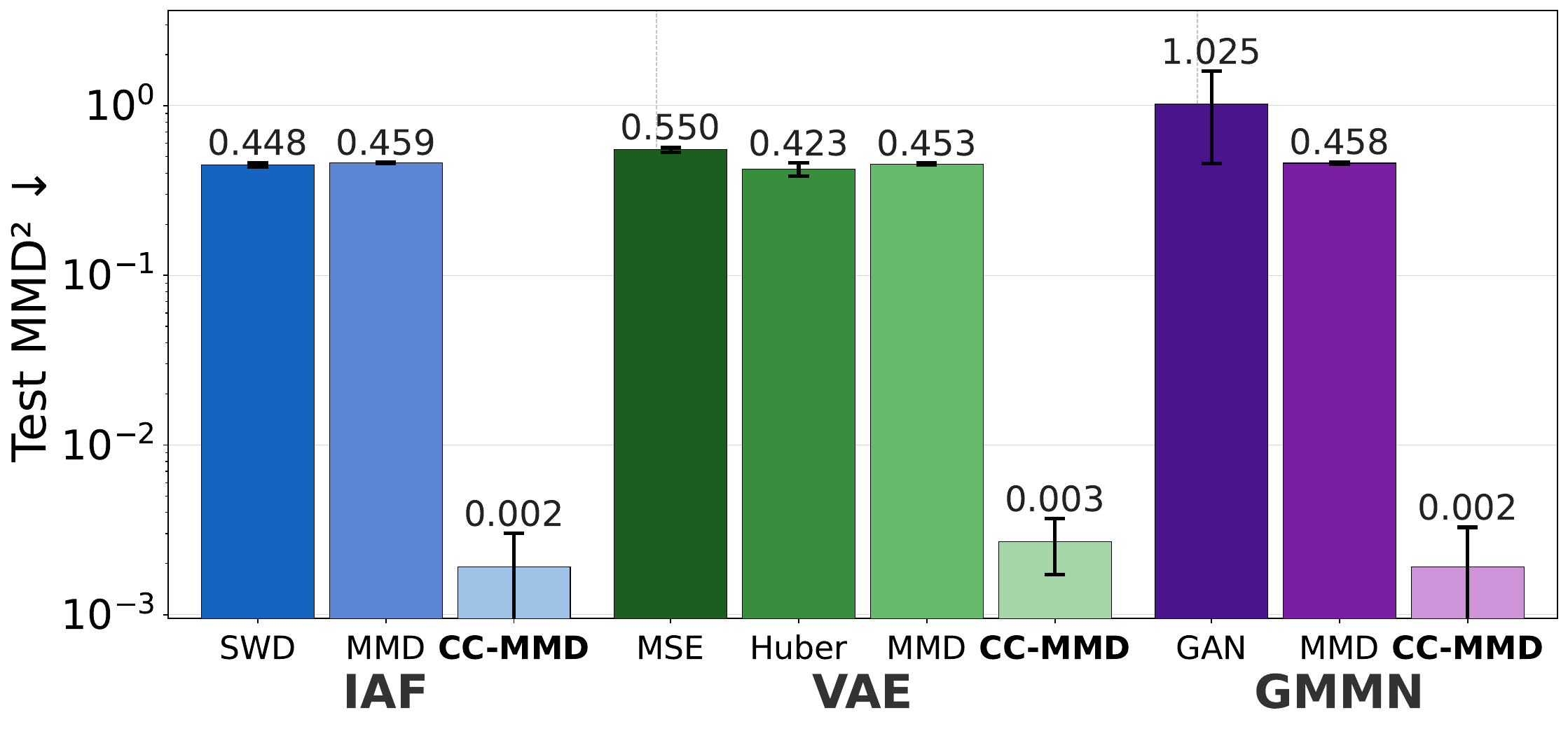} \\
    \includegraphics[width=0.49\linewidth]{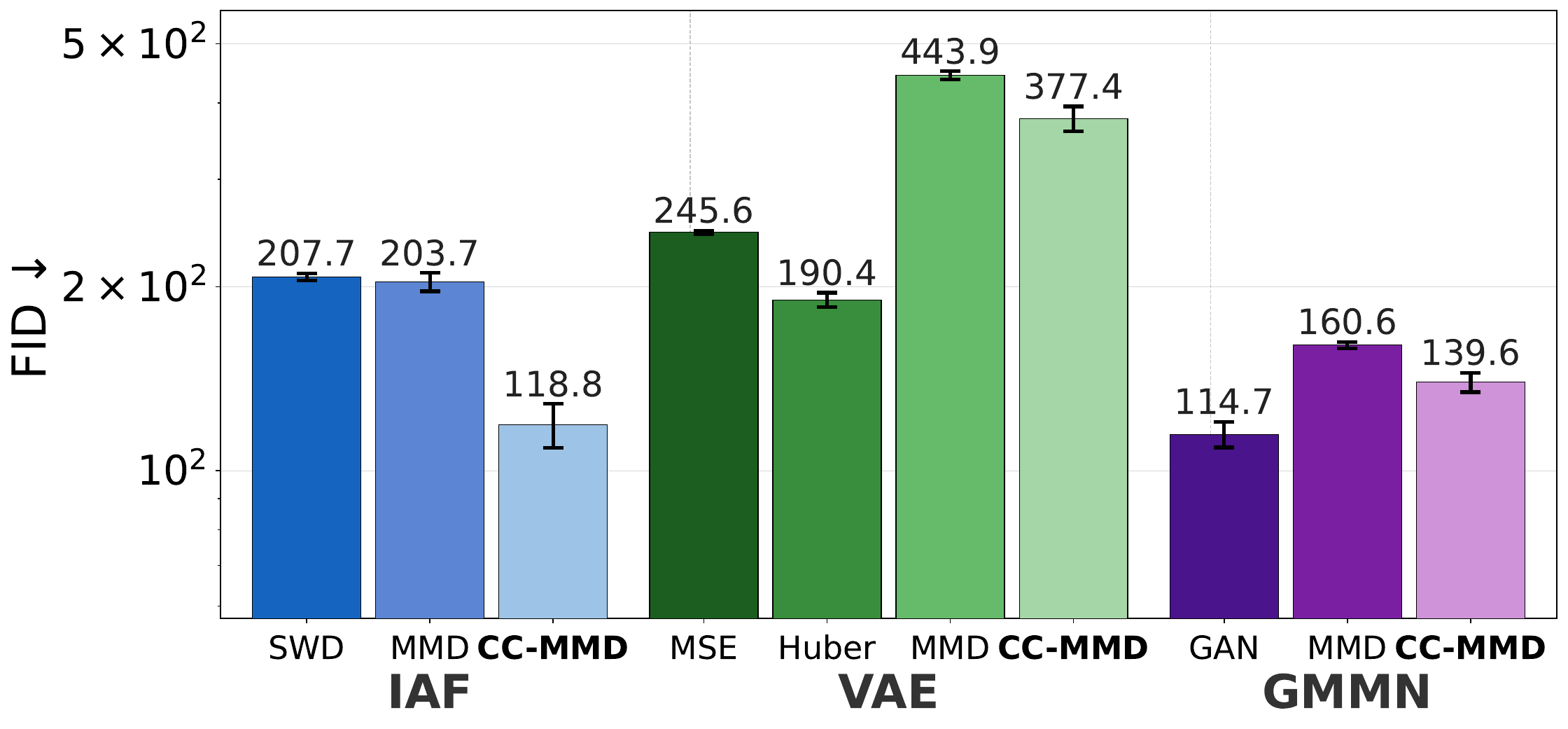}
    \hfill
    \includegraphics[width=0.49\linewidth]{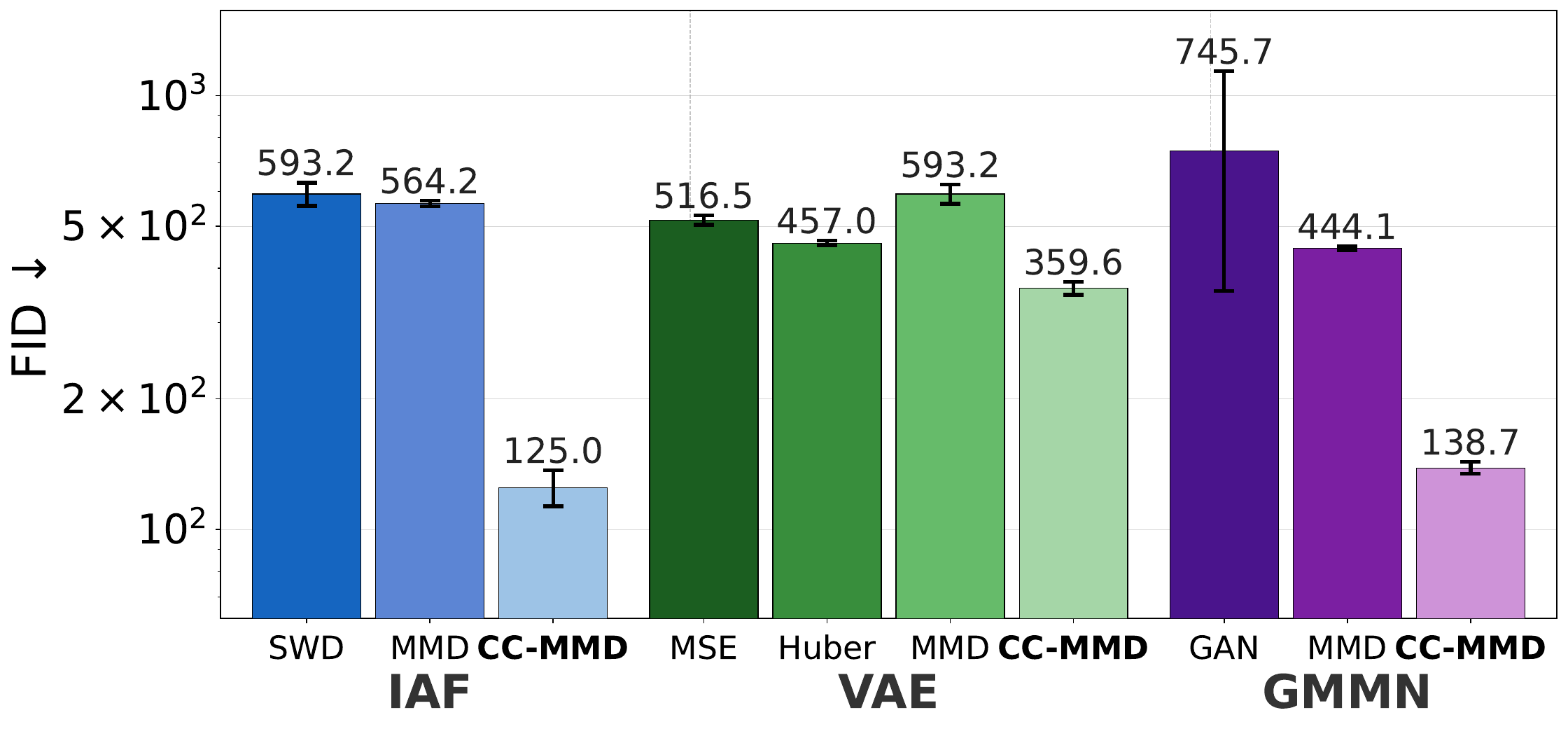} \\
     \vspace{-1mm}
    \caption{\textbf{Generative modeling on MNIST.} Comparison of test MMD$^2$ (\textbf{top}) and FID (\textbf{bottom}) at $\epsilon=0.60$, for K-MNIST (\textbf{left}) and Omniglot (\textbf{right}) contaminants, across three architectures and respective baselines. Each method run on 5 seeds with error bars showing $\pm 1$ std.} \vspace{-3mm}
    \label{fig:ccmmd_contam_metrics}
\end{figure}

Under $\epsilon=0.60$ contamination, uncorrected baselines optimize distances to the contaminated training pool, thus, learning mixtures that retain substantial contaminant structure. This results in poor generative performance. CC--MMD instead uses audit-informed proxy scores to correct the distribution-matching loss toward the clean MNIST target. Across all three architectures and both contamination settings, CC--MMD achieves substantially lower test MMD$^2$ than all baselines. FID improvements are consistent for IAF and under Omniglot contamination, where baselines degrade most severely. For VAE and GMMN under K-MNIST contamination, ELBO-based reconstruction and adversarial training achieve competitive FID despite higher test MMD$^2$, suggesting these objectives partially capture perceptual quality while still fitting the contaminated distribution. The main result is that, regardless of the architecture, the CC-MMD loss improves performance under severely contaminated data conditions.

The contamination-rate ablation varies $\epsilon$ from $0$ to $0.8$ while fixing $N=10{,}000$, $m=1{,}000$, proxy AUROC $\approx 0.70$, and run 20 random realizations with IAF (Figure~\ref{fig:ablation_epsilon_mmd2}). Na\"ive MMD and SWD degrade rapidly as contamination increases. At $\epsilon=0.8$, na\"ive MMD reaches test MMD$^2=0.046$, more than two orders of magnitude above the oracle baseline of $0.00055$. CC--MMD is robust to this degradation. It obtains test MMD$^2=0.00121$ at $\epsilon=0.6$ and $0.00371$ at $\epsilon=0.8$, giving $21\times$ and $12\times$ reductions relative to na\"ive MMD. Audit-only training performs competitively at very high contamination because it discards the contaminated pool, but CC--MMD outperforms audit-only for $\epsilon \leq 0.5$ on test MMD$^2$ by using non-audited clean samples through the corrected full-pool objective.

The proxy-quality and audit-size ablations show that CC--MMD benefits from a larger audit sample but does not require a near-perfect proxy to improve over contaminated distribution matching. As proxy AUROC increases, performance weakly improves, consistent with the interpretation that proxy quality mainly affects finite-sample efficiency and optimization rather than the validity of the correction itself (Figure~\ref{fig:ablation_proxy}). Increasing the audit budget significantly improves performance by producing stronger proxy scores and more stable clean-audit kernel estimates (Figure~\ref{fig:ablation-audit-size}), as anticipated by Theorem~\ref{thm:design_clt_full}. The gains are largest when moving from very small to moderate audit sizes, after which returns diminish. Even with a small audit set, CC--MMD meaningfully outperforms na\"ive MMD trained on the full contaminated pool, indicating that limited but informative audit supervision is enough to correct much of the contamination-induced bias.

\begin{figure}[th]
    \centering
    \includegraphics[width=0.49\linewidth]{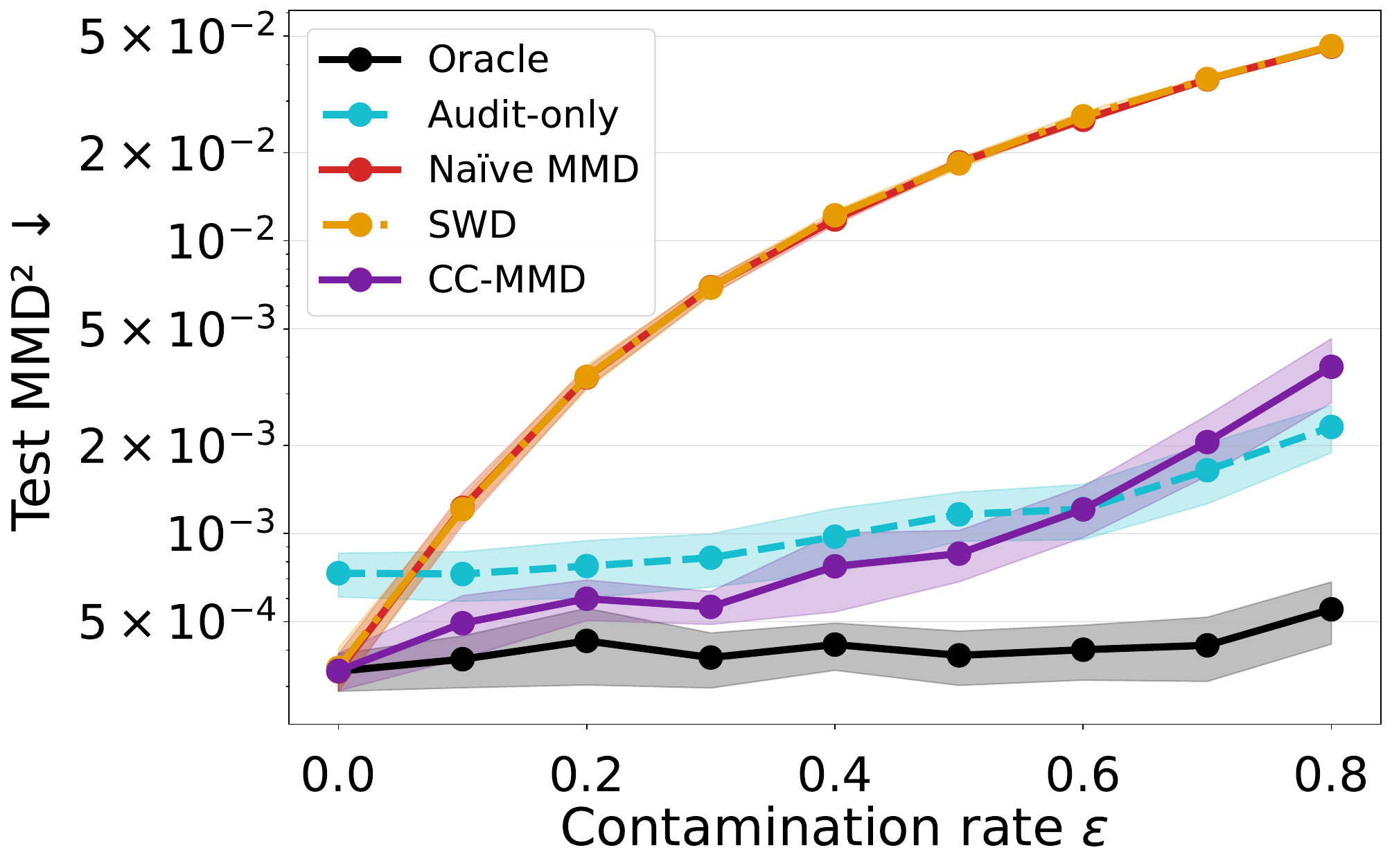}
    \hfill
    \includegraphics[width=0.48\linewidth]{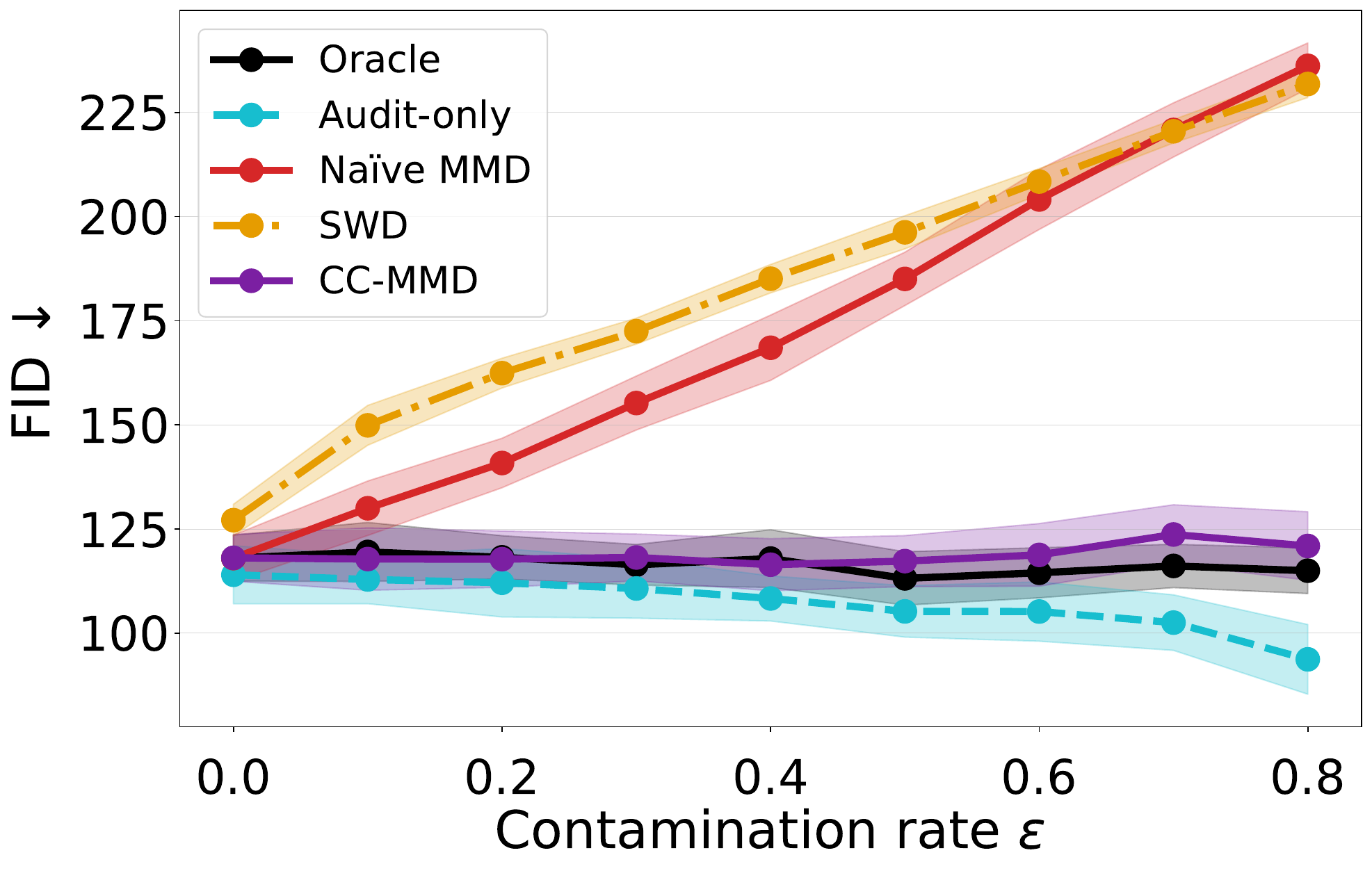}
    \caption{
        \textbf{Generative quality vs.\ contamination rate on MNIST.} 
        Test $\mathrm{MMD}^2$ in log scale (\textbf{left}) and FID (\textbf{right}) against the clean MNIST test set. CC-MMD uses a proxy with AUROC $\approx 0.70$ and audited $m/N=10\%$ of the training pool. Shaded bands show $\pm 1$ standard deviation over 20 random seeds. }
    \label{fig:ablation_epsilon_mmd2}
\end{figure}

\section{Discussion}
\label{sec:discussion}

Our results show that CC--MMD addresses a central failure mode in learning from contaminated data \citep{huber1992robust, hampel1974influence}: standard losses can remain stable and well optimized while still targeting the wrong distribution. MMD, SWD, MSE, and robust pointwise losses treat the contaminated pool as the target sample, so they often learn the mixture rather than the clean distribution. Robust losses such as Huber can reduce the influence of large residuals, and existing MMD robustness results explain why MMD may behave more stably than likelihood- or squared-error-based alternatives under contamination \citep{briol2019statistical, cherief2020mmd, barp2019minimum, cherief2022finite, dellaporta2022robust, alquier2024universal}. However, robust losses typically control, bound, or downweight contamination; but they do not determine the target distribution. As a result, when contaminants lie within the support of the target data, a fixed contamination level can still induce nonvanishing asymptotic biases. This matters when contaminants are not residual outliers, as in the concrete prediction experiment, where off-target samples shift the conditional relationship within overlapping feature support. In contrast, CC--MMD is robust by design. It uses audit information to reweight the contaminated pool and correct the distribution-matching objective itself, shifting the estimand toward the clean target discrepancy rather than the mixture discrepancy. Thus, our contribution is not another bounded-influence guarantee, but a design-based correction that removes contamination-induced bias from the MMD estimand.

\paragraph{Limitations and Open Challenges.} CC--MMD has not been systematically evaluated in high-dimensional settings, and how the required audit sample size should scale with dimension remains an open question. CC--MMD is most effective when the contaminated pool still contains enough clean samples and the audit set provides sufficient signal to guide reweighting. In this regime, it combines the efficiency of full-pool training with the bias reduction of audit supervision, explaining why it can outperform both na\"ive full-pool and audit-only training. At very high contamination, or when audits are too small, biased, or unavailable, performance may degrade because the unaudited pool contributes fewer clean samples and correction variance increases (e.g., Figure~\ref{fig:exp4_quant}). Additionally, poor proxies increase variance and can make optimization more challenging, particularly at high contamination levels. Understanding the trade-offs among audit size, proxy accuracy, and estimator variance remains an avenue of further investigation. 

CC--MMD inherits the computational and modeling considerations of MMD. Exact computation is dominated by pairwise kernel evaluations and therefore requires $O(N^2)$ time. The additional correction terms involve the audited subset, with $m \ll N$, and do not change the leading-order complexity. As with standard MMD, this quadratic cost may become a bottleneck for very large datasets. However, existing scalable approximations (e.g., random Fourier features \cite{rahimi2007random}, Nyström methods \cite{williams2000using}, block estimators \cite{zaremba2013b}, and linear-time MMD estimators \cite{gretton2012kernel}) can be adapted to CC-MMD with minimal modification because the correction is expressed through the same empirical kernel sums. Performance depends on the choice of kernel and bandwidth, and there is no guarantee that a single kernel configuration is optimal for both the full-pool and audit-correction components. Systematic kernel selection and bandwidth tuning remain open challenges. While we demonstrate strong empirical performance across inference, prediction, and generative modeling tasks, the current study focuses on settings with a single contaminating distribution and relatively simple proxy models. Real-world data may involve multiple sources of contamination, structured or adversarial shifts, and more complex dependence between features and membership. Extending or evaluating CC--MMD for these richer contamination regimes and integrating it with domain-specific models or representations presents an important avenue for future development.

\begin{ack}
The authors acknowledge support from the National Science Foundation under Cooperative Agreement No. 2421782 and from the Simons Foundation under Award No. MPS-AI-00010515, through the NSF–Simons AI Institute for Cosmic Origins (CosmicAI, \url{https://cosmicai.org/}). Additionally, the authors acknowledge the Texas Advanced Computing Center (TACC, \url{http://www.tacc.utexas.edu}) at The University of Texas at Austin for providing computational resources that have contributed to the research results reported within this paper. 
\end{ack}

\bibliographystyle{unsrt}
\bibliography{bibliography}


\appendix

\section*{Related Work}
\label{sec:related_work}

Our setting (a contaminated target pool, a noisy membership score available for all pooled points, and an audited uniformly random labeled subset) touches three mature literatures: (i) kernel mean embeddings and MMD-based two-sample methodology, (ii) contamination/mixture and weak-supervision paradigms with latent membership, and (iii) design-based estimation with auxiliary information, including recent prediction-powered (``predict-then-correct'') inference.

\paragraph{Kernel mean embeddings, MMD, and two-sample testing.} Kernel mean embeddings map probability measures into an RKHS, enabling comparison of distributions via distances between embeddings \citep{muandet2017kernel}.
Maximum mean discrepancy (MMD) is the induced integral probability metric obtained by taking the supremum over the unit ball of the RKHS; empirical MMD estimators are typically U-statistics (unbiased) or V-statistics (biased but simpler) built from pairwise kernel evaluations \citep{gretton2006kernel,gretton2012kernel}.
A substantial body of work develops calibration and power analyses for MMD two-sample tests, including asymptotic null theory for degenerate U-statistics and practical resampling methods.
In particular, wild bootstrap procedures provide consistent calibration for degenerate kernel tests in settings where na\"ive permutation can fail (e.g., dependence) \citep{chwialkowski2014wild}, and aggregated/ensemble constructions improve adaptivity and empirical power across kernel choices \citep{schrab2023mmd}.
Most of this literature assumes direct access to i.i.d.\ samples from each distribution; our goal is instead to target the latent clean component $q$ when the observed target pool is a mixture.

\paragraph{MMD as a learning objective and distribution alignment.} Beyond testing, MMD is widely used as a discrepancy to \emph{learn} representations or generative models by matching distributions in feature space.
In domain adaptation, MMD penalties are used to align source and target feature distributions.
Under covariate shift and related dataset-shift settings, kernel mean matching (KMM) uses RKHS moment matching to construct importance weights \citep{gretton2009covariate}.
In generative modeling, MMD objectives provide a tractable alternative to likelihoods and adversarial training; for example, generative moment matching networks minimize empirical MMD between generated and observed samples \citep{li2015generative,li2017mmd}.
More recently, MMD and related kernel discrepancies have been developed as tools for \emph{parameter estimation} and \emph{likelihood-free} (simulation-based) inference. In particular, kernel-based minimum-distance objectives yield consistent estimators under weak assumptions and provide a principled way to fit implicit models by matching distributions in RKHS \citep{briol2019statistical}. Complementary work in simulation-based inference uses MMD as a discrepancy between observed and simulated samples (for example, kernel-embedding ABC methods such as K2-ABC and robust simulator-based Bayesian procedures based on the MMD posterior bootstrap) and further develops MMD estimators tailored to simulator-based objectives, including optimally weighted constructions that improve efficiency and sample complexity \citep{park2016k2,dellaporta2022robust,bharti2023optimally}. Collectively, this literature establishes MMD as a practical objective for estimation and inference, motivating the question of whether such objectives can be made to target a clean component when the observed sample is contaminated. 

\paragraph{Robustness and contamination in kernel methods.}
A parallel literature studies the robustness of kernel mean embeddings and MMD-type quantities under outliers or H\"uber/adversarial contamination. Median-of-means and related robust-mean principles have been adapted to kernel mean embedding estimation and MMD, yielding estimators with high resistance to outliers and sub-Gaussian deviation bounds under mild conditions \citep{lerasle2019monk}.
In parametric and regression contexts, minimum-distance estimators based on MMD have been developed as ``universal'' estimators that can be robust to misspecification and contamination when bounded kernels are used \citep{cherief2022finite,alquier2024universal}.
Related robust constructions in RKHSs appear in nonparametric density estimation as well \citep{vandermeulen2014robust}.
These works typically pursue robustness by modifying the estimator or objective (e.g., robustification of the embedding or the empirical risk).
Our focus is complementary: while similar to robustness literature, we do not attempt to model contamination; we aim to \emph{debias toward a latent clean component} using an audit set, rather than to robustify against arbitrary outliers in the observed sample.

\paragraph{Mixtures, latent membership, and weak supervision (PU, label noise).}
Modeling a contaminated pool as a mixture $r=(1-\epsilon)q+\epsilon z$ connects to mixture proportion estimation (MPE) and to positive-unlabeled (PU) learning \citep{du2014analysis,du2015convex,bekker2020learning,zhao2022dist}, where one observes samples from a component and from a mixture and seeks to recover properties of the latent component.
Rates and identifiability conditions for MPE have been studied extensively \citep{scott2015rate}, and kernel mean embeddings have been used to design practical MPE procedures with convergence guarantees \citep{ramaswamy2016mixture}.
In PU learning, classical methods exploit assumptions on the labeling mechanism to learn classifiers from positives and unlabeled data \citep{elkan2008learning}.
This literature motivates our latent-indicator formulation and clarifies what can and cannot be identified from mixtures alone.
When a credible parametric model for the contamination process is available, likelihood-based mixture modeling — for instance via EM with latent membership indicators — offers an alternative route to the same latent component. That approach requires specifying the number and form of the mixture components and assuming their identifiability \cite{scott2015rate, ramaswamy2016mixture}, whereas our guarantees require no assumptions on $z$ beyond kernel regularity, and $z$ may itself be a heterogeneous mixture of unknown sources.
However, unlike MPE/PU work, we assume access to a small gold-standard labeled subset sampled by a known design; this enables exact design-based bias correction of the \emph{kernel U-statistics} that define MMD.

\paragraph{Design-based estimation with auxiliary information: HT, calibration, and two-phase sampling.}
Our correction is rooted in finite-population sampling theory.
The Horvitz--Thompson (HT) estimator provides design-unbiased estimation of population totals under known inclusion probabilities \citep{horvitz1952generalization}, with simple random sampling without replacement (SRSWOR) as a basic case \citep{cochran1977sampling}.
When auxiliary information is available for the whole population, model-assisted and calibration/GREG estimators improve efficiency while retaining design-based validity; a common form is ``plug-in prediction + design-unbiased residual correction'' \citep{cassel1976some,deville1992calibration,sarndal2003model}.
In our setting, the proxy membership score $\widetilde I$ plays the role of an auxiliary/predictive variable observed for all $W_j$, while the audited subset provides unbiased estimation of residual totals.
A key difference is that MMD involves \emph{pairwise} (U-statistic) terms, so the correction must address both linear and quadratic components---hence the appearance of both single-unit and ordered-pair HT scalings.

\paragraph{Prediction-powered inference and modern audit-based correction.}
Recent work in prediction-powered inference (PPI) formalizes the idea of combining a large set of ML predictions with a smaller set of gold labels to obtain valid inference without requiring calibration or correct specification of the predictor \citep{angelopoulos2023prediction}.
These methods are conceptually close to model-assisted survey sampling but are designed for modern ML pipelines and downstream inferential targets.
Extensions explicitly connect PPI rectifiers to HT/H\'ajek-style estimators under informative or nonuniform labeling and to bootstrap-based procedures under complex sampling \citep{datta2025prediction,song2026demystifying}.
Our estimator can be viewed as an instance of this ``predict-then-correct'' template for kernel U-statistics: $\widetilde I$ supplies a cheap global plug-in, and the audit set supplies an exactly unbiased design-based correction.

\paragraph{Positioning.}
Overall, our method synthesizes (i) kernel two-sample testing via MMD and its U-statistic structure \citep{gretton2012kernel}, (ii) mixture/contamination and weak-supervision formulations with latent membership \citep{scott2015rate, ramaswamy2016mixture}, and (iii) design-based correction with auxiliary predictions via HT/calibration and prediction-powered inference \citep{horvitz1952generalization,deville1992calibration,angelopoulos2023prediction}.
The distinctive feature of our setting is that the audited sample is drawn under a known probability design (SRSWOR), enabling exact design-unbiased recovery of the oracle kernel sums without assumptions on the nuisance distribution $z$ or on the calibration of $\widetilde I$.

\paragraph{Extended comparison with robustness literature.}
The proposed estimator is complementary to, rather than a replacement for, the existing robustness literature on MMD and kernel-based inference. Classical robust statistics typically asks whether an estimator remains stable when the data-generating distribution is perturbed by outliers, misspecification, or Huber contamination \citep{huber1992robust,hampel1974influence}. Recent MMD-based work follows this tradition by exploiting the boundedness and geometry of kernel mean embeddings: because bounded kernels limit the influence any single observation can have on the empirical embedding, MMD-based minimum-distance estimators can exhibit robustness to misspecification, dependence, Huber contamination, and adversarial outliers \citep{briol2019statistical,cherief2020mmd,barp2019minimum,alquier2024universal}. Similar ideas appear in robust Bayesian and simulator-based inference, where the likelihood is replaced or supplemented by an MMD-based discrepancy to reduce sensitivity to model misspecification; for example, MMD-Bayes and the MMD posterior bootstrap use kernel discrepancies to obtain posterior procedures with improved stability relative to likelihood-based alternatives \citep{cherief2020mmd,dellaporta2022robust}. In these settings, robustness usually means that contamination has a controlled effect on the estimator: the error bound may depend on the contamination fraction $\epsilon$, or the estimator may remain close to the best-fitting model under a contaminated or misspecified distribution. This is an important but different notion from the one considered here. Standard MMD robustness can explain why a naive MMD estimator often behaves better than squared-error or likelihood-based methods under outliers, but it does not generally remove the estimand shift from $q$ to the contaminated mixture $r=(1-\epsilon)q+\epsilon z$. With fixed $\epsilon>0$, the naive empirical MMD converges to a discrepancy involving $r$, not the desired discrepancy involving the clean target $q$. Thus, the residual bias is not a finite-sample instability but a population-level identification problem. Our contribution addresses this different issue: when partial ground-truth membership labels are available through an audit design, one can correct the contaminated kernel sums so that the resulting estimator targets the same oracle quantity that would have been computed from the latent clean points. In this sense, existing MMD robustness is primarily \emph{stability-based}: it controls the damage caused by contamination. The present approach is \emph{design-based}: it uses randomized auditing to debias the contaminated sample and recover the clean-target discrepancy. The two perspectives are therefore compatible. Bounded-kernel MMD provides a stable discrepancy even before correction, while the audit-based correction changes the target from the contaminated distribution back to the scientifically relevant clean distribution.

\subsection*{Illustrative Examples}
\label{sec:illustrative_exmpl}

\paragraph{Example 1: LLM-assisted filtering of political content.}
Suppose we wish to study the distribution $q$ of social media posts related to U.S. politics (e.g., for measuring narrative dynamics or topic drift). A large collection of posts $W_1,\dots,W_N$ is retrieved using keyword queries and automated filtering. However, this retrieval pipeline is imperfect: some posts truly concern U.S. politics, while others are irrelevant but still appear in the pool. Hence, the pooled sample follows $W_j\sim r=(1-\epsilon)q+\epsilon z$, where $q$ denotes genuinely relevant political posts and $z$ represents unrelated or weakly related content. Each post $W_j$ has a latent membership indicator $I_j\in\{0,1\}$, where $I_j=1$ if the post is genuinely about U.S. politics and $I_j=0$ otherwise. Manually labeling every post is infeasible, so an LLM is used as a judge to produce a relevance score $\widetilde I_j\in[0,1]$, interpreted as the model's confidence that the post concerns U.S. politics. Because LLM judgments may be biased, inconsistent, or poorly calibrated, it is infeasible to assume $\widetilde I_j=\mathbb{P}(I_j=1\mid W_j)$. Instead, a small subset $L$ of posts is manually annotated by human reviewers, providing the true indicators $I_j$ for $j\in L$. 

\paragraph{Example 2: Galaxy population inference with star contamination.}
In astronomical surveys, catalogs of galaxies are constructed by automated detection and classification pipelines applied to telescope images. Let $W_{1,\dots,_N}$ denote objects classified as galaxies by the pipeline. In practice, some stars are mistakenly classified as galaxies \citep{fadely2012star}, thereby contaminating the observed catalog and following the mixture $r=(1-\epsilon)q+\epsilon z$, where $q$ is the true distribution of galaxies and $z$ is the distribution of misclassified objects. For each object $W_j$, the latent indicator $I_j$ equals $1$ if the object is truly a galaxy and $0$ otherwise. The pipeline provides a classification score $\widetilde I_j\in[0,1]$ representing the algorithm's confidence that the object is a galaxy. Because classification models are trained on limited labeled data and may suffer from domain shift or instrumental artifacts, these scores are generally not perfectly calibrated posterior probabilities. A small subset $L$ of objects is therefore inspected by astronomers using higher-resolution imaging or spectroscopic follow-up, yielding the true labels $I_j$ for those objects. 

\paragraph{Example 3: Species-distribution inference with environmental DNA contamination.}
In ecological surveys, environmental DNA (eDNA) samples are often used to study the distribution $q$ of a target species across sites or habitats. Let $W_1,\dots,W_N$ denote genetic sequences or site-level detections retained by an automated bioinformatic pipeline as candidate evidence for the target species. In practice, some retained detections may arise from contamination, sequencing error, barcode misassignment, or DNA transported from other locations, so the pooled sample follows $r=(1-\epsilon)q+\epsilon z$, where $q$ is the distribution of genuine target-species detections and $z$ is the distribution of spurious or non-target detections. For each observation $W_j$, the latent indicator $I_j$ equals $1$ if the detection genuinely corresponds to the target species at the sampled site and $0$ otherwise. The pipeline provides a confidence or assignment score $\widetilde I_j\in[0,1]$, reflecting sequence quality, taxonomic match strength, or model-based confidence. Because these scores may be affected by incomplete reference databases, PCR bias, environmental degradation, or site-specific contamination, they cannot generally be treated as calibrated posterior probabilities. A small subset $L$ is therefore validated through expert review, replicate assays, negative controls, or independent field confirmation, yielding exact labels $I_j$ for those observations.

\section{Assumptions}
\label{sec:assumptions}

We separate assumptions used for (i) the design-based correction, which conditions on the realized finite population, from those used to link the oracle estimator to the population MMD.

\begin{assumption}[Kernel regularity]
\label{ass:kernel}
The kernel is bounded:
\begin{equation}
\label{eq:bounded_kernel}
\sup_{x,x'\in\mathcal{X}} |k(x,x')| \le K < \infty.
\end{equation}
\end{assumption}

\begin{remark}[Unbounded kernels]
Boundedness is used throughout our proofs. We expect the analysis to extend to unbounded kernels under the moment conditions
\begin{equation}
\label{eq:moment_unbounded}
\mathbb{E}_{X,X'\sim p}[k(X,X')^2] < \infty,\quad
\mathbb{E}_{W,W'\sim r}[k(W,W')^2] < \infty,\quad
\mathbb{E}_{X\sim p,\,W\sim r}[k(X,W)^2] < \infty,
\end{equation}
which require finite second moments of $k$ under the mixture distribution $r$, and therefore under both $q$ and $z$. A full treatment of this case would replace the uniform bounds in the proofs with $L^2$ arguments and is left to future work.
\end{remark}

\begin{assumption}[Mixture-with-indicator model]
\label{ass:mixture_model}
The contaminated pool is generated via latent membership indicators:
\begin{equation}
\label{eq:indicator_mixture}
I_j \sim \mathrm{Bernoulli}(1-\epsilon), \qquad
W_j \mid (I_j=1) \sim q, \qquad
W_j \mid (I_j=0) \sim z,
\end{equation}
independently across $j=1,\dots,N$. In particular, conditional on $I_j=1$, $W_j$ is an unbiased draw from the target distribution $q$.
\end{assumption}

\begin{assumption}[Simple random labeling]
\label{ass:srs}
The labeled subset $L\subset\{1,\dots,N\}$ is a simple random sample without replacement of size $m$, independent of $(W,I,\widetilde I)$. 
\end{assumption}

\begin{assumption}[Asymptotic regime]
\label{ass:asymptotic_regime}
We consider $n,N,m\to\infty$ such that
\begin{equation}
\label{eq:asymp_rates}
\frac{m}{N}\to \rho\in(0,1],
\qquad
\frac{n}{N}\to \gamma\in(0,\infty).
\end{equation}
The case $\rho\downarrow 0$ corresponds to vanishing audit rates and leads to slower convergence.
\end{assumption}

\begin{assumption}[Finite-population variance stability for linear corrections]
\label{ass:linear_variance}
Conditionally on the realized finite population, the per-unit residual contributions entering the linear $pq$ corrections have stable finite-population variance. Specifically, with
\[
g_j := \sum_{i=1}^n (I_j-\widetilde I_j)\,k(X_i,W_j),
\qquad
u_j := I_j-\widetilde I_j,
\]
\[
\tilde g_j := \frac{1}{n}\sum_{i=1}^n (I_j-\widetilde I_j)\,k(X_i,W_j),
\qquad
u_j := I_j-\widetilde I_j,
\]

their finite-population variances satisfy
\[
S_{\tilde g}^2 \to v_{pq}\in[0,\infty),
\qquad
S_u^2 \to w_{pq}\in[0,\infty).
\]
along the asymptotic regime of Assumption~\ref{ass:asymptotic_regime}. When a nondegenerate Gaussian limit is desired, the limits are assumed strictly positive.
\end{assumption}

\begin{assumption}[Finite-population variance stability for the $qq$ linear corrections]
\label{ass:linear_variance_qq}
Conditionally on the realized finite population, the per-unit residual contributions entering the linear part of the $qq$ corrections have stable finite-population variance. Specifically, define the rescaled per-unit contributions
\[
\tilde a_j := \frac{2}{N-1}\sum_{j'\neq j} \Delta_j\widetilde I_{j'}\,k(W_j,W_{j'}),
\]
\[
\tilde a_j^{(B)} := \frac{2}{N-1}\sum_{j'\neq j} \Delta_j\widetilde I_{j'},
\]
where $\Delta_j := I_j-\widetilde I_j$ on the realized finite population. Their finite-population variances satisfy
\[
S_{\tilde a}^2 \to v_{qq}\in[0,\infty),
\qquad
S_{\tilde a^{(B)}}^2 \to w_{qq}\in[0,\infty),
\]
along the asymptotic regime of Assumption~\ref{ass:asymptotic_regime}. When a nondegenerate Gaussian limit is desired, the limits are assumed strictly positive.
\end{assumption}

\begin{assumption}[Gaussian limit for sampled-pair corrections]
\label{ass:pair_clt}
Conditionally on the realized finite population, the sampled-pair Horvitz--Thompson correction terms entering $\widehat A_{qq}$ and $\widehat B_{qq}$, when rescaled by $1/N^2$, admit a joint asymptotically Gaussian limit under the simple-random-sampling design. Equivalently, the relevant incomplete U-statistic / sampled-pair averages satisfy a nondegenerate multivariate CLT at the $\sqrt m$ scale. In addition, this convergence holds jointly with the centered linear correction terms entering all four corrected sums, normalized by $nN$ for the $pq$ terms and by $N^2$ for the $qq$ terms, at the $\sqrt m$ scale.
\end{assumption}

\begin{assumption}[Denominator nondegeneracy]
\label{ass:denom_nondeg}
The oracle and corrected denominators remain separated from zero with high probability along the asymptotic regime. In particular, for the events under consideration,
\[
B_{qq}^\star>0,\qquad B_{pq}^\star>0,
\]
and
\[
\widehat B_{qq}>0,\qquad \widehat B_{pq}>0
\]
with probability tending to one. 
\end{assumption}

\paragraph*{Non-degeneracy of denominators.}
Let $N_q^\star:=\sum_{j=1}^N I_j$ denote the (random) number of genuine $q$ points in the pool.
Then $B_{qq}^\star = N_q^\star(N_q^\star-1)$ (ordered-pair convention) and $B_{pq}^\star = n N_q^\star$.
Under \ref{ass:mixture_model} and $\epsilon<1$, $N_q^\star/N \to (1-\epsilon)$ almost surely, so denominators scale as $N^2$ and $nN$ with high probability.

\begin{remark}
Assumption~\ref{ass:srs} implies
\begin{equation}
\Pr(j\in L)=\frac{m}{N}, \qquad
\Pr(j,\ell\in L)=\frac{m(m-1)}{N(N-1)}, \quad j\neq \ell.    
\end{equation}
\end{remark}

Assumption~\ref{ass:denom_nondeg} is required to guarantee that the ratio expansions used for CC-MMD are well defined.

\begin{remark}
When the audit is not a simple random sample, Assumption~\ref{ass:srs} can be relaxed to requiring that the inclusion probabilities $\pi_j = \Pr(j \in L)$ and pair probabilities $\pi_{jj'} = \Pr(j, j' \in L)$ are strictly positive and known by design. Then, the corrected sums in Equations~\eqref{eq:Aqq_hat}--\eqref{eq:Apq_hat} and denominators in Equation~\eqref{eq:Bpq_hat} remain design-unbiased by replacing the uniform weights $N/m$ and $N(N-1)/\{m(m-1)\}$ with $1/\pi_j$ and $1/\pi_{jj'}$. This is the general Horvitz--Thompson construction \citep{horvitz1952generalization} of which our estimator is the equal-probability case. When some inclusion probabilities are very small, so that their inverses become large, self-normalized H\'ajek weights can be used. The extension is straightforward when the probabilities are controlled by the study design, and requires additional conditions (adequate overlap, stable weights, sufficiently accurate propensity estimation) when they must be estimated from data. Estimated inclusion probabilities and self-normalized Hájek weights do not generally preserve exact design-unbiasedness. The corresponding development for prediction-powered inference under informative labeling is given by Datta and Polson \citep{datta2025prediction}; a full treatment for kernel $U$-statistics raises design questions specific to the audit mechanism and is beyond the scope of this work.
\end{remark}

\begin{remark}[Computational bookkeeping for $\Delta_j$]
The latent residual $\Delta_j := I_j - \widetilde{I}_j$ is defined
for every pooled unit, but is observed only for $j\in L$, where the true label $I_j$ is observed. In implementations, unaudited entries, $j \notin L$, of the stored residual array may be set to zero, $\Delta_j := 0$, since the estimator accesses residuals only through audit-restricted sums. In population identities  and proofs, however, $\Delta_j$ always denotes the latent residual, not the zero-filled computational array. So this is purely a bookkeeping device and has no effect on the estimator's value or its statistical properties.
\end{remark}

\section{Extra Equations}

\begin{align}
\label{eq:Bpq_hat}
\widehat B_{pq} &:= n\left( \sum_{j=1}^N \widetilde I_j + \frac{N}{m}\sum_{j\in L} \Delta_j
\right), \\
\widehat B_{qq} &:= \sum_{j\neq j'} \widetilde I_j \widetilde I_{j'} + \frac{N}{m}\sum_{j\in L}\sum_{j'\neq j} \Delta_j \widetilde I_{j'} + \frac{N}{m}\sum_{j'\in L}\sum_{j\neq j'} \widetilde I_j \Delta_{j'} + \frac{N(N-1)}{m(m-1)} \sum_{\substack{j,j'\in L\\ j\neq j'}} \Delta_j \Delta_{j'}.  \nonumber
\end{align}

\section{Proof of Lemmas, Theorems, Corollaries in the main text}
\label{sec:all_proofs}

\begin{lemma} \label{lem:ht_identities}
Fix any finite population $\{(W_j,\widetilde I_j,I_j)\}_{j=1}^N$ and let $L$ be a simple random sample without replacement of size $m$.
Then for any real-value function $g:\{1,\dots,N\}\to\mathbb{R}$,
\begin{equation}
\label{eq:ht_single}
\mathbb{E}_L\!\left[\frac{N}{m}\sum_{j\in L} g(j)\right] = \sum_{j=1}^N g(j),
\end{equation}
and for any real-value function $h:\{1,\dots,N\}^2\to\mathbb{R}$ with $h(j,j)=0$,
\begin{equation}
\label{eq:ht_pair}
\mathbb{E}_L\!\left[\frac{N(N-1)}{m(m-1)}\sum_{\substack{j,j'\in L\\ j\neq j'}} h(j,j')\right] = \sum_{\substack{j,j'\in\{1,\dots,N\}\\ j\neq j'}} h(j,j').
\end{equation}
\end{lemma}

\begin{proof}
We prove Equation~\eqref{eq:ht_single}; Equation~\eqref{eq:ht_pair} is analogous.

Under simple random sampling without replacement, $\mathbb{P}(j\in L)=m/N$. Hence
\[
\mathbb{E}_L\!\left[\frac{N}{m}\sum_{j\in L} g(j)\right] = \frac{N}{m}\sum_{j=1}^N g(j)\,\mathbb{P}(j\in L) = \frac{N}{m}\sum_{j=1}^N g(j)\,\frac{m}{N} = \sum_{j=1}^N g(j).
\]
For Equation~\eqref{eq:ht_pair}, note that for $j\neq j'$,
$\mathbb{P}(j,j'\in L)=m(m-1)/[N(N-1)]$, and compute
\[
\mathbb{E}_L\!\left[\frac{N(N-1)}{m(m-1)}\sum_{j\neq j'\in L} h(j,j')\right] = \frac{N(N-1)}{m(m-1)}\sum_{j\neq j'} h(j,j')\,\mathbb{P}(j,j'\in L) = \sum_{j\neq j'} h(j,j').
\]
\end{proof}

\begin{proof}[Proof of Theorem~\ref{thm:unbiased_sums}]
We prove the claim for $\widehat A_{qq}$; the other three follow by identical arguments with $k\equiv 1$ or with the linear identity Equation~\eqref{eq:I_identity}.

Fix $j\neq j'$ and define $K_{jj'}:=k(W_j,W_{j'})$.
By Equation~\eqref{eq:I_identity},
\begin{equation}
\label{eq:expand_Aqq_star}
A_{qq}^\star = \sum_{j\neq j'} I_j I_{j'} K_{jj'} =
\sum_{j\neq j'}\Big(
\widetilde I_j \widetilde I_{j'} + \Delta_j \widetilde I_{j'}
+ \widetilde I_j \Delta_{j'}
+ \Delta_j \Delta_{j'}
\Big)K_{jj'}.
\end{equation}
The first term $\sum_{j\neq j'} \widetilde I_j \widetilde I_{j'}K_{jj'}$ is deterministic given the finite population (it does not depend on $L$).
For the second term in Equation~\eqref{eq:expand_Aqq_star}, define
\[
g(j) := \sum_{j'\neq j} \Delta_j \widetilde I_{j'} K_{jj'}.
\]
Conditionally on the realized finite population, $\Delta_j$ and hence $g(j)$ are fixed for every $j$ and do not depend on $L$.
Therefore,
\[
\frac{N}{m}\sum_{j\in L}\sum_{j'\neq j}\Delta_j \widetilde I_{j'} K_{jj'}
=
\frac{N}{m}\sum_{j\in L} \underbrace{\sum_{j'\neq j}(I_j-\widetilde I_j)\widetilde I_{j'}K_{jj'}}_{=:g^\star(j)},
\]
and by Lemma~\ref{lem:ht_identities} (single-sum identity Equation~\eqref{eq:ht_single}),
\[
\mathbb{E}_L\!\left[\frac{N}{m}\sum_{j\in L} g^\star(j)\right] 
= \sum_{j=1}^N g^\star(j)
= \sum_{j\neq j'} (I_j-\widetilde I_j)\widetilde I_{j'}K_{jj'}.
\]
This is exactly the second term on the right-hand side of Equation~\eqref{eq:expand_Aqq_star}. The third term is symmetric and follows by the same argument, exchanging roles of $j$ and $j'$.

Finally, for the fourth term define
\[
h(j,j') := (I_j-\widetilde I_j)(I_{j'}-\widetilde I_{j'})K_{jj'}
\quad\text{for } j\neq j',
\qquad h(j,j):=0.
\]
Then, by Lemma~\ref{lem:ht_identities} (pair-sum identity Equation~\eqref{eq:ht_pair}),
\[
\mathbb{E}_L\!\left[\frac{N(N-1)}{m(m-1)}\sum_{j\neq j'\in L} \Delta_j\Delta_{j'}K_{jj'}\right]
= \sum_{j\neq j'} (I_j-\widetilde I_j)(I_{j'}-\widetilde I_{j'})K_{jj'}.
\]
Combining the expectations of all four pieces shows $\mathbb{E}_L[\widehat A_{qq}]=A_{qq}^\star$.

For $\widehat A_{pq}$, apply the linear identity $I_j=\widetilde I_j+(I_j-\widetilde I_j)$ and Lemma~\ref{lem:ht_identities} for the single HT correction term. For $\widehat B_{qq}$ and $\widehat B_{pq}$ set $k\equiv 1$.
\end{proof}

\subsection{Design-based Asymptotic Normality [proofs]}
\label{app:design_clt_proofs}

\begin{proof}[Proof of Lemma~\ref{lem:ratio_linearization}.]
Starting from a common denominator,
\[
\frac{\widehat A}{\widehat B}-\frac{A}{B}
= \frac{B\widehat A-A\widehat B}{B\widehat B}
= \frac{B(\widehat A-A)-A(\widehat B-B)}{B\widehat B}.
\]
Add and subtract the first-order terms with denominator $B$:
\[
\frac{\widehat A-A}{\widehat B}
=
\frac{\widehat A-A}{B} - \frac{(\widehat A-A)(\widehat B-B)}{B\widehat B},
\]
and
\[
-\frac{A(\widehat B-B)}{B\widehat B}
= -\frac{A}{B^2}(\widehat B-B) + \frac{A(\widehat B-B)^2}{B^2\widehat B}.
\]
Substituting these two identities into the previous display yields
Equation~\eqref{eq:ratio_identity} with remainder
Equation~\eqref{eq:ratio_remainder_exact}. The bound
Equation~\eqref{eq:ratio_remainder_bound} follows immediately by the triangle inequality.
\end{proof}

Lemma~\ref{lem:ratio_linearization} shows that controlling the CC-MMD estimator reduces to controlling the fluctuations of its corrected numerators and denominators. We now study those fluctuations under the audit design, conditional on the realized finite population. The linear $pq$ terms are standard Horvitz--Thompson corrections for finite-population totals, so classical finite-population central limit theory applies. The quadratic $qq$ terms are more delicate: in addition to a linear Horvitz--Thompson component, they contain a second-order correction over sampled ordered pairs. We therefore analyze these two cases separately before combining them into a joint limit theorem.

For any finite-population sequence $\{g_j\}_{j=1}^N$, define
\[
\bar g := \frac{1}{N}\sum_{j=1}^N g_j,
\qquad
S_g^2 := \frac{1}{N-1}\sum_{j=1}^N (g_j-\bar g)^2.
\]

\begin{theorem}[Design-based CLT for the linear $pq$ terms]
\label{thm:design_clt_pq}
Under Assumptions~\ref{ass:kernel}, \ref{ass:srs}, \ref{ass:asymptotic_regime}, and \ref{ass:linear_variance}, condition on the realized samples $\{X_i\}_{i=1}^n$ and $\{(W_j,I_j,\widetilde I_j)\}_{j=1}^N$. Define
\[
g_j := \sum_{i=1}^n (I_j-\widetilde I_j)\,k(X_i,W_j),
\qquad
u_j := I_j-\widetilde I_j.
\]
Then, as $N,m\to\infty$ with $m/N\to \rho\in(0,1]$,
\begin{align}
\frac{\sqrt{m}}{nN}\,(\widehat A_{pq}-A_{pq}^\star)
&\Rightarrow
\mathcal{N}\bigl(0,(1-\rho)v_{pq}\bigr), \\
\frac{\sqrt{m}}{nN}\,(\widehat B_{pq}-B_{pq}^\star)
&\Rightarrow
\mathcal{N}\bigl(0,(1-\rho)w_{pq}\bigr).
\end{align}
In particular, both corrected $pq$ sums fluctuate at the $m^{-1/2}nN$ scale under the audit design.
\end{theorem}

\begin{proof}[Proof of Theorem~\ref{thm:design_clt_pq}]
From Equation~\eqref{eq:Apq_hat},
\[
\widehat A_{pq}
=
\sum_{i=1}^n\sum_{j=1}^N \widetilde I_j k(X_i,W_j)
+
\frac{N}{m}\sum_{j\in L}\sum_{i=1}^n (I_j-\widetilde I_j)k(X_i,W_j).
\]
Since
\[
A_{pq}^\star
=
\sum_{i=1}^n\sum_{j=1}^N I_j k(X_i,W_j)
=
\sum_{i=1}^n\sum_{j=1}^N \widetilde I_j k(X_i,W_j)
+
\sum_{j=1}^N g_j,
\]
we obtain
\[
\widehat A_{pq}-A_{pq}^\star
=
\frac{N}{m}\sum_{j\in L} g_j - \sum_{j=1}^N g_j.
\]
Equivalently, with $\tilde g_j := g_j/n$,
\[
\frac{1}{nN}(\widehat A_{pq}-A_{pq}^\star)
=
\frac{1}{m}\sum_{j\in L}\tilde g_j - \frac{1}{N}\sum_{j=1}^N \tilde g_j.
\]
Thus, $\frac{1}{nN}(\widehat A_{pq}-A_{pq}^\star)$ is exactly the Horvitz--Thompson error for the finite-population mean $\bar{\tilde g} = \frac{1}{N}\sum_{j=1}^N \tilde g_j$ under simple random sampling without replacement. Standard sampling theory gives
\[
\mathrm{Var}_L\!\left(\frac{1}{m}\sum_{j\in L}\tilde g_j\right)
=
\Bigl(1-\frac{m}{N}\Bigr)\frac{S_{\tilde g}^2}{m}.
\]
Under Assumption~\ref{ass:kernel}, $|k|\le K$ and $|I_j-\widetilde I_j|\le 1$, so
\[
|\tilde g_j| \leq K.
\]

Hence the finite population is uniformly bounded. When $0<\rho<1$ and $v_{pq}>0$, the first claim follows from the standard finite-population CLT \citep{fuller2011sampling,cochran1977sampling} and Assumption~\ref{ass:linear_variance}:
\[
\frac{\frac{1}{m}\sum_{j\in L} \tilde g_j - \bar{\tilde g}}
{\sqrt{(1-m/N)S_{\tilde g}^2/m}}
\Rightarrow \mathcal{N}(0,1).
\]
If $\rho=1$ or $v_{pq}=0$, the same claim follows instead from Chebyshev's inequality, since the centered quantity $\sqrt m(\widehat A_{pq}-A_{pq}^\star)/(nN)$ has variance $(1-m/N)S_{\tilde g}^2\to0$.

For $\widehat B_{pq}$, Equation~\eqref{eq:Bpq_hat} gives
\[
\frac{1}{nN}(\widehat B_{pq}-B_{pq}^\star)
=
\frac{1}{m}\sum_{j\in L} u_j - \frac{1}{N} \sum_{j=1}^N u_j,
\]
where $u_j=I_j-\widetilde I_j$ and $|u_j|\le 1$. The same argument (using the CLT or Chebyshev's inequality, as appropriate), with $u_j$ in place of $\tilde g_j$ and $S_u^2\to w_{pq}$, therefore applies, with variance
\[
\Bigl(1-\frac{m}{N}\Bigr)\frac{S_u^2}{m},
\]
which yields the second display after multiplying by $\sqrt m$ and using $S_u^2\to w_{pq}$.
\end{proof}

Theorem~\ref{thm:design_clt_pq} shows that the linear corrections behave exactly as in classical survey sampling: if the audit fraction vanishes, the correction becomes noisy, whereas a proportional audit budget $m\asymp N$ preserves the usual $N^{-1/2}$ scale. This is directly relevant for CC-MMD because the $pq$ component is one of the two nontrivial ratio terms, and Lemma~\ref{lem:ratio_linearization} turns the CLT above into first-order control of the corresponding ratio error.

We next consider the quadratic $qq$ terms. Unlike the $pq$ terms, these contain both a linear Horvitz--Thompson correction and a second-order sampled-pair correction. The latter is the only place where we invoke the additional high-level Gaussian-limit assumption from the appendix.

\begin{theorem}[Design-based asymptotic normality for the $qq$ corrections]
\label{thm:design_clt_qq}
Under Assumptions~\ref{ass:kernel}, \ref{ass:srs}, \ref{ass:asymptotic_regime}, \ref{ass:linear_variance_qq}, and \ref{ass:pair_clt}, condition on the realized samples $\{X_i\}_{i=1}^n$ and $\{(W_j,I_j,\widetilde I_j)\}_{j=1}^N$. Then
\[
\frac{\sqrt m}{N^2}\,(\widehat A_{qq}-A_{qq}^\star)
\Rightarrow \mathcal N(0,\sigma_{qq}^2),
\qquad
\frac{\sqrt m}{N^2}\,(\widehat B_{qq}-B_{qq}^\star)
\Rightarrow \mathcal N(0,\tau_{qq}^2),
\]
for finite design-based variances $\sigma_{qq}^2,\tau_{qq}^2\in[0,\infty)$, with possibly degenerate limits when the corresponding asymptotic variances vanish.
\end{theorem}

\begin{proof}
We first treat $\widehat A_{qq}-A_{qq}^\star$. Recall
\[
\Delta_j := I_j-\widetilde I_j,
\]
and the identity
\[
I_jI_{j'} =
  \widetilde I_j\widetilde I_{j'}
+ \Delta_j\widetilde I_{j'}
+ \widetilde I_j\Delta_{j'}
+ \Delta_j\Delta_{j'},
\qquad j\neq j',
\]
the oracle sum expands as
\begin{align}
A_{qq}^\star
&= \sum_{j\neq j'} I_jI_{j'}\,k(W_j,W_{j'}) \nonumber\\
&= \sum_{j\neq j'} \widetilde I_j\widetilde I_{j'}\,k(W_j,W_{j'})
+ \sum_{j\neq j'} \Delta_j\widetilde I_{j'}\,k(W_j,W_{j'}) \nonumber\\
&\quad + \sum_{j\neq j'} \widetilde I_j\Delta_{j'}\,k(W_j,W_{j'})
+ \sum_{j\neq j'} \Delta_j\Delta_{j'}\,k(W_j,W_{j'}).
\label{eq:Aqq_oracle_expand_revised}
\end{align}
Likewise, by definition of $\widehat A_{qq}$,
\begin{align}
\widehat A_{qq}
&=
\sum_{j\neq j'} \widetilde I_j\widetilde I_{j'}\,k(W_j,W_{j'})
+
\frac{N}{m}\sum_{j\in L}\sum_{j'\neq j} \Delta_j\widetilde I_{j'}\,k(W_j,W_{j'}) \nonumber\\
&\qquad+
\frac{N}{m}\sum_{j'\in L}\sum_{j\neq j'} \widetilde I_j\Delta_{j'}\,k(W_j,W_{j'})
+
\frac{N(N-1)}{m(m-1)}
\sum_{\substack{j,j'\in L\\j\neq j'}}
\Delta_j\Delta_{j'}\,k(W_j,W_{j'}).
\label{eq:Aqq_hat_expand_revised}
\end{align}
Subtracting \eqref{eq:Aqq_oracle_expand_revised} from \eqref{eq:Aqq_hat_expand_revised}, the proxy-only term cancels, yielding
\begin{align}
\widehat A_{qq}-A_{qq}^\star
&=
\left[
\frac{N}{m}\sum_{j\in L}\sum_{j'\neq j} \Delta_j\widetilde I_{j'}\,k(W_j,W_{j'})
-
\sum_{j\neq j'} \Delta_j\widetilde I_{j'}\,k(W_j,W_{j'})
\right]
\nonumber\\
&\qquad+
\left[
\frac{N}{m}\sum_{j'\in L}\sum_{j\neq j'} \widetilde I_j\Delta_{j'}\,k(W_j,W_{j'})
-
\sum_{j\neq j'} \widetilde I_j\Delta_{j'}\,k(W_j,W_{j'})
\right]
\nonumber\\
&\qquad+
\left[
\frac{N(N-1)}{m(m-1)}
\sum_{\substack{j,j'\in L\\j\neq j'}}
\Delta_j\Delta_{j'}\,k(W_j,W_{j'})
-
\sum_{j\neq j'} \Delta_j\Delta_{j'}\,k(W_j,W_{j'})
\right].
\label{eq:Aqq_diff_three_terms_revised}
\end{align}

Define the per-unit contribution
\begin{equation}
\label{eq:a_j_def_revised}
{
a_j
:=
2\sum_{j'\neq j} \Delta_j\widetilde I_{j'}\,k(W_j,W_{j'}),
}
\end{equation}
and the ordered-pair contribution
\begin{equation}
\label{eq:b_jj_def_revised}
b_{jj'}
:=
\Delta_j\Delta_{j'}\,k(W_j,W_{j'}),
\qquad j\neq j'.
\end{equation}
Then
\begin{equation}
\label{eq:Aqq_linear_pair_decomp_revised}
\widehat A_{qq}-A_{qq}^\star
=
\underbrace{\left(
\frac{N}{m}\sum_{j\in L} a_j - \sum_{j=1}^N a_j
\right)}_{=:U_{A,N}}
+
\underbrace{\left(
\frac{N(N-1)}{m(m-1)}
\sum_{\substack{j,j'\in L\\j\neq j'}}
b_{jj'}
-
\sum_{j\neq j'} b_{jj'}
\right)}_{=:V_{A,N}}.
\end{equation}

The first term $U_{A,N}$ is an ordinary Horvitz--Thompson error for the finite-population total $\sum_{j=1}^N a_j$. 
Define the rescaled per-unit contribution
\[
\tilde a_j := \frac{a_j}{N-1},
\]
so that, after dividing by $N(N-1)$,
\[
\frac{1}{N(N-1)}\,U_{A,N}
=
\frac{1}{m}\sum_{j\in L}\tilde a_j - \frac{1}{N}\sum_{j=1}^N \tilde a_j.
\]
This is the standard centered sample-mean form of a Horvitz--Thompson estimator for the finite-population mean $\bar{\tilde a} := N^{-1}\sum_{j=1}^N \tilde a_j$ under simple random sampling without replacement.

Under Assumption~\ref{ass:kernel},
\[
|\tilde a_j|
\le
\frac{2}{N-1} \sum_{j'\neq j}|\Delta_j|\,|\widetilde I_{j'}|\,|k(W_j,W_{j'})|
\le 2K.
\]

Thus the linear component is uniformly bounded after normalization. Standard sampling theory gives
\[
\Var_L\!\left(\frac{1}{m}\sum_{j\in L}\tilde a_j\right)
=
\Bigl(1-\frac{m}{N}\Bigr)\frac{S_{\tilde a}^2}{m},
\qquad
S_{\tilde a}^2 := \frac{1}{N-1}\sum_{j=1}^N(\tilde a_j-\bar{\tilde a})^2.
\]
Assumption~\ref{ass:linear_variance_qq} provides $S_{\tilde a}^2\to v_{qq}\in[0,\infty)$. 

A finite-population CLT for simple random sampling therefore yields
\[
\sqrt m\,\Bigl(\frac{1}{m}\sum_{j\in L}\tilde a_j - \bar{\tilde a}\Bigr)
\Rightarrow \mathcal N\bigl(0,(1-\rho)v_{qq}\bigr).
\]

Multiplying $U_{A,N}$ by $\sqrt m$ and rescaling by $N(N-1)$, since $N(N-1)/N^2\to 1$, by Slutsky's Theorem we have:
\begin{equation}
\label{eq:UA_qq_scaled}
\frac{\sqrt m}{N^2}\,U_{A,N}
\Rightarrow \mathcal N\bigl(0,\sigma_{A,\mathrm{lin}}^2\bigr),
\qquad
\sigma_{A,\mathrm{lin}}^2 := (1-\rho)v_{qq}.
\end{equation}

The second term $V_{A,N}$ is a second-order Horvitz--Thompson correction over sampled ordered pairs. Again by Assumption~\ref{ass:kernel},
\[
|b_{jj'}|
=
|\Delta_j\Delta_{j'}k(W_j,W_{j'})|
\le K.
\]
Assumption~\ref{ass:pair_clt} now implies that this sampled-pair term satisfies a design-based Gaussian limit at the $\frac{\sqrt m}{N^2}$ scale:
\[
\frac{\sqrt m}{N^2}\,V_{A,N}\Rightarrow \mathcal N(0,\sigma_{A,\mathrm{pair}}^2)
\]
for some finite $\sigma_{A,\mathrm{pair}}^2\in[0,\infty)$.

By the joint version of Assumption~\ref{ass:pair_clt}, the vector
\[
\left( \frac{\sqrt m}{N^2}\,U_{A,N},\frac{\sqrt m}{N^2}\,V_{A,N} \right)
\]
is asymptotically jointly Gaussian. Since
\[
\frac{\sqrt m}{N^2}\,(\widehat A_{qq}-A_{qq}^\star)
=
\frac{\sqrt m}{N^2}\,U_{A,N}+\frac{\sqrt m}{N^2}\,V_{A,N},
\]
the continuous mapping theorem implies
\[
\frac{\sqrt m}{N^2}\,(\widehat A_{qq}-A_{qq}^\star)
\Rightarrow
\mathcal N(0,\sigma_{qq}^2),
\]
where
\[
\sigma_{qq}^2
=
\Var(Z_{A,\mathrm{lin}}+Z_{A,\mathrm{pair}})
\]
for the limiting jointly Gaussian components.

The proof for $\widehat B_{qq}-B_{qq}^\star$ is identical with the kernel replaced by the constant function $k\equiv 1$. Specifically, define
\[
a_j^{(B)}
:=
2\sum_{j'\neq j} \Delta_j\widetilde I_{j'},
\qquad
b_{jj'}^{(B)}
:=
\Delta_j\Delta_{j'},
\]
and the rescaled per-unit contribution $\tilde a_j^{(B)} := a_j^{(B)}/(N-1)$, so that
\[
\widehat B_{qq}-B_{qq}^\star
=
\left(\frac{N}{m}\sum_{j\in L} a_j^{(B)}-\sum_{j=1}^N a_j^{(B)}\right)
+
\left(
\frac{N(N-1)}{m(m-1)}
\sum_{\substack{j,j'\in L\\j\neq j'}}
b_{jj'}^{(B)}
-
\sum_{j\neq j'} b_{jj'}^{(B)}
\right).
\]
Since $|\Delta_j|\le 1$ and $0\le \widetilde I_j\le 1$,
\[
|a_j^{(B)}|\le 2(N-1),
\qquad
|b_{jj'}^{(B)}|\le 1.
\]

Therefore,
\[
|\tilde a_j^{(B)}|\le 2.
\]
Applying the same argument, with $S_{\tilde a^{(B)}}^2\to w_{qq}\in[0,\infty)$ from Assumption~\ref{ass:linear_variance_qq} and the corresponding rescaled pair-CLT from Assumption~\ref{ass:pair_clt}, yields
\[
\frac{\sqrt m}{N^2}\,(\widehat B_{qq}-B_{qq}^\star)
\Rightarrow
\mathcal N(0,\tau_{qq}^2)
\]
for some finite $\tau_{qq}^2\in[0,\infty)$.
\end{proof}

Theorem~\ref{thm:design_clt_qq} shows that the quadratic correction remains asymptotically Gaussian despite the presence of a sampled-pair term. This matters because the $qq$ contribution is the only genuinely second-order component of CC-MMD. Together with Theorem~\ref{thm:design_clt_pq}, it shows that every corrected sum entering the estimator fluctuates at a scale proportional to $m^{-1/2}$ relative to its natural size ($N^2$ and $nN$, respectively). This ensures the final ratio estimator achieves the standard $m^{-1/2}$ convergence rate.

\begin{proof}[Proof of Theorem~\ref{thm:design_clt_full}.]
Theorem~\ref{thm:design_clt_pq} gives asymptotic normality of the linear $pq$ corrections, and Theorem~\ref{thm:design_clt_qq} gives asymptotic normality of the quadratic $qq$ corrections. By Assumption~\ref{ass:pair_clt}, the centered linear and sampled-pair components entering these four quantities converge jointly to a multivariate Gaussian limit at the $\sqrt m$ scale, after normalization by $nN$ for the $pq$ components and by $N^2$ for the $qq$ components. The displayed vector is obtained from those components by fixed linear combinations. The claim therefore follows from the Cram\'er--Wold device \citep{van2000asymptotic}.
\end{proof}

Theorem~\ref{thm:design_clt_full} is the design-based input needed for asymptotic analysis of CC-MMD itself. Combined with Lemma~\ref{lem:ratio_linearization}, it implies that the discrepancy between CC-MMD and the oracle estimator is asymptotically Gaussian at the $\sqrt m$ scale, provided the oracle and corrected denominators remain nondegenerate. In particular, when the audit fraction satisfies $m/N\to\rho>0$, the correction does not introduce a slower rate than the usual finite-sample fluctuations of empirical MMD.

\begin{proof}[Proof of Theorem~\ref{thm:cc_to_oracle}.]
Since the $pp$ term is observed exactly in both estimators,
\[
\widehat{\mathrm{MMD}}^2_{\mathrm{CC}}-\widehat{\mathrm{MMD}}^2_\star
=
\left(\frac{\widehat A_{qq}}{\widehat B_{qq}}-\frac{A_{qq}^\star}{B_{qq}^\star}\right)
-
2\left(\frac{\widehat A_{pq}}{\widehat B_{pq}}-\frac{A_{pq}^\star}{B_{pq}^\star}\right).
\]
Thus it suffices to control the two ratio differences.

Apply Lemma~\ref{lem:ratio_linearization} to the $qq$ ratio with
\[
A=\frac{A_{qq}^\star}{N^2},\quad
B=\frac{B_{qq}^\star}{N^2},\quad
\widehat A=\frac{\widehat A_{qq}}{N^2},\quad
\widehat B=\frac{\widehat B_{qq}}{N^2},
\]
and to the $pq$ ratio with
\[
A=\frac{A_{pq}^\star}{nN},\quad
B=\frac{B_{pq}^\star}{nN},\quad
\widehat A=\frac{\widehat A_{pq}}{nN},\quad
\widehat B=\frac{\widehat B_{pq}}{nN}.
\]
By Assumption~\ref{ass:mixture_model} with $\epsilon<1$, $N_q^\star/N\to 1-\epsilon$ a.s., so the denominators satisfy $B_{qq}^\star=\Theta_{\mathbb P}(N^2)$ and $B_{pq}^\star=\Theta_{\mathbb P}(nN)$, while boundedness of $k$ gives $A_{qq}^\star=O_{\mathbb P}(N^2)$ and $A_{pq}^\star=O_{\mathbb P}(nN)$.
Under Theorem~\ref{thm:design_clt_full},
\[
\widehat A-A=O_{\mathbb P_L}(m^{-1/2}),\qquad
\widehat B-B=O_{\mathbb P_L}(m^{-1/2}).
\]

By Assumption~\ref{ass:denom_nondeg}, the oracle denominators are strictly positive and the corrected denominators are bounded away from zero with probability tending to one. Therefore, the coefficients appearing in Lemma~\ref{lem:ratio_linearization} are well defined and $O_{\mathbb P_L}(1)$. It follows that each leading linear term is $O_{\mathbb P_L}(m^{-1/2})$, while the remainder term is
\[
O_{\mathbb P_L}\!\left(
|\widehat A-A|\,|\widehat B-B| + |\widehat B-B|^2
\right)
=
O_{\mathbb P_L}(m^{-1}),
\]
since both $|\widehat A-A|$ and $|\widehat B-B|$ are $O_{\mathbb P_L}(m^{-1/2})$. 
Hence
\[
\frac{\widehat A_{qq}}{\widehat B_{qq}}-\frac{A_{qq}^\star}{B_{qq}^\star}
=
O_{\mathbb P_L}(m^{-1/2}),
\qquad
\frac{\widehat A_{pq}}{\widehat B_{pq}}-\frac{A_{pq}^\star}{B_{pq}^\star}
=
O_{\mathbb P_L}(m^{-1/2}),
\]
and therefore
\[
\widehat{\mathrm{MMD}}^2_{\mathrm{CC}}-\widehat{\mathrm{MMD}}^2_\star
=
O_{\mathbb P_L}(m^{-1/2}).
\]
This proves Equation~\eqref{eq:cc_oracle_consistency}.

For the asymptotic normality statement, define
\[
U_m:=
\bigl(
\frac{\widehat A_{qq}}{N^2},\frac{\widehat B_{qq}}{N^2},\frac{\widehat A_{pq}}{nN},\frac{\widehat B_{pq}}{nN}
\bigr),
\qquad
U^\star:=
\bigl(
\frac{A_{qq}^\star}{N^2},\frac{B_{qq}^\star}{N^2},\frac{A_{pq}^\star}{nN},\frac{B_{pq}^\star}{nN}
\bigr).
\]
By Theorem~\ref{thm:design_clt_full},
\[
\sqrt m\,(U_m-U^\star)\Rightarrow \mathcal N(0,{\Sigma}_{\mathrm{design}})
\]
for some finite covariance matrix $\Sigma_{\mathrm{design}}$. Now define the smooth map
\[
\phi(a,b,c,d):= \frac{a}{b}-2\frac{c}{d}.
\]
Then
\[
\widehat{\mathrm{MMD}}^2_{\mathrm{CC}}-\widehat{\mathrm{MMD}}^2_\star
= \phi(U_m)-\phi(U^\star).
\]
Since $b$ and $d$ are bounded away from zero on the event guaranteed by Assumption~\ref{ass:denom_nondeg}, the map $\phi$ is continuously differentiable in a neighborhood of $U^\star$. Its gradient is
\[
\nabla\phi(a,b,c,d)
=
\left( \frac{1}{b},\; -\frac{a}{b^2},\; -\frac{2}{d},\; \frac{2c}{d^2} \right).
\]

By the law of large numbers for bounded oracle kernel
averages, $U^\star$ converges almost surely to a limit
with positive denominator coordinates.
Therefore, by the multivariate delta method,
\[
\sqrt m\,\bigl(\phi(U_m)-\phi(U^\star)\bigr)
\Rightarrow
\mathcal N(0,V_{\mathrm{design}}).
\]
This is exactly Equation~\eqref{eq:cc_oracle_clt}, with
\[
V_{\mathrm{design}}
:= \lim_{m\to\infty}\nabla\phi(U^\star)^\top {\Sigma}_{\mathrm{design}} \nabla\phi(U^\star).
\]
\end{proof}

\begin{proof}[Proof of Theorem~\ref{thm:consistency_population}.]
Decompose
\[
\widehat{\mathrm{MMD}}^2_{\mathrm{CC}}-\mathrm{MMD}^2(p,q;k)
=
\underbrace{\Bigl(\widehat{\mathrm{MMD}}^2_{\mathrm{CC}}-\widehat{\mathrm{MMD}}^2_\star\Bigr)}_{(\mathrm{I})}
+
\underbrace{\Bigl(\widehat{\mathrm{MMD}}^2_\star-\mathrm{MMD}^2(p,q;k)\Bigr)}_{(\mathrm{II})}.
\]

We first treat term $(\mathrm{I})$. By Theorem~\ref{thm:cc_to_oracle}, conditional on the realized samples,
\[
\widehat{\mathrm{MMD}}^2_{\mathrm{CC}}-\widehat{\mathrm{MMD}}^2_\star
\xrightarrow{\mathbb P_L} 0
\qquad\text{as } m\to\infty.
\]
Since Assumption~\ref{ass:asymptotic_regime} implies $m\to\infty$, it follows that
\[
\widehat{\mathrm{MMD}}^2_{\mathrm{CC}}-\widehat{\mathrm{MMD}}^2_\star
\xrightarrow{\mathbb P} 0
\]
under the full probability law. Thus $(\mathrm{I})\to 0$ in probability.

We next treat term $(\mathrm{II})$. Let
\[
N_q^\star := \sum_{j=1}^N I_j
\]
denote the number of genuine $q$-points in the contaminated pool. Under Assumption~\ref{ass:mixture_model}, the indicators $I_j$ are i.i.d.\ $\mathrm{Bernoulli}(1-\epsilon)$, so by the strong law of large numbers,
\[
\frac{N_q^\star}{N}\xrightarrow{\mathrm{a.s.}} 1-\epsilon.
\]
Because $\epsilon<1$, this implies
\[
N_q^\star \xrightarrow{\mathbb P} \infty,
\qquad
\mathbb P(N_q^\star\ge 2)\to 1.
\]
Hence the oracle denominators are eventually well defined with probability tending to one.

Conditional on the event $\{N_q^\star=r\}$, the retained points $\{W_j:I_j=1\}$ are i.i.d.\ draws from $q$. Therefore, conditional on $N_q^\star$, the oracle $qq$ term is the usual ratio-form U-statistic based on an i.i.d.\ sample of size $r$ from $q$:
\[
\frac{A_{qq}^\star}{B_{qq}^\star}
=
\frac{1}{N_q^\star(N_q^\star-1)}
\sum_{j\neq j'} I_jI_{j'}\,k(W_j,W_{j'}).
\]
Under Assumption~\ref{ass:kernel}, the kernel is bounded, so the law of large numbers for U-statistics yields
\[
\frac{A_{qq}^\star}{B_{qq}^\star}
\xrightarrow{\mathbb P}
\mathbb E[k(Y,Y')],
\qquad Y,Y'\overset{iid}{\sim} q.
\]

Likewise,
\[
\frac{A_{pp}^\star}{B_{pp}^\star}
=
\frac{1}{n(n-1)}\sum_{i\neq i'} k(X_i,X_{i'}),
\]
which is the standard U-statistic estimator of $\mathbb E[k(X,X')]$ based on $X_1,\dots,X_n\overset{iid}{\sim} p$. Again by the law of large numbers for U-statistics under bounded kernels,
\[
\frac{A_{pp}^\star}{B_{pp}^\star}
\xrightarrow{\mathbb P}
\mathbb E[k(X,X')],
\qquad X,X'\overset{iid}{\sim} p.
\]

For the cross term, conditional on $N_q^\star$, the oracle ratio can be written as
\[
\frac{A_{pq}^\star}{B_{pq}^\star}
=
\frac{1}{nN_q^\star}
\sum_{i=1}^n\sum_{j=1}^N I_j\,k(X_i,W_j).
\]
Equivalently, if $Y_1,\dots,Y_{N_q^\star}\overset{iid}{\sim} q$ denote the retained target points, then
\[
\frac{A_{pq}^\star}{B_{pq}^\star}
=
\frac{1}{nN_q^\star}\sum_{i=1}^n\sum_{\ell=1}^{N_q^\star} k(X_i,Y_\ell).
\]
Since $X_i\sim p$ and $Y_\ell\sim q$ are independent, boundedness of $k$ implies by the law of large numbers for averages of independent product samples that
\[
\frac{A_{pq}^\star}{B_{pq}^\star}
\xrightarrow{\mathbb P}
\mathbb E[k(X,Y)],
\qquad X\sim p,\;Y\sim q,\;X\perp Y.
\]

Combining the three convergences, we obtain
\[
\widehat{\mathrm{MMD}}^2_\star
= \frac{A_{pp}^\star}{B_{pp}^\star} + \frac{A_{qq}^\star}{B_{qq}^\star} - 2\frac{A_{pq}^\star}{B_{pq}^\star}
\xrightarrow{\mathbb P}
\mathrm{MMD}^2(p,q;k).
\]
Thus $(\mathrm{II})\to 0$ in probability.

Since both $(\mathrm{I})$ and $(\mathrm{II})$ converge to zero in probability, Slutsky's theorem yields
\[
\widehat{\mathrm{MMD}}^2_{\mathrm{CC}} \xrightarrow{\mathbb P} \mathrm{MMD}^2(p,q;k),
\]
which proves Equation~\eqref{eq:cc_consistency_population}.
\end{proof}

The result also clarifies the role of the assumptions. Assumptions~\ref{ass:srs}, \ref{ass:linear_variance}, and \ref{ass:pair_clt} control the audit correction, whereas Assumption~\ref{ass:mixture_model} is what turns the retained latent target points into an i.i.d.\ sample from $q$. Without that link, one could still recover the oracle finite-sample estimator, but there would be no reason for the oracle estimator to target $\mathrm{MMD}^2(p,q;k)$.

\section{Algorithm}
\label{sec:algorithm}

Algorithm~\ref{alg:pp_mmd_contam} computes a contamination-corrected estimate of $\mathrm{MMD}^2$ when the reference sample ${X}_{i=1}^n$ is clean, but the comparison sample is observed only through a contaminated pool ${W}_{j=1}^N$. Each pooled observation has a noisy membership score $\widetilde I_j$, while exact target-membership labels $I_j$ are available only on a small labeled subset $L$. The algorithm first forms residual corrections $\Delta_j=I_j-\widetilde I_j$ on the labeled subset, treating unlabeled residuals as zero. These residuals are then used in the estimators $\widehat A_{qq}$, $\widehat B_{qq}$, $\widehat A_{pq}$, and $\widehat B_{pq}$ to correct the kernel sums involving the contaminated $q$-pool. The clean-sample term $A_{pp}^\star/B_{pp}^\star$ is computed by the usual unbiased U-statistic. The final estimate combines the clean-clean, corrected pool-pool, and corrected cross-sample terms in the standard MMD decomposition, yielding a design-based contamination-corrected estimator that leverages noisy scores for all pooled observations and exact labels from the audited subset.

\begin{algorithm}[t]
\caption{Contamination-Corrected MMD with contaminated $q$-pool}
\label{alg:pp_mmd_contam}
\begin{algorithmic}[1]
\Require Clean sample $\{X_i\}_{i=1}^n$; pooled sample $\{(W_j,\widetilde I_j)\}_{j=1}^N$; labeled set $L$ with true $I_j$; kernel $k$
\Ensure $\widehat{\mathrm{MMD}}^2_{\mathrm{CC}}$
\State $\Delta_j \gets I_j-\widetilde I_j$ for $j\in L$; $\Delta_j\gets 0$ for $j\notin L$
\State $A_{pp}^\star \gets \sum_{i\neq i'} k(X_i,X_{i'})$;\quad $B_{pp}^\star \gets n(n-1)$
\State Compute $\widehat A_{qq}$ from Equation~\eqref{eq:Aqq_hat}; compute $\widehat B_{qq}$ from Equation~\eqref{eq:Bpq_hat}
\State Compute $\widehat A_{pq}$ from Equation~\eqref{eq:Apq_hat}; compute $\widehat B_{pq}$ from Equation~\eqref{eq:Bpq_hat}
\State \Return $\widehat{\mathrm{MMD}}^2_{\mathrm{CC}}
\gets A_{pp}^\star/B_{pp}^\star + \widehat A_{qq}/\widehat B_{qq} - 2\,\widehat A_{pq}/\widehat B_{pq}$
\end{algorithmic}
\end{algorithm}

\section{Details of Simulation Experiment 1}
\label{sec:appendix_simulation}

This section describes the simulation experiment used to evaluate the contamination-corrected MMD estimator. The goal of the experiment is to study how different estimators behave when one sample is contaminated by an unknown background distribution while only a small subset of observations are audited. In particular, we compare four estimators: (i) the oracle estimator that uses the true latent labels, (ii) the naive estimator that ignores contamination, (iii) an estimator that uses only the audited subset, and (iv) the proposed contamination-corrected estimator (CC--MMD).

\subsection{Reference and Target Distributions}

All experiments are conducted in $\mathbb{R}^2$. The target distribution $q$ is defined as a mixture of five Gaussian clusters arranged on a circle. Let
\[
\theta_k = \frac{2\pi k}{5}, \qquad k=0,1,2,3,4,
\]
and define the cluster centers
\[
\mu_k = \left(R\cos\theta_k,\; R\sin\theta_k\right),
\]
where the circle radius is $R=5$. The target distribution is the equally weighted Gaussian mixture
\[
q(x) = \frac{1}{5}\sum_{k=0}^{4} \mathcal{N}\!\left(x \mid \mu_k,\, \sigma_q^2 I_2\right),
\]
where $\sigma_q = 0.6$ and $I_2$ is the $2\times2$ identity matrix.

The clean reference distribution $p$ is identical to $q$ in the simulation. Thus any difference between the samples arises solely from contamination in the pooled sample rather than from a mismatch between $p$ and $q$.

For each simulation replicate we draw
\[
X_1,\dots,X_n \sim p
\]
with $n=400$ reference observations.

\subsection{Contamination Model}

The pooled sample is drawn from the mixture distribution
\[
r = (1-\epsilon) q + \epsilon z,
\]
where $z$ is a nuisance distribution representing background contamination and $\epsilon \in [0,1)$ controls the contamination level.

For each pooled observation $W_j$, a latent membership indicator
\[
I_j \sim \text{Bernoulli}(1-\epsilon)
\]
is first generated. If $I_j=1$, then $W_j$ is sampled from $q$; otherwise it is sampled from the nuisance distribution $z$. The pooled sample size is fixed at
\[
N = 400 .
\]

\subsection{Background Contamination Distribution}

The nuisance distribution $z$ is designed to produce structured contamination rather than uniform noise. Points are first drawn uniformly from the square
\[
[-8,8]^2 .
\]
These candidate points are then accepted with a probability that depends on their horizontal coordinate. Let $x$ denote the first coordinate of a candidate point. Define the normalized coordinate
\[
u = \frac{x + 8}{16} \in [0,1].
\]
The acceptance probability is
\[
\alpha(x) = 0.25 + 0.75\,u.
\]

This construction induces a spatial gradient in the contamination: background points are more likely to appear on the right side of the domain than on the left. Accepted points form the nuisance distribution $z$. Rejection sampling is used until the required number of nuisance samples is obtained.

The purpose of this gradient structure is to produce contamination that partially overlaps with the target clusters while remaining statistically distinguishable.

\subsection{Weak Proxy Classifier}

Each pooled observation $W_j$ is assigned a proxy score
\[
\widetilde I_j \in [0,1]
\]
intended to approximate the latent membership indicator $I_j$. This score represents the output of a weak classifier that attempts to distinguish target points from contaminants using simple geometric features.

Let
\[
r_j = \|W_j\|_2
\]
denote the radial distance of $W_j$ from the origin. Because the target clusters lie on a circle of radius $R=5$, points close to this radius are more likely to belong to the target distribution. The classifier therefore uses the radial deviation
\[
d_j = |r_j - 5|.
\]

In addition, the background contamination distribution exhibits a gradient in the horizontal coordinate. Points with large positive $x$ values are therefore more likely to be contaminants. The classifier incorporates this information using the feature
\[
x_j = W_{j,1}.
\]

A raw classifier score is computed as
\[
s_j
=
1.5
-
0.9\, d_j
-
0.20\, x_j
+
\eta_j ,
\]
where
\[
\eta_j \sim \mathcal{N}(0,\sigma_{\text{cls}}^2)
\]
is Gaussian noise with standard deviation $\sigma_{\text{cls}}=1.0$. The noise term intentionally weakens the classifier so that the proxy scores are imperfect and miscalibrated.

The final proxy score is obtained by passing the raw score through a logistic transformation:
\[
\widetilde I_j = \sigma(s_j)
= \frac{1}{1 + e^{-s_j}} .
\]

This construction produces scores that correlate with the true membership labels but are not calibrated probabilities. In particular, the proxy scores do not satisfy
\[
\widetilde I_j = \mathbb{P}(I_j=1 \mid W_j),
\]
and may systematically over- or under-estimate the true membership probability in different regions of the feature space.

\subsection{Audited Subset}

A small subset of pooled observations is selected for auditing. Let
\[
L \subset \{1,\dots,N\}
\]
denote the audited indices. The subset is sampled uniformly without replacement, with audit size
\[
m = 80 .
\]
For each $j \in L$, the true membership label $I_j$ is observed without error. For the remaining points, the analyst only observes the proxy score $\widetilde I_j$.

\subsection{Kernel and MMD Estimation}

All MMD estimators use a Gaussian radial basis function kernel
\[
k(x,y)
=
\exp\!\left(
-\frac{\|x-y\|^2}{2\sigma_k^2}
\right),
\]
with bandwidth
\[
\sigma_k = 1.5 .
\]

The following estimators are evaluated in each simulation replicate:

\begin{enumerate}
\item \textbf{Oracle MMD.}  
The oracle estimator uses the true membership indicators $I_j$ to restrict the pooled sample to genuine target observations.

\item \textbf{Naive contaminated MMD.}  
This estimator ignores contamination and computes MMD between the clean reference sample and the entire pooled sample.

\item \textbf{Audit-only estimator.}  
This estimator uses only audited points that are labeled as target samples. Because the audit subset is small, this estimator typically has high variance.

\item \textbf{Contamination-corrected estimator (CC--MMD).}  
This estimator combines the proxy scores $\widetilde I_j$ with the audited residuals $\Delta_j$ using the Horvitz--Thompson correction described in Section~\ref{sec:contamination_problem}.
\end{enumerate}

\subsection{Contamination Sweep}

To study the effect of contamination, experiments are repeated across a grid of contamination levels
\[
\epsilon \in \{0.0, 0.1, 0.2, 0.3, 0.4, 0.5, 0.6\}.
\]

For each value of $\epsilon$, the experiment is repeated multiple times (50) with independent random seeds. The resulting estimates are averaged to evaluate the bias and variability of the different estimators as contamination increases.

The experiment therefore isolates the impact of contamination on MMD estimation while maintaining a controlled setting in which the underlying target distribution remains fixed.

Across the contamination sweep, for seeds $42 + 10000 \times i + r$ where $r=0, \dots, 49$, the weak proxy classifier attains AUROC $0.78 \pm 0.03$ against the true membership indicators (mean ± standard deviation over the 50 replicates at each contamination level; AUROC is undefined at $\epsilon = 0$, where the pool contains no contaminants). The value is essentially constant in $\epsilon$.

\begin{figure}[th!]
    \centering
    \includegraphics[width=0.98\linewidth]{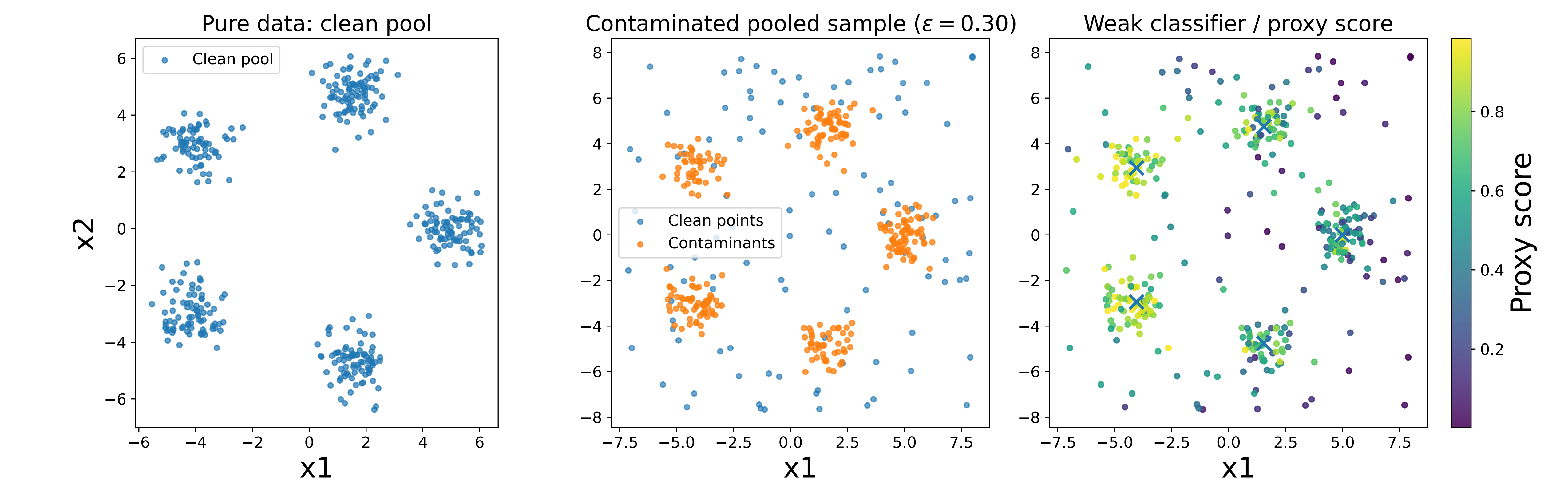}
    \caption{\textbf{Simulation setup for Experiment 1.} \textbf{Left:} the clean target distribution $q$, a mixture of five Gaussian clusters on a circle of radius 5. \textbf{Middle:} a contaminated pooled sample at $\epsilon = 0.30$, with clean points in blue and contaminants from $z$ in orange. \textbf{Right:} the same pooled sample colored by the weak proxy score $\widetilde{I}_j$ of Appendix~\ref{sec:appendix_simulation}.}
    \label{fig:placeholder}
\end{figure}

\section{Contaminated Linear Regression Experiment}
\label{sec:appendix_linear_regression}

This section describes the synthetic linear regression experiment used to evaluate the behavior of standard, robust, and MMD-based estimators under clustered contamination with partial auditing.

\subsection{Data Generating Process}

We consider a univariate linear regression model with additive noise. The clean (inlier) data are generated according to
\begin{equation}
\label{eq:linear_model}
Y = \beta_0^\star + \beta_1^\star X + \varepsilon,
\qquad
\varepsilon \sim \mathcal{N}(0, \sigma^2),
\end{equation}
where the covariate is sampled as
\[
X \sim \mathrm{Uniform}[-3,3].
\]
In all experiments, the true parameters are fixed to
\[
\beta_0^\star = 1.0, \qquad \beta_1^\star = 2.0, \qquad \sigma = 0.5.
\]

\paragraph{Contamination model.}
We observe a pooled dataset of size $N=300$ generated from the mixture
\begin{equation}
\label{eq:regression_contamination}
r = (1-\epsilon) q + \epsilon z,
\end{equation}
where $q$ denotes the clean linear model in Equation~\eqref{eq:linear_model} and $z$ is a nuisance distribution. The contamination rate is set to
\[
\epsilon = 0.40,
\]
so that $40\%$ of the observations are outliers.

Each observation $(X_j,Y_j)$ is associated with a latent indicator
\[
I_j \in \{0,1\},
\]
where $I_j=1$ denotes target population (an inlier) and $I_j=0$ denotes a contaminated point.

\paragraph{Clustered contamination.}
The nuisance distribution $z$ is constructed as a mixture of Gaussian clusters in the $(x,y)$ plane. Specifically, we sample $K=4$ cluster centers
\[
(\mu^{(x)}_k, \mu^{(y)}_k), \quad k=1,\dots,K,
\]
uniformly from rectangular domains
\[
\mu^{(x)}_k \sim \mathrm{Uniform}[-4.5,4.5],
\qquad
\mu^{(y)}_k \sim \mathrm{Uniform}[-8,8].
\]
Each contaminated observation is then generated by first selecting a cluster index uniformly and sampling
\[
X \sim \mathcal{N}(\mu^{(x)}_k, \sigma_x^2),
\qquad
Y \sim \mathcal{N}(\mu^{(y)}_k, \sigma_y^2),
\]
with
\[
\sigma_x = 0.35, \qquad \sigma_y = 0.75.
\]
This construction produces spatially localized clusters of outliers that do not follow the linear relationship, creating a challenging contamination pattern.

\subsection{Proxy Scores and Auditing}

For each observation, we construct a proxy score $\widetilde I_j \in [0,1]$ that approximates the latent indicator $I_j$. The proxy is based on residuals from an initial ordinary least squares (OLS) fit.

Let $(\widehat\beta_0^{\mathrm{OLS}}, \widehat\beta_1^{\mathrm{OLS}})$ denote the OLS estimate computed on the full contaminated dataset. Define residuals
\[
r_j = Y_j - \widehat\beta_0^{\mathrm{OLS}} - \widehat\beta_1^{\mathrm{OLS}} X_j.
\]
We compute a robust scale estimate using the median absolute deviation (MAD),
\[
\widehat\sigma_{\mathrm{MAD}} = 1.4826 \cdot \mathrm{median}\big(|r_j - \mathrm{median}(r)|\big),
\]
and define a standardized residual
\[
z_j = \frac{|r_j|}{\widehat\sigma_{\mathrm{MAD}}}.
\]
The proxy score is then constructed as
\[
\widetilde I_j
=
\sigma\big(1.5 - 1.2 z_j + \eta_j\big),
\qquad
\eta_j \sim \mathcal{N}(0,\,0.9^2),
\]
where $\sigma(\cdot)$ denotes the logistic function. This yields a weak, noisy classifier that correlates with inlier status but is not calibrated.

\paragraph{Audit subset.}
An audited subset $L \subset \{1,\dots,N\}$ is drawn uniformly without replacement with size
\[
|L| = m = 40.
\]
For indices $j \in L$, the true labels $I_j$ are observed without error. The audit rate is therefore approximately $13\%$.

\subsection{Estimators}

We estimate the regression parameters $(\beta_0,\beta_1)$ using five methods:

\paragraph{(i) Ordinary least squares (MSE).}
The classical estimator minimizes
\[
\frac{1}{N} \sum_{j=1}^N (Y_j - \beta_0 - \beta_1 X_j)^2.
\]

\paragraph{(ii) Huber regression.}
This estimator minimizes the Huber loss
\[
\frac{1}{N} \sum_{j=1}^N \rho_\delta(Y_j - \beta_0 - \beta_1 X_j),
\]
with threshold $\delta = 1.25$.

\paragraph{(iii) Tukey biweight regression.}
This method uses Tukey's redescending loss applied to standardized residuals based on a robust scale estimate (MAD). The tuning constant is set to $c = 4.685$.

\paragraph{(iv) Simple MMD estimator.}
We adopt a simulation-based approach. For each candidate parameter $\beta = (\beta_0,\beta_1)$, we generate a synthetic sample
\[
(X_i^{(\beta)}, Y_i^{(\beta)}),
\quad i=1,\dots,n_{\mathrm{model}},
\]
with
\[
X_i^{(\beta)} \sim \mathrm{Uniform}[-3,3],
\qquad
Y_i^{(\beta)} = \beta_0 + \beta_1 X_i^{(\beta)} + \varepsilon_i,
\]
where $\varepsilon_i \sim \mathcal{N}(0,\sigma^2)$. The model sample size is
\[
n_{\mathrm{model}} = 10{,}000.
\]
The parameter is estimated by minimizing
\[
\widehat{\mathrm{MMD}}^2\big(q_\beta, r\big),
\]
computed using a Gaussian kernel.

\paragraph{(v) Contamination-corrected MMD (CC--MMD).}
This method replaces the naive MMD objective with the contamination-corrected estimator described in Section~\ref{sec:contamination_problem}. It combines:
\begin{itemize}
\item proxy scores $\widetilde I_j$,
\item audited residuals $\Delta_j = I_j - \widetilde I_j$,
\end{itemize}
to construct a Horvitz--Thompson correction of the kernel U-statistics. The resulting objective approximates the oracle MMD between $q_\beta$ and the clean distribution $q$.

\subsection{Kernel and Optimization}

All MMD-based methods use a Gaussian kernel
\[
k((x,y),(x',y'))
=
\exp\!\left(
-\frac{\|(x,y)-(x',y')\|^2}{2\sigma_k^2}
\right),
\]
with bandwidth
\[
\sigma_k = 1.0.
\]

All estimators are optimized using the L-BFGS-B algorithm. The initialization for all methods is the OLS estimate computed on the full contaminated dataset. The model-based MMD objectives use a fixed synthetic sample (shared across iterations) to reduce stochastic variability and ensure a smooth optimization landscape.

\subsection{Summary of Parameters}

The full experimental configuration is:
\[
\begin{aligned}
&\text{Sample size:} && N = 300, \\
&\text{Contamination rate:} && \epsilon = 0.40, \\
&\text{Audit size:} && m = 40, \\
&\text{True parameters:} && (\beta_0^\star, \beta_1^\star) = (1.0, 2.0), \\
&\text{Noise standard deviation:} && \sigma = 0.5, \\
&\text{Number of contamination clusters:} && K = 4, \\
&\text{Model sample size:} && n_{\mathrm{model}} = 10{,}000, \\
&\text{Kernel bandwidth:} && \sigma_k = 1.0.
\end{aligned}
\]

This setup induces structured, non-Gaussian contamination together with imperfect proxy information, providing a challenging testbed for comparing classical robust regression methods and MMD-based approaches.

\paragraph{Multiple realizations under contamination.}

Figure~\ref{fig:regression_multi} shows 15 independent realizations of the contaminated regression setting described above. Each panel is generated using the same data-generating process and parameter configuration, with independent random seeds $42, 1042, \dots, 14042$. The achieved proxy classifier's AUROC across the 15 realizations is $0.6593 \pm 0.1332$. The plots display the observed contaminated data (orange points), the target population (blue points), the audited subset, and the fitted regression lines obtained from each method (MSE, Huber, Tukey, naive MMD, audit-only, and CC--MMD).

Across realizations, the qualitative behavior of the estimators is consistent. The least-squares estimator is systematically biased toward the contamination clusters, while robust methods such as Huber partially mitigate this effect but still exhibit noticeable bias. Tukey's biweight shows high variability, often underestimating the slope due to aggressive downweighting. The naive MMD estimator captures the overall data geometry but remains sensitive to contamination, leading to systematic deviations. The audit-only estimator shows no systematic bias but varies markedly from panel to panel, reflecting the small number of audited inliers it relies on. In contrast, the CC--MMD estimator consistently aligns with the latent inlier structure, producing estimates that closely track the true regression line across all panels. This stability demonstrates the advantage of combining distribution matching with audit-based correction.

\begin{figure}[h!]
    \centering
    \includegraphics[width=0.9\linewidth]{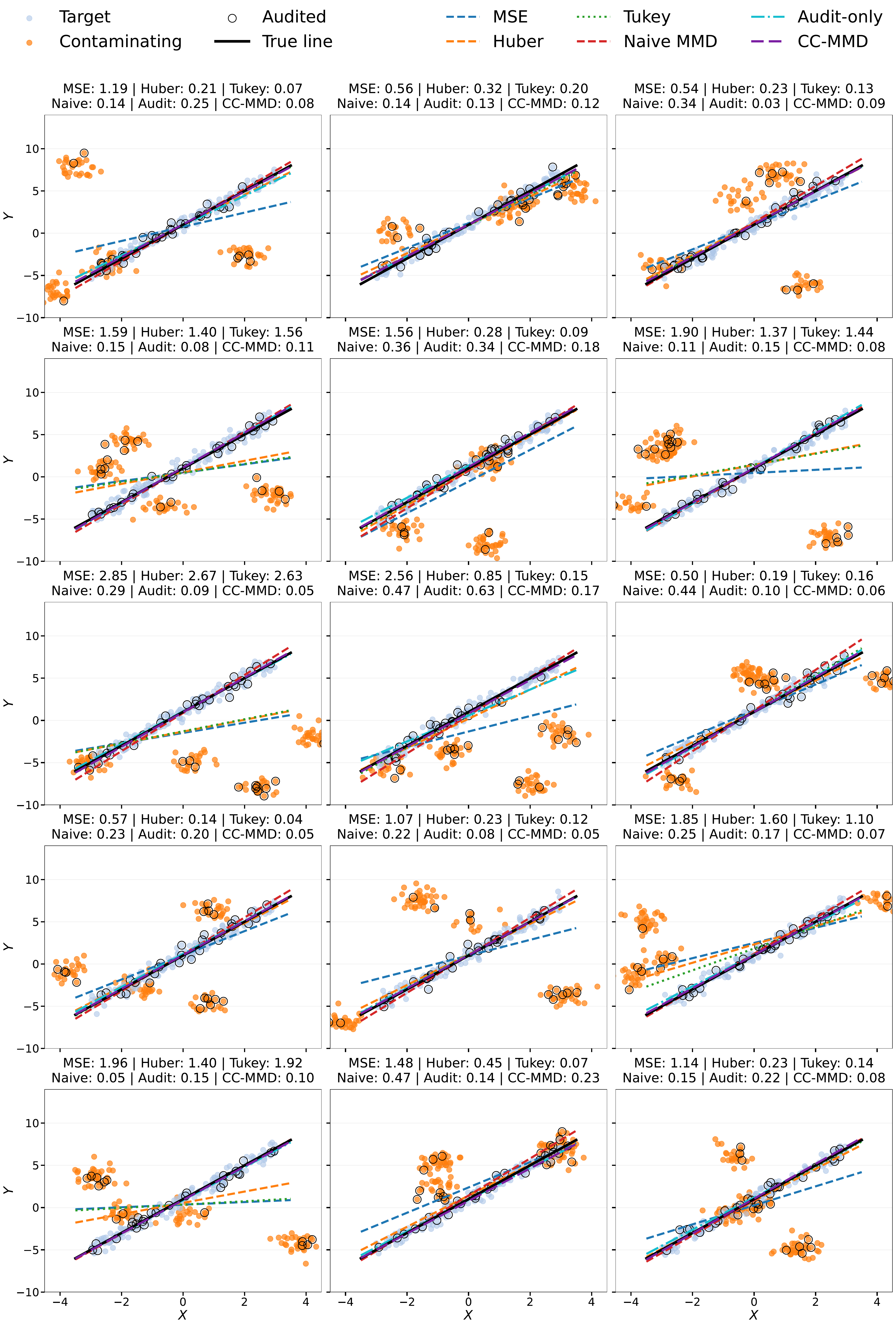}
    \caption{Fifteen independent realizations of the contaminated linear regression problem. Each panel shows the observed data, target population (latent inliers), audited points, and fitted lines for different estimators. While classical and robust methods exhibit bias or variability under clustered contamination, the CC--MMD estimator consistently recovers the true linear relationship across realizations. All realizations are generated randomly.}
    \label{fig:regression_multi}
\end{figure}

\paragraph{Parameter recovery over 200 independent realizations.}

Figure~\ref{fig:regression_200} shows the full distribution of the estimates summarized in Table~\ref{tab:regression} over 200 independent realizations of the setting above, using seeds $42, 43, \dots, 241$, and achieving a proxy AUROC of $0.6617 \pm 0.1112$. The panels show the estimated intercept, slope, and $\ell_2$ parameter error. Faint points are individual realizations; the marker and bar give the mean $\pm$ one standard deviation, and dashed lines mark the true values.

\begin{figure}[h!]
    \centering
    \includegraphics[width=0.9\linewidth]{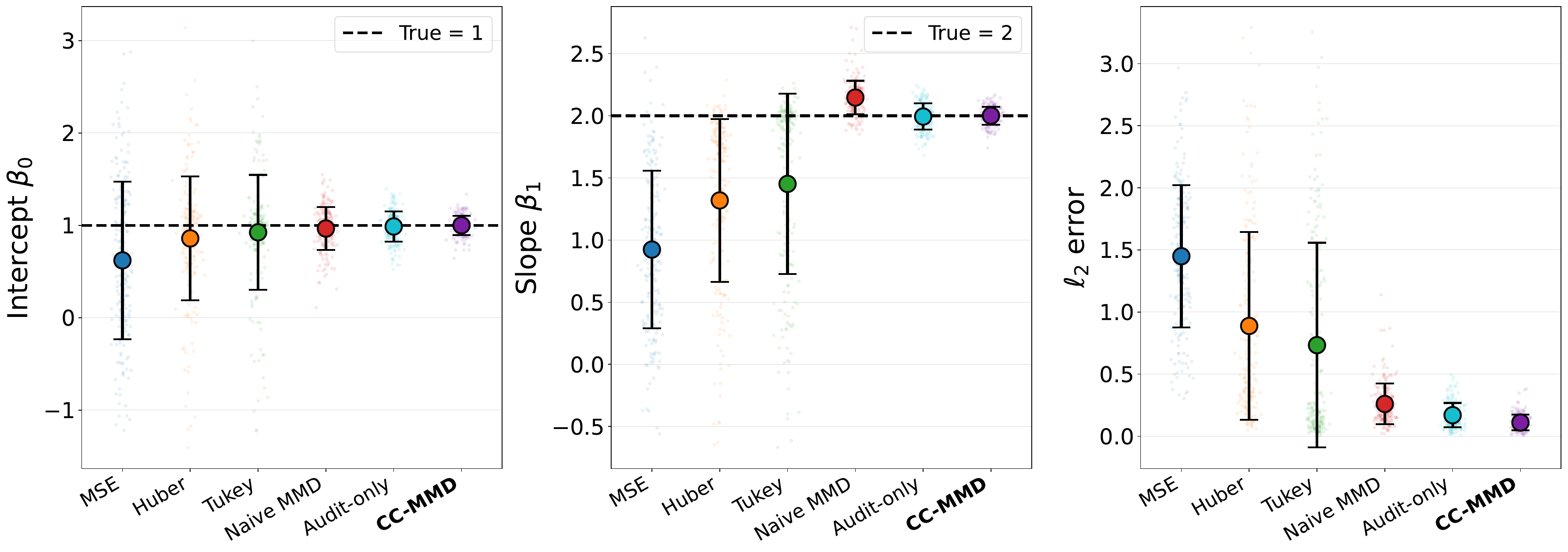}
    \caption{\textbf{Parameter recovery over 200 independent realizations.} Intercept (\textbf{left}), slope (\textbf{middle}) and $\ell_2$ parameter error (\textbf{right}) for every estimator at $\varepsilon=0.40$, $N=300$, $m=40$. Faint points show each of the 200 realizations; the filled marker and error bar give the mean $\pm$ one standard deviation across realizations, and dashed lines mark the true values $\beta_0^\star=1$ and $\beta_1^\star=2$. CC--MMD has both the smallest bias and the smallest spread on each parameter.}
    \label{fig:regression_200}
\end{figure}

\newpage

\section{Concrete Regression under Structured Contamination}
\label{sec:exp_concrete_full}

\paragraph{Dataset and contamination construction.}
We evaluate \textsc{CC-MMD} on the UCI Concrete Compressive Strength dataset~\citep{yeh1998modeling}, which contains measurements of compressive strength as a function of seven material components. The full dataset contains $1{,}030$ samples.

To construct a realistic contamination scenario, we partition the data based on curing time. The target distribution, $q$, consists of samples cured for exactly 28 days ($N_q = 425$), while the contaminating distribution, $z$, consists of samples cured for $\leq 7$ days ($N_z = 262$). The curing time variable is removed from the feature set, ensuring that contamination is not directly observable.

\begin{figure}[h]
  \centering
    \includegraphics[width=0.8\linewidth]{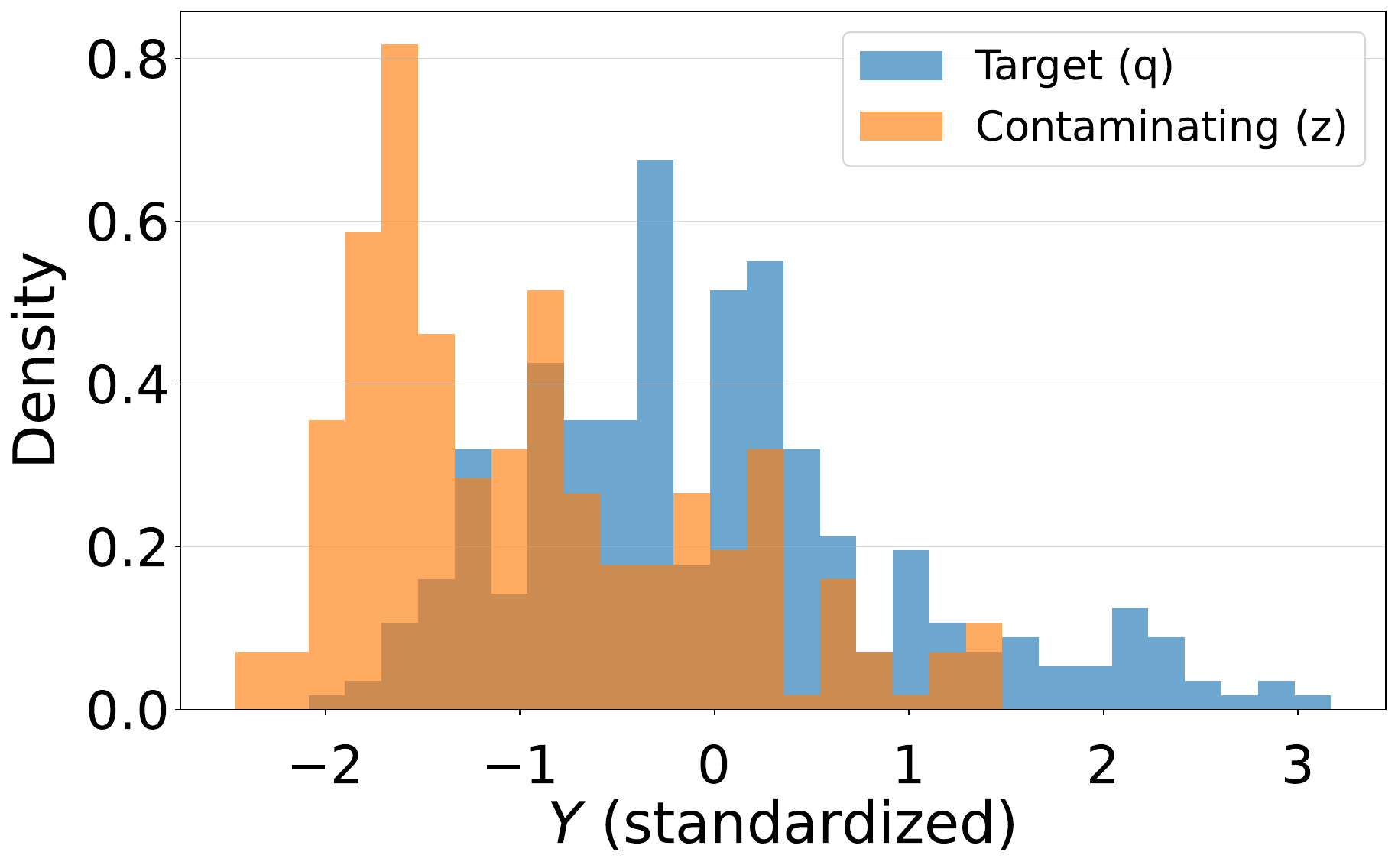}
    \caption{\textbf{Training pool composition at $\epsilon = 0.5$.}
      Contaminating samples (early-cured concrete) occupy the lower tail of
      the strength distribution, pulling the empirical pool mean downward.
      Both histograms are density-normalized.}
  \label{fig:exp4_context}
\end{figure}

This partition induces a structured shift in the joint distribution $(X,Y)$: early-cured concrete exhibits systematically lower compressive strength, resulting in a downward shift in $\mathbb{E}[Y \mid X]$. Importantly, this shift is moderate in magnitude and overlaps with the support of the clean distribution, meaning contaminants are not easily separable via residual-based methods. Figure~\ref{fig:exp4_context} visualizes the split between the latent clean distribution and the contaminating distribution.

\paragraph{Training pool construction.}
We fix the number of clean target samples at $N_q = 300$, drawn from $q$, and augment with contaminants drawn with replacement from $z$. The contamination level is varied as
\[
\epsilon \in \{0.0, 0.1, 0.2, 0.3, 0.4, 0.5, 0.6\}.
\]
For each $\epsilon$, we add $\lfloor N_q\epsilon/(1-\epsilon)\rfloor$ contaminants, so that the total pool size is $\lfloor N_q/(1-\epsilon)\rfloor$ and the fraction of contaminants equals $\epsilon$ by construction.

A clean test set of size $125$ is held out from $q$ and used for all evaluations.

\paragraph{Audit design and proxy estimation.}
We assume access to a small audit subset of size equal to $20\%$ of the total contaminated training dataset ($m = {\rm round}(0.20 * N)$), sampled uniformly from the pool. It ranges from $m=60$ at $\epsilon=0$ to $m=150$ at $\epsilon=0.6$. 

A Random Forest classifier, with $100$ trees and maximum depth 4, is trained on the joint $(X,Y)$ feature space of this noisy audit set to produce proxy membership scores $\tilde{I}_i$ for all training samples.
To simulate realistic conditions of having an imperfect proxy classifier, $\eta \approx 20\%$ of audited labels are randomly flipped before fitting the proxy. Note that the audit labels are flipped only for proxy fitting purposes, while the residual corrections $\Delta_j$ always use the true audited memberships. The resulting proxy AUROC against ground-truth membership on the full pool rises with the contamination level, from $0.758 \pm 0.091$ at $\epsilon = 0.1$ to $0.875 \pm 0.027$ at $\epsilon = 0.6$ (mean $\pm$ standard deviation) over 20 seeds ($42, 1042, \dots, 19042$). At the demonstration setting $\epsilon = 0.5$, seed 42, used in Figures~\ref{fig:exp4_scatter} and \ref{fig:exp4_qq}, achieved proxy AUROC is $0.8633$. AUROC is undefined at $\epsilon = 0$, where the pool contains no contaminants. For reference, when $\eta = 0$, the proxy classifier's AUROC in Figure~\ref{fig:exp4_quant} is $0.86-0.94$, while in Figures~\ref{fig:exp4_scatter} and \ref{fig:exp4_qq}, proxy AUROC is $0.9498$.

See Appendix~\ref{app:additional-results} for a discussion on the way proxy scores are obtained in Experiments 3 (Appendix~\ref{sec:exp_concrete_full}) and 4 (Appendix~\ref{sec:appendix_generative}) together with additional results to support the discussion.

\paragraph{Model and training details.}
All methods use the same MLP architecture with hidden layers $[64,64,32]$ and ReLU activations, optimized with Adam (learning rate $10^{-3}$, weight decay $10^{-4}$). Pointwise methods (Oracle, Audit-only, MSE, Huber) are trained for 300 epochs, while Tukey uses an additional 50-epoch MSE warm-up (350 total).
For MMD-based methods, training proceeds with a 50-epoch Huber warm-up followed by optimisation of the MMD objective for 400 epochs (450 total). The kernel bandwidth is set as
\[
\sigma = \frac{\sigma_{\mathrm{med}}}{D}, \quad D = d+1 = 8,
\]
where $\sigma_{\mathrm{med}}$ is the median pairwise distance in the joint $(X,Y)$ space. The target variable $Y$ is rescaled by $\sqrt{d_{\mathrm{feat}}} \approx 2.65$ prior to concatenation.

\paragraph{Methods.}
We compare the following MLP regressors, all using the same architecture and optimizer, as specified in the paragraph above:

\begin{itemize}
  \item \textbf{Oracle}: MSE loss on the true target subset of the pool (requires knowledge of true membership $I_i$); serves as an unattainable performance upper bound.
  \item \textbf{Audit-only}: MSE loss restricted to the labeled target points in the audit subset (approximately 60 inlier samples in expectation: $m \cdot (1-\epsilon) \approx 0.2 \times 300 = 60$); represents the strongest purely label-guided baseline achievable with the audit budget.
  \item \textbf{MSE / Huber / Tukey}: Standard mean-squared error, Huber loss ($\delta = 1$), and Tukey biweight loss ($c = 4.685$) applied to the full contaminated pool; 50 MSE warm-up epochs precede Tukey training.
  \item \textbf{Na\"{i}ve MMD}: MMD loss matching the joint $(X,Y)$ distribution of model samples to the full contaminated pool, without any correction for membership.
  \item \textbf{CC-MMD}: Horvitz--Thompson-corrected MMD loss, Equation~\eqref{eq:mmd_cc}, using bandwidth $\sigma = \sigma_{\mathrm{med}} / D$, where $\sigma_{\mathrm{med}}$ is the median pairwise distance in the $D$-dimensional joint feature--target space and $D = d + 1 = 8$. The proxy score $\tilde{I}_i$ is computed by a Random Forest classifier trained on the audit subset using the joint $(X, Y)$ feature space.
\end{itemize}

\paragraph{Evaluation protocol.}
All results are averaged over $n_{\mathrm{rep}} = 20$ random seeds. We report:
\begin{itemize}
    \item Test mean squared error (MSE) on clean data,
    \item Test MMD$^2$ between predicted and true output distributions, computed with a Gaussian kernel whose bandwidth is estimated via the median heuristic on the combined $1-D$ sample $\{Y_{\mathrm{test}}, \hat{Y}\}$ (standard heuristic without the $\div D$ normalisation used during training).
\end{itemize}

\paragraph{Quantitative results.}
Figure~\ref{fig:exp4_quant} shows that at $\epsilon=0$, all methods achieve comparable performance (MSE $\approx 0.15$), indicating that \textsc{CC-MMD} does not introduce bias in the absence of contamination. Audit-only training plateaus at higher error ($\approx 0.37$), reflecting limited sample size.

As $\epsilon$ increases, standard methods degrade sharply. MSE and Huber reach $\approx 0.76$--$0.77$ at $\epsilon=0.6$, indicating that moderate contamination is not mitigated by residual-based robustness. Tukey performs worse ($\approx 0.95$), consistent with inflated scale estimates causing indiscriminate downweighting.
Na\"{i}ve MMD similarly fails ($\approx 0.98$), as it aligns the model with the contaminated distribution.

\textbf{CC-MMD consistently outperforms all non-oracle baselines}. At $\epsilon=0.5$, it achieves $0.47 \pm 0.10$ MSE, representing a $34\%$ reduction relative to the next best baseline. At $\epsilon=0.6$, the improvement remains substantial ($25\%$ reduction). Gains are largest at intermediate contamination levels ($\epsilon=0.2$--$0.4$), reaching up to $39\%$.

In distribution space, CC-MMD maintains near-zero test MMD$^2$ across all $\epsilon$, closely matching the oracle, while all other methods trained on the full contaminated dataset, diverge monotonically.

\paragraph{Qualitative analysis.}
Figure~\ref{fig:exp4_scatter} shows true-versus-predicted compressive strength at $\epsilon=0.5$ for a representative seed. Oracle and CC--MMD predictions concentrate tightly around the $y=\hat{y}$ diagonal, indicating that they recover both the level and spread of the clean target relationship. In contrast, MSE, Huber, Tukey, and na\"ive MMD exhibit a systematic downward bias, most visibly for the largest true values. In particular, high-strength samples are underpredicted because early-cured contaminants pull the pooled conditional mean $\mathbb{E}_{\mathrm{pool}}[Y \mid X]$ below the clean target relationship. The similar failure of na\"ive MMD shows that the benefit does not come merely from replacing a pointwise loss with a distributional objective, and the MMD objective itself must be corrected for contamination.

The QQ plots in Figure~\ref{fig:exp4_qq} provide a complementary distributional view. CC--MMD closely matches the marginal test distribution of $Y$, with predicted quantiles lying near the diagonal ($R^2=0.873$). By contrast, the contaminated baselines compress the predicted range, underestimating high-strength samples and overestimating low-strength samples ($R^2 \leq 0.755$). Thus, CC--MMD corrects not only mean-level prediction error but also the broader distributional distortion induced by contamination.

\begin{figure}[th!]
  \centering
  \includegraphics[width=\linewidth]{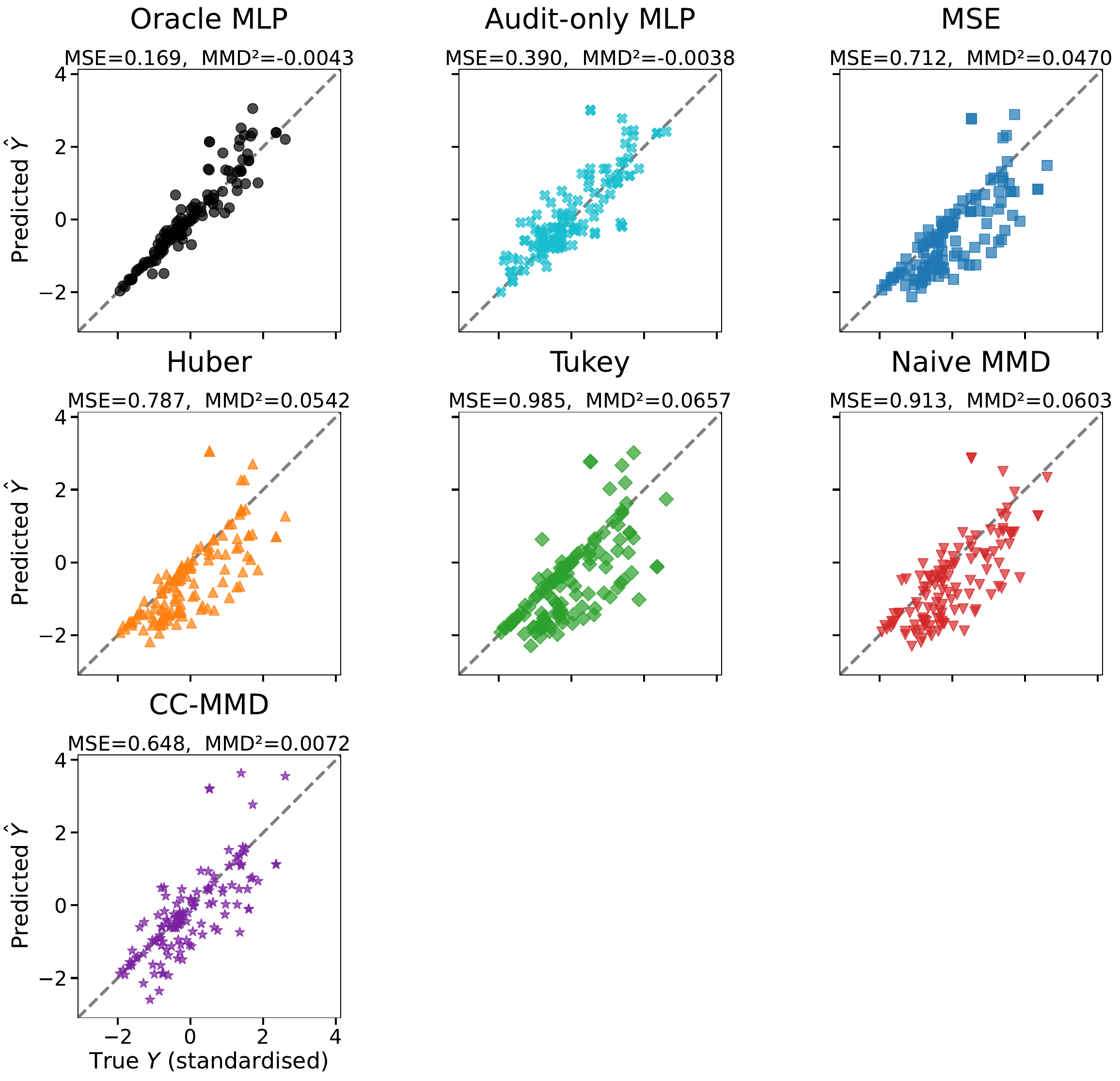}
  \caption{\textbf{True vs.\ predicted compressive strength on the clean test set at $\varepsilon = 0.5$.} Each panel reports MSE and MMD$^2$ for that method. Oracle and CC-MMD align closely with the diagonal, while the contaminated baselines exhibit a systematic downward bias in the upper range.}
  \label{fig:exp4_scatter}
\end{figure}

\begin{figure}[th!]
  \centering
  \includegraphics[width=\linewidth]{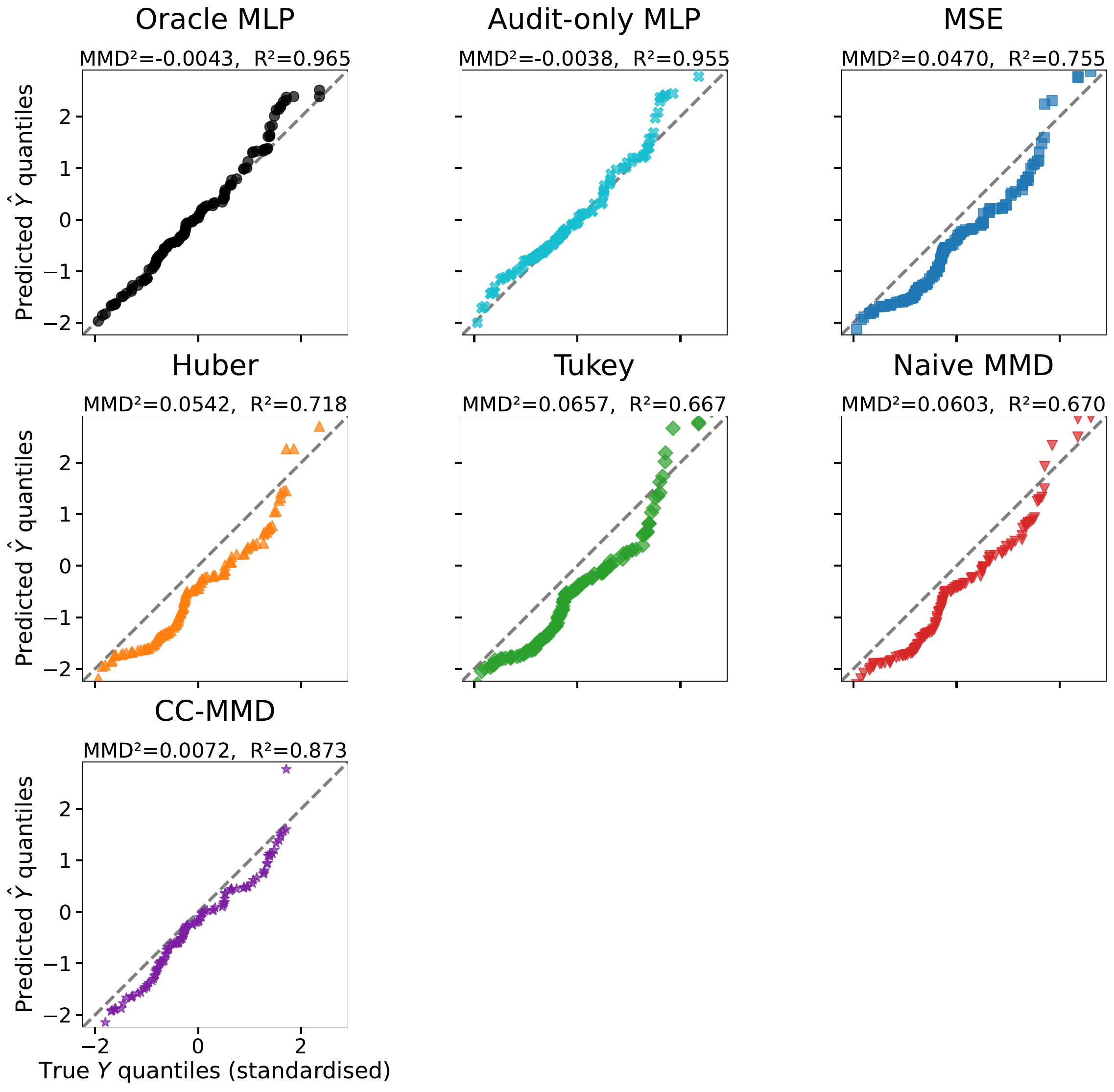}
  \caption{\textbf{Predicted-vs-true output quantile plots at $\varepsilon = 0.5$.} A perfect predictor lies on the diagonal. CC-MMD recovers the full marginal distribution of $Y$, whereas contaminated baselines compress and shift the predicted quantiles, revealing distributional bias beyond mean-level degradation.}
  \label{fig:exp4_qq}
\end{figure}

\paragraph{Discussion.}
This experiment highlights a failure mode distinct from classical outlier robustness: contamination induces a \emph{structured shift within the support}, rather than extreme residuals. In this regime, both residual-based robust losses and na\"{i}ve distribution matching fail.

\textsc{CC-MMD} succeeds because it reconstructs the clean distribution via design-based correction, leveraging both proxy scores and audit information without requiring calibrated probabilities or assumptions on the contaminating distribution.

\section{High-Dimensional Generative Modeling under Contamination}
\label{sec:appendix_generative}

\paragraph{Overview.}
We evaluate CC--MMD in a high-dimensional generative modeling setting where the training data are contaminated with samples from a different distribution. The goal is to assess whether CC--MMD can guide a generative model to recover the clean target distribution despite substantial contamination in the training pool. We consider MNIST~\citep{deng2012mnist} as the target distribution and introduce contamination from K-MNIST~\citep{clanuwat2018deep} and Omniglot~\citep{lake2015human} (not simultaneously).

\paragraph{Datasets.}
The target dataset is MNIST, consisting of grayscale images of handwritten digits of size $28 \times 28$. The contaminating datasets are:
\begin{itemize}
    \item \textbf{K-MNIST:} Kuzushiji-MNIST, containing Japanese characters with similar resolution and format.
    \item \textbf{Omniglot:} A dataset of handwritten characters from diverse alphabets.
\end{itemize}
We use the standard MNIST training set to construct the contaminated pool and the standard MNIST test set ($10{,}000$ samples) for evaluation.

Before creating a contaminated dataset, both target and contaminating datasets undergo identical, independent preprocessing so that images live in a common feature space. Omniglot images are first resized to $28 \times 28$, while MNIST and K-MNIST are already at this resolution and require no resizing. Each $28 \times 28$ grayscale image is converted to a float tensor in $[0, 1]$, clamped to $[\alpha, 1-\alpha]$ with $\alpha = 0.05$ to bound values away from the endpoints, mapped to logit space via $x \mapsto \log\!\big(x / (1-x)\big)$ element-wise, and flattened to a $784$-dimensional vector. We use continuous intensities throughout (no binarization). Unlike the other two datasets, Omniglot characters are in black pixels with background in white pixels and we do not invert this. This makes the contamination problem harder with Omniglot than with K-MNIST.

\paragraph{Contaminated training pool.}
After independent and identical preprocessing of the target and contaminating dataset, we construct a contaminated training dataset of size $N=10{,}000$. A fraction $\epsilon=0.60$ of the pool consists of contaminant samples drawn from either K-MNIST or Omniglot, while the remaining $1-\epsilon$ fraction consists of clean MNIST samples. Samples are drawn uniformly at random without replacement from each source dataset, with no explicit class-balancing constraint. After sampling, the $N = 10{,}000$ images are concatenated and a random permutation is applied to remove any ordering by source.

\paragraph{Audit subset and proxy estimation.} \label{sec:audit_and_proxy_estimation}
We assume access to an audit subset of size $m=1{,}000$ ($10\%$ of the training pool), sampled uniformly at random. For each audited sample, we observe its true membership label (clean vs.\ contaminant).

We use the audited subset to train a logistic regression classifier to predict membership. The classifier operates on the flattened $784$-dimensional logit-space pixel vectors, the same representation used by the generative model. We use scikit-learn's \texttt{LogisticRegression} with $\ell_2$ regularization (inverse strength $C = 1.0$), the SAGA solver, and a maximum of $1{,}000$ iterations. The classifier is fit on the $m = 1{,}000$ audited samples and applied to all $N$ points in the contaminated dataset. The trained classifier outputs the predicted probability of the membership label which we directly use as proxy scores $\widetilde{I}_i \in [0,1]$ for all samples in the training pool.

To assess CC--MMD's behavior in a realistic regime where the membership proxy is imperfect, we control proxy quality by perturbing audit labels rather than post-hoc rescaling the classifier scores. Before fitting the logistic regression, we sample a fraction $\eta \in [0, 1]$ of the $m$ audit indices uniformly at random without replacement and invert their membership labels ($0 \leftrightarrow 1$). Because indices are sampled uniformly across both classes, the induced per-class flip rate is also $\eta$ in expectation. Setting $\eta = 0$ recovers the clean-audit setting, while larger $\eta$ progressively degrades the classifier and lowers its AUROC against ground-truth membership on the full pool. The resulting predicted inlier probabilities are used directly as proxy scores $\widetilde{I}_i$, with no additional rescaling. 
Note that the audit labels are flipped only for proxy fitting purposes, while the residual corrections $\Delta_j$ always use the true audited memberships. For the main experiment we set $\eta = 0.25$. The resulting proxy achieves AUROC $0.754 \pm 0.019$ under K-MNIST contamination and $0.831 \pm 0.019$ under Omniglot (mean $\pm$ standard deviation) over five seeds (42-46), measured against ground-truth membership on the full training pool. Both are moderate operating points, well above chance (0.5) but well below a near-oracle proxy. The proxy separates Omniglot more easily because its characters are dark on a light background. For reference, when $\eta = 0$, the proxy classifier attains AUROC $0.995 \pm 0.001$ for K-MNIST and $1.000 \pm 0.000$ for Omniglot.

Appendix~\ref{app:additional-results} provides a discussion on the way proxy scores are obtained in Experiments 3 (Appendix~\ref{sec:exp_concrete_full}) and 4 (Appendix~\ref{sec:appendix_generative}) together with additional results to support the discussion.

\paragraph{Generative model architecture.}

We evaluate CC--MMD across three generative model families.

\begin{itemize}
\item \textbf{Inverse Autoregressive Flow (IAF).}
The IAF~\citep{kingma2016improved} is implemented using the \texttt{nflows} library. Each of the 8 flow layers wraps a masked affine autoregressive transform (MADE conditioner with hidden width 512 and 2 residual blocks), producing dimension-wise affine parameters $\mu_d(u_{<d})$ and $\sigma_d(u_{<d})$ that can be evaluated in parallel along the sampling direction. Batch normalization is applied within each MADE block and between consecutive layers. Random permutations of all 784 dimensions are inserted between layers. The base distribution is $u \sim \mathcal{N}(0, I_{784})$.

\item \textbf{Variational Autoencoder (VAE).}
The VAE~\citep{kingma2013auto} has a symmetric MLP encoder--decoder pair with ReLU activations throughout. The encoder maps a 784-dimensional logit-space input through three hidden layers, $784 \to 1024 \to 512 \to 256$, followed by two parallel linear heads that output the mean and log-variance of a diagonal-Gaussian approximate posterior in a 128-dimensional latent space. The decoder mirrors the encoder, $128 \to 256 \to 512 \to 1024 \to 784$, with no final activation on the final layer so that the output is in logit space as the data.

\item \textbf{Convolutional Generator (GMMN).}
The generator follows the DCGAN architecture~\citep{radford2015unsupervised}, trained with MMD-based objectives in the spirit of Generative Moment Matching Networks (GMMN)~\citep{li2015generative}. It maps a noise vector $u \sim \mathcal{N}(0, I_{128})$ through a fully connected layer ($128 \to 256\!\times\!7\!\times\!7$, BN, ReLU), two transposed convolutions (channels $256 \to 128 \to 64$, kernel 4, stride 2, padding 1, each followed by BN and ReLU), and a final $3\!\times\!3$ convolution to a single channel, flattened to $\mathbb{R}^{784}$ with no final activation. Weights are initialized as $\mathcal{N}(0, 0.02)$ for linear and convolutional layers and $\mathcal{N}(1, 0.02)$ for batch normalization scale parameters.

\end{itemize}

\paragraph{Training setup.}
All distribution-matching methods (SWD, Naive MMD, CC--MMD) across all three model families are trained for a maximum of $10{,}000$ gradient steps using the Adam optimizer with default momentum parameters ($\beta_1 = 0.9$, $\beta_2 = 0.999$, $\varepsilon = 10^{-8}$, no weight decay). The learning rate follows a cosine annealing schedule decaying from a peak of $3\times10^{-4}$ to a minimum of $3\times10^{-6}$ over the full $10{,}000$ steps. At each gradient step, $n_{\mathrm{gen}} = 1{,}000$ samples are drawn from the generator. Gradients are clipped to a maximum $\ell_2$ norm of $1.0$ before each parameter update. Early stopping monitors the smoothed training loss, averaged within non-overlapping windows of 100 steps. Training halts and restores the best checkpoint if no improvement is observed for 15 consecutive windows ($1{,}500$ steps of patience). All three models use identical optimization settings for their distribution-matching objectives.

The VAE introduces two additional training paradigms for the ELBO baselines (VAE\nobreakdash+MSE and VAE\nobreakdash+Huber). These are trained epoch-wise minimizing the ELBO: 
$\mathcal{L} = \mathcal{L}_{\mathrm{recon}} + \beta_{\mathrm{KL}}\, \mathrm{KL}\!\bigl(\phi_\theta(\cdot\,|\,x)\,\|\,\mathcal{N}(0,I)\bigr)$,
where $\phi_\theta(\cdot|x)$ denotes the variational posterior, and $\mathcal{L}_{\mathrm{recon}}$, MSE or Huber, is summed over all 784 dimensions and averaged over the mini-batch, and $\beta_{\mathrm{KL}} = 1.0$. Optimization uses Adam with peak learning rate $10^{-3}$, cosine annealing to $10^{-5}$ over 200 epochs, mini-batch size 256, and gradient clipping norm 5.0. A random 15\% validation split is held out for early stopping with patience of 15 epochs (best model restored). The Huber loss uses a transition point of $\delta = 1.0$ in logit space.

For the VAE distribution-matching variants (VAE\nobreakdash+MMD and VAE\nobreakdash+CC--MMD), a KL regularizer is computed from a fresh mini-batch of 256 training points at each step and added with weight $\beta_{\mathrm{KL}} = 0.01$. This small value is necessary since MMD$^2$ on 784-dimensional logit-space data is of order $10^{-2}$--$10^{-1}$, while the KL term is of order $10$--$50$, so $\beta_{\mathrm{KL}} = 1$ would collapse training to prior matching. All other settings match the flow and GMMN distribution-matching experiments.

The GAN~\citep{goodfellow2020generative} baseline is trained adversarially for $30{,}000$ steps with a batch size of 500 real and 500 fake samples. The generator uses Adam with $(\beta_1, \beta_2) = (0.5, 0.999)$ and learning rate $2\times10^{-4}$; the discriminator uses Adam with the same $\beta$ values and learning rate $5\times10^{-5}$. One discriminator update is performed per generator update. The discriminator is a three-layer MLP ($784 \to 512 \to 256 \to 1$) with LeakyReLU(0.2) activations and no final activation (logit output for use with \texttt{BCEWithLogitsLoss}). The non-saturating generator objective is used: the generator minimizes $-\log D(G(z))$ rather than $\log(1 - D(G(z)))$. Gradient clipping of norm 1.0 is applied to both generator and discriminator.

\paragraph{Loss functions.}
All kernel-based methods use a Gaussian kernel $k(x, y) = \exp\!\bigl(-\|x - y\|^2 / (2\sigma^2)\bigr)$ in the 784-dimensional logit-space. The bandwidth $\sigma$ is the median pariwise distance among the known target points in the audit set: $\sigma = \mathrm{median}\{\|w_i - w_j\| : i < j\}$, $\{W_j : j\in L,\,I_j=1\}$, subsampled to at most $2{,}000$ points for efficiency. Using only confirmed clean points prevents the contaminating distribution from biasing the bandwidth estimate. For both Naive MMD and CC--MMD, all pool-side quantities that depend only on the fixed training data are computed once before training begins and cached. For Naive MMD this is the pool self-kernel sum $\sum_{j\neq j'} k(W_j, W_{j'})$, while for CC--MMD these are the corrected sums $\widehat{A}_{qq}$, $\widehat{B}_{qq}$ and the $W$-side factors of $\widehat{A}_{pq}$ and $\widehat{B}_{pq}$. Only the terms involving newly generated samples $X_{\mathrm{gen}}$ ($K(X_{\mathrm{gen}}, X_{\mathrm{gen}})$ and $K(X_{\mathrm{gen}}, W)$)---are recomputed at each gradient step.

The loss objectives compared across all three model families are:

\begin{itemize}
  \item \textbf{Na\"{i}ve MMD.} Defined in Equation~\ref{eq:mmd_target}.

  \item \textbf{CC--MMD.} Defined in Equation~\ref{eq:mmd_cc}.

  \item \textbf{Sliced Wasserstein Distance (SWD).}~\citep{rabin2011wasserstein} Approximates the 2-Wasserstein distance using $P = 200$ random one-dimensional projections. At each training step, $W$ is subsampled to $n_{\mathrm{gen}}$ points, and the loss is the mean squared difference between sorted projections:
    \[
      \widehat{\mathrm{SWD}}^2
      = \frac{1}{P}\sum_{s=1}^{P}\frac{1}{n_{\mathrm{gen}}}
        \sum_{\ell=1}^{n_{\mathrm{gen}}}
        \bigl((X_{\mathrm{gen}}\cdot v_s)_{(\ell)}
              -(W_{\mathrm{sub}}\cdot v_s)_{(\ell)}\bigr)^2,
    \]
    where $v_s$ are random unit vectors and $(\cdot)_{(\ell)}$ is the
    $\ell$-th order statistic. SWD is applied only to the IAF.

  \item \textbf{GAN (adversarial, GMMN only).}~\citep{goodfellow2020generative} Non-saturating adversarial objective; see training details above. This baseline is included only for the GMMN experiment to assess whether an implicit generative model trained without any explicit kernel can resist contamination.

  \item \textbf{VAE+MSE and VAE+Huber (ELBO, VAE only).} These baselines train the VAE by maximizing the standard ELBO with MSE or Huber reconstruction losses respectively, without any explicit distribution-matching objective.
\end{itemize}

\paragraph{Evaluation metrics.}
We evaluate generative performance using two metrics, all computed against the $10{,}000$ clean MNIST test images.

\begin{itemize}
    \item \textbf{Test MMD$^2$.}
    We use the same Gaussian kernel as during training, with bandwidth $\sigma$ estimated via the median heuristic on the test set itself (rather than the training pool). We draw $5{,}000$ samples from the generator. The unbiased U-statistic estimator is evaluated in logit space.
    \item \textbf{Fréchet Inception Distance (FID).}
    Because standard FID uses Inception-v3 features designed for $299\times299$ RGB images, we replace the feature extractor with a small CNN trained on clean MNIST. The network applies two strided $3\times3$ convolutions (channels $1\to32\to64$, stride 2) followed by a linear projection to a 128-dimensional feature space, with a 10-class classification head used only during training. It is trained on the full $60{,}000$-image MNIST training set for 10 epochs (Adam, lr\,$=10^{-3}$, cross-entropy, batch size 256) and fixed before any generative model is trained. Test-set features are precomputed once and reused across all model comparisons. We use $50{,}000$ generated samples. Generated outputs are converted to $[0,1]$ via sigmoid before feature extraction. All reported FID values use this MNIST-specific extractor.
\end{itemize}

\paragraph{Main results ($\epsilon = 0.60$).}
Table~\ref{tab:step4_results} and Figure~\ref{fig:ccmmd_contam_metrics} show the numerical results of the generative modeling under contamination experiment. We provide further detailed commentary here.

\begin{table}[t]
\centering
\small
\begin{tabular}{llcccc}
\toprule
& & \multicolumn{2}{c}{K-MNIST} & \multicolumn{2}{c}{Omniglot} \\
\cmidrule(lr){3-4}\cmidrule(lr){5-6}
Architecture & Method & MMD$^2$ ($\times10^{-3}$) & FID & MMD$^2$ ($\times10^{-3}$) & FID \\
\midrule
\multirow{3}{*}{IAF}
 & SWD       & $26.5\pm0.6$ & $207.7\pm2.5$  & $447.7\pm12.5$ & $593.2\pm32.4$ \\
 & MMD       & $25.9\pm0.4$ & $203.7\pm6.4$  & $459.5\pm3.2$  & $564.2\pm7.5$  \\
 & \textbf{CC--MMD}   & $\mathbf{1.1\pm0.2}$  & $\mathbf{118.8\pm8.8}$  & $\mathbf{1.9\pm1.0}$    & $\mathbf{125.0\pm10.6}$ \\
\midrule
\multirow{4}{*}{VAE}
 & MSE       & $131.2\pm4.9$ & $245.6\pm1.4$  & $550.2\pm17.0$ & $516.5\pm11.5$ \\
 & Huber     & $42.6\pm3.5$  & $\mathbf{190.4\pm4.7}$  & $422.6\pm33.7$ & $457.0\pm5.3$  \\
 & MMD       & $27.9\pm1.1$  & $443.9\pm6.2$  & $452.5\pm4.9$  & $593.2\pm27.2$ \\
 & \textbf{CC--MMD}   & $\mathbf{2.4\pm0.3}$   & $377.4\pm15.7$ & $\mathbf{2.7\pm0.9}$    & $\mathbf{359.6\pm10.9}$ \\
\midrule
\multirow{3}{*}{GMMN}
 & GAN       & $27.2\pm0.7$  & $\mathbf{114.7\pm4.9}$  & $1024.9\pm509.5$ & $745.7\pm350.2$ \\
 & MMD       & $27.1\pm0.8$  & $160.6\pm1.7$  & $458.2\pm7.1$  & $444.1\pm4.1$  \\
 &\textbf{ CC--MMD}   & $\mathbf{1.2\pm0.2}$   & $139.6\pm4.6$  & $\mathbf{1.9\pm1.2}$    & $\mathbf{138.7\pm3.9}$  \\
\bottomrule
\end{tabular}
\caption{Test MMD$^2$ ($\times 10^{-3}$) and FID on clean MNIST test images under $\epsilon = 0.60$ contamination, $\mu \pm \sigma$ \ over 5 seeds. Lower is better for both metrics. The same results are also shown on Figure~\ref{fig:ccmmd_contam_metrics}.}
\label{tab:step4_results}
\end{table}

Under severe contamination, the architecture-specific baselines all fail to recover the clean MNIST distribution since each optimizes a distance to the contaminated training pool, causing the generator to reproduce the mixture rather than the target. This is reflected in high test MMD$^2$ values across all three architectures (IAF: $25.9$--$26.5 \times 10^{-3}$; GMMN: $27.1$--$27.2 \times 10^{-3}$; VAE: $27.9$--$131.2 \times 10^{-3}$) and correspondingly high FID scores.

Instead, CC--MMD uses audit-informed proxy scores to reweight the training objective toward the clean distribution, consistently recovering clean MNIST samples across all three architectures and both contamination sources. Test MMD$^2$ is reduced by approximately $23\times$ on K-MNIST and over $200\times$ on Omniglot relative to the best uncorrected baseline, with the larger gain on Omniglot reflecting the greater visual dissimilarity between that contaminant and MNIST. FID improvements are substantial in all settings. Under Omniglot contamination, where baselines degrade most severely, CC--MMD achieves FID $125.0 \pm 10.6$ versus $564.2 \pm 7.5$ for IAF and $138.7 \pm 3.9$ versus $444.1 \pm 4.1$ for GMMN. Under K-MNIST contamination the improvement is smaller in absolute terms but remains consistent across architectures, since the visual similarity between K-MNIST and MNIST means uncorrected baselines still produce partly digit-like outputs. Across both settings, CC--MMD results are stable across seeds, with low variance in both MMD$^2$ and FID.

\paragraph{Ablation study on $\epsilon$ (Robustness to contamination rate)}
In this ablation study, we demonstrate CC--MMD's robustness to increasing contamination rate, $\epsilon$. We do this by examining IAF's generative quality with different loss functions as the contamination level $\epsilon$ increases, $\epsilon \in \{0.0, 0.1, \dots, 0.8\}$, while holding fixed $N=10{,}000$, audit fraction $m/N=0.10$, and proxy quality at AUROC $\approx 0.70$ ($\eta = 0.30$). Here, we focus on contamination coming from K-MNIST.
We run each configuration over 20 random seeds, $42 + 10000 \cdot i + r$, where $i = 0, \dots, 8$ indexes $\epsilon$ and $r = 0, \dots, 19$ indexes the replicate, and report $\mu \pm \sigma$ of test MMD$^2$ and test FID evaluated against clean MNIST test set ($10{,}000$ samples). Results are shown in Figure~\ref{fig:ablation_epsilon_mmd2}.

Na\"ive MMD and SWD are blind to contamination and their test MMD$^2$ grows roughly exponentially with $\epsilon$, reaching $0.046$ at $\epsilon = 0.8$. This is nearly two orders of magnitude above the oracle baseline of $0.00055$. SWD tracks Na\"ive MMD closely throughout, confirming that the insensitivity to contamination is not specific to the MMD loss.

CC--MMD tracks the oracle trend consistently across the full sweep, achieving test MMD$^2$ of $0.001209$ at $\epsilon = 0.6$ and $0.003713$ at $\epsilon = 0.8$, corresponding to $21\times$ and $12\times$ reductions relative to Na\"ive MMD. The gap between CC--MMD and the oracle grows modestly at high contamination rates, which is expected as at large $\epsilon$, the effective number of confirmed clean points within the audit set shrinks, increasing the variance of the Horvitz--Thompson correction.

Training only on the audited clean subset avoids contamination entirely and performs competitively with CC--MMD at high $\epsilon$, reaching $0.0023$ at $\epsilon = 0.8$ where its zero proxy error is most advantageous. At lower contamination levels ($\epsilon \leq 0.5$), however, CC--MMD consistently outperforms audit-only on MMD$^2$ by leveraging the full pool through proxy correction, demonstrating advantages of such approach.

\paragraph{Ablation study on proxy scores quality (Robustness to proxy score estimator performance)}
The theoretical guarantees of CC-MMD require no assumptions on the quality of the proxy scores $\widetilde{I}_j$. The correction remains design-unbiased regardless of how well $\widetilde{I}_j$ approximates the true membership indicator $I_j$ (Theorem \ref{thm:unbiased_sums}). 
We show this empirically.

We hold the following fixed: IAF generator, $\epsilon = 0.60$, $N = 10{,}000$, $m/N=10\%$ ($m = 1{,}000$ labeled points). Again, we focus on contamination coming from K-MNIST dataset.
We vary proxy scores AUROC (in the way described in Paragraph~\ref{sec:audit_and_proxy_estimation}) across $\approx {0.5, 0.6, 0.7, 0.8, 0.9, 0.99}$, and report results as a function of the achieved proxy scores AUROC.
Each condition is repeated over 20 independent seeds ($42-61$) and evaluated againstthe  clean MNIST test set using test $\mathrm{MMD}^2$ and FID. All results report $\mu \pm \sigma$.

Figure~\ref{fig:ablation_proxy} summarizes the findings.

As proxy AUROC increases from near-random (${\approx}0.50$) to near-perfect (${\approx}0.99$), test MMD$^2$ improves weakly, from $(13.6 \pm 3.1) \times 10^{-4}$ to $(5.8 \pm 1.2) \times 10^{-4}$, with most of the change concentrated in the upper range of AUROC values and a plateau across AUROC $0.50$--$0.90$.
The widening standard-deviation bands at lower AUROC confirm that proxy quality primarily affects finite-sample efficiency and optimization rather than the validity of the correction itself. Recall that Theorem~\ref{thm:design_clt_full} guarantees design-unbiasedness for any proxy, so a weaker proxy increases the variance of the Horvitz--Thompson weights without introducing bias in the estimator.
Taken together, these results indicate that \textsc{CC-MMD} does not require a near-perfect proxy to improve over na\"ive contaminated distribution matching and a modest audit effort yielding AUROC in the $0.60$--$0.80$ range is sufficient to recover the bulk of the benefit.

\begin{figure}[t]
    \centering
    \includegraphics[width=0.49\linewidth]{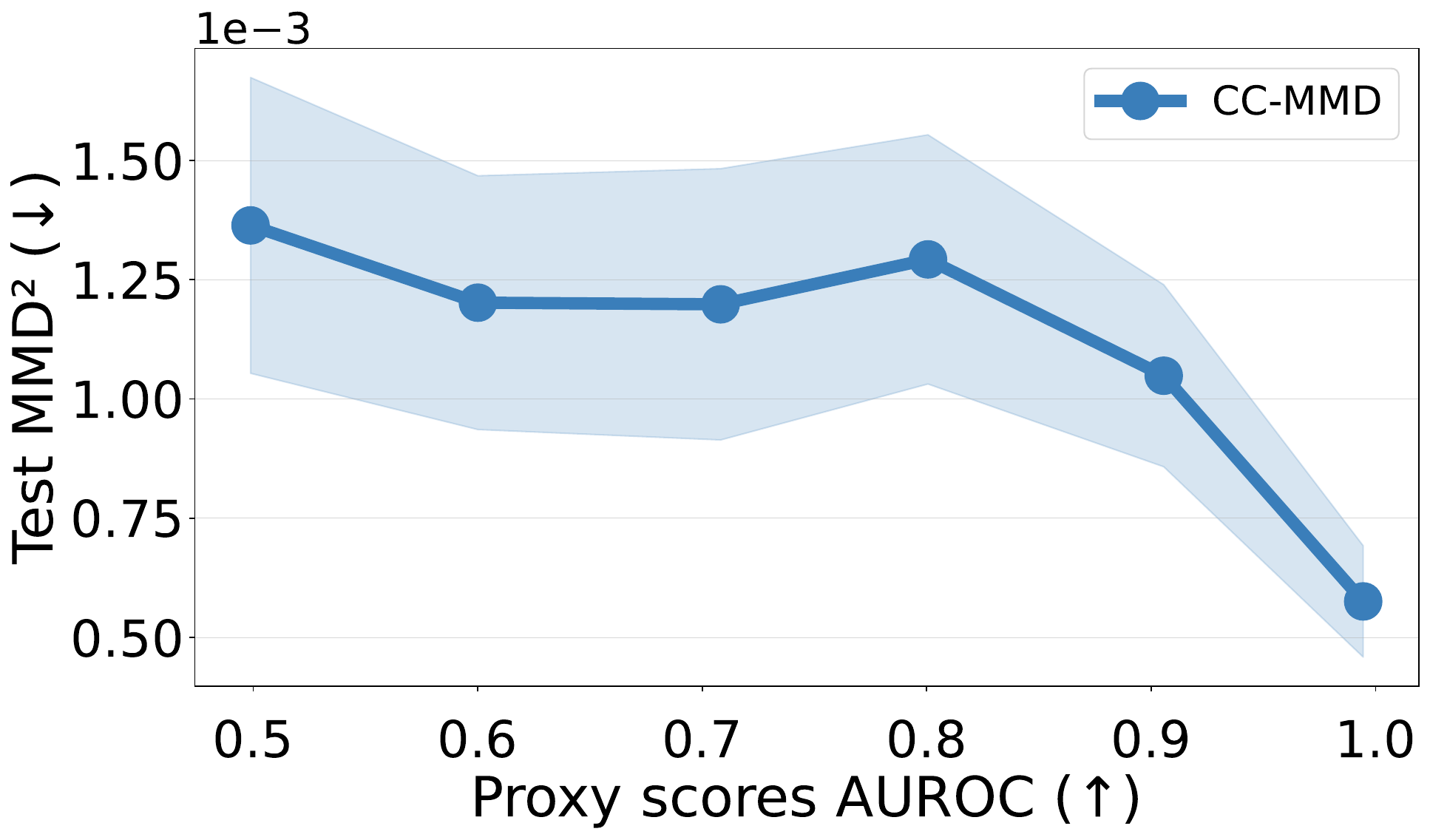}
    \hfill
    \includegraphics[width=0.49\linewidth]{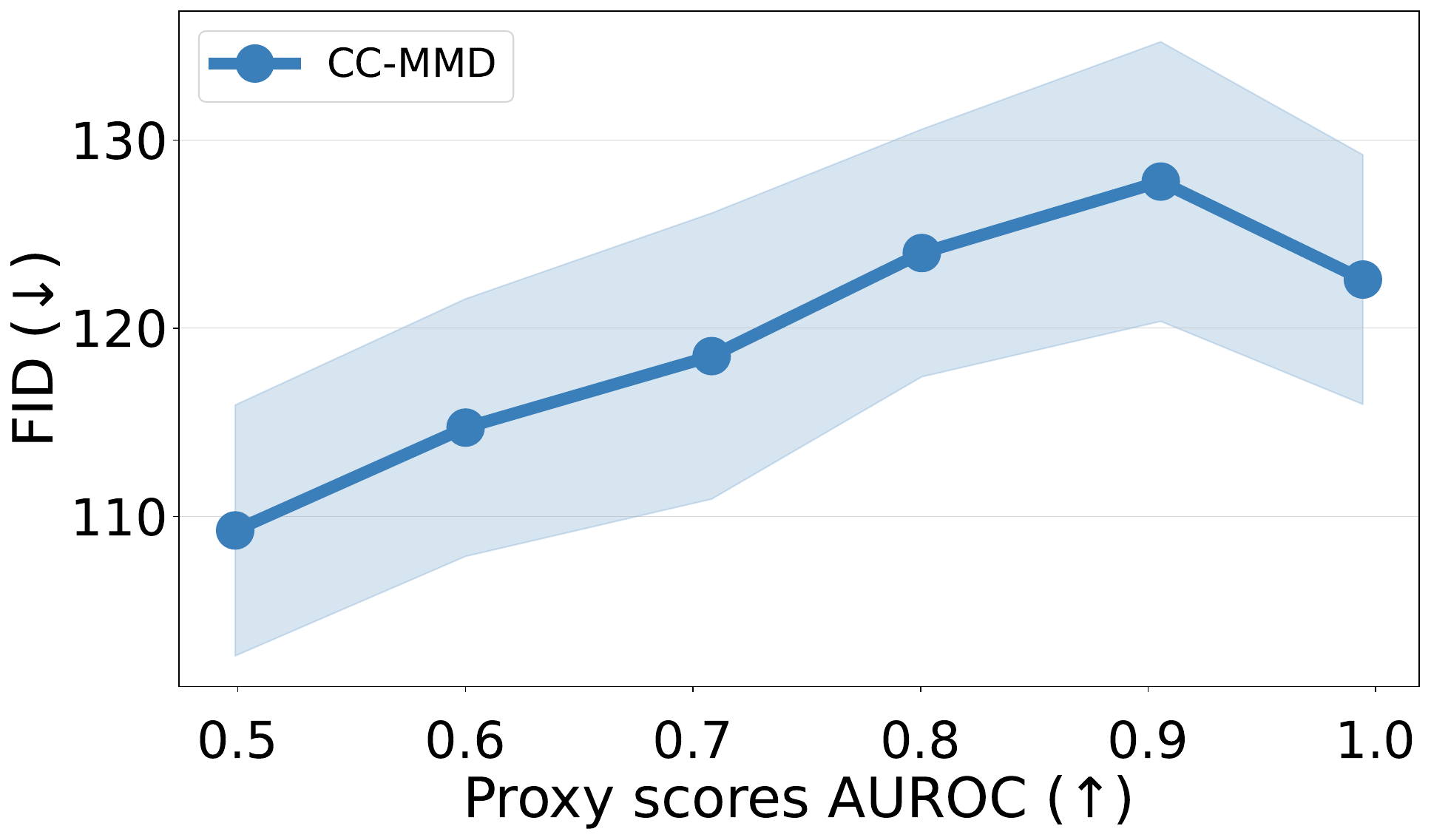}
    \caption{\textbf{Generative quality vs. proxy scores AUROC on MNIST.} Test $\mathrm{MMD}^2$ (\textbf{left}) and FID (\textbf{right}) against the clean MNIST test set. Contamination rate is fixed at $\epsilon = 0.60$ and so is audit fraction at $m/N = 0.10$. Contamination source is K-MNIST and architecture is IAF. Shaded bands show $\pm 1$ standard deviation over 20 random seeds.}
    \label{fig:ablation_proxy}
\end{figure}

\paragraph{Ablation study on $m$ (audit sample size)}
Here, we investigate the impact of the audit sample size on CC-MMD performance. We begin by fixing: IAF generator, $\epsilon = 0.6$, $N = 10{,}000$, AUROC$\approx 0.7$ ($\eta$ varies to achieve desired AUROC across different $m$), K-MNIST as a contamination source. We vary audit size $m/N \in \{5\%,10\%,20\%,30\%,40\%,50\%\}$. Each condition is repeated over 20 independent seeds ($42-61$) and evaluated against the clean MNIST test set using test $\mathrm{MMD}^2$ and FID. All results report $\mu \pm \sigma$ and are shown in Figure~\ref{fig:ablation-audit-size}.

As $m$ increases, performance improves rapidly: with only $m = 5\%$ of the pool labeled (500 points), CC--MMD already corrects the bulk of the contamination-induced bias and substantially outperforms na\"ive MMD trained on the full contaminated pool (see Figure~\ref{fig:ccmmd_contam_metrics}).
Increasing the budget to $m = 30\%$ yields a further $3.18\times$ reduction in test MMD$^2$ ($1.898 \times 10^{-3} \to 5.973 \times 10^{-4}$), driven by a more accurate bandwidth estimate from the growing clean audit set.
Beyond $m = 30\%$, performance is largely stable, indicating that a moderate annotation budget captures most of the available signal and that the method does not require exhaustive labeling to operate effectively.
These results confirm that even limited but informative audit supervision is sufficient to correct much of the contamination-induced bias, a practically important property when ground-truth membership labels are expensive to obtain

\begin{figure}[t]
  \centering
  \includegraphics[width=0.49\linewidth]{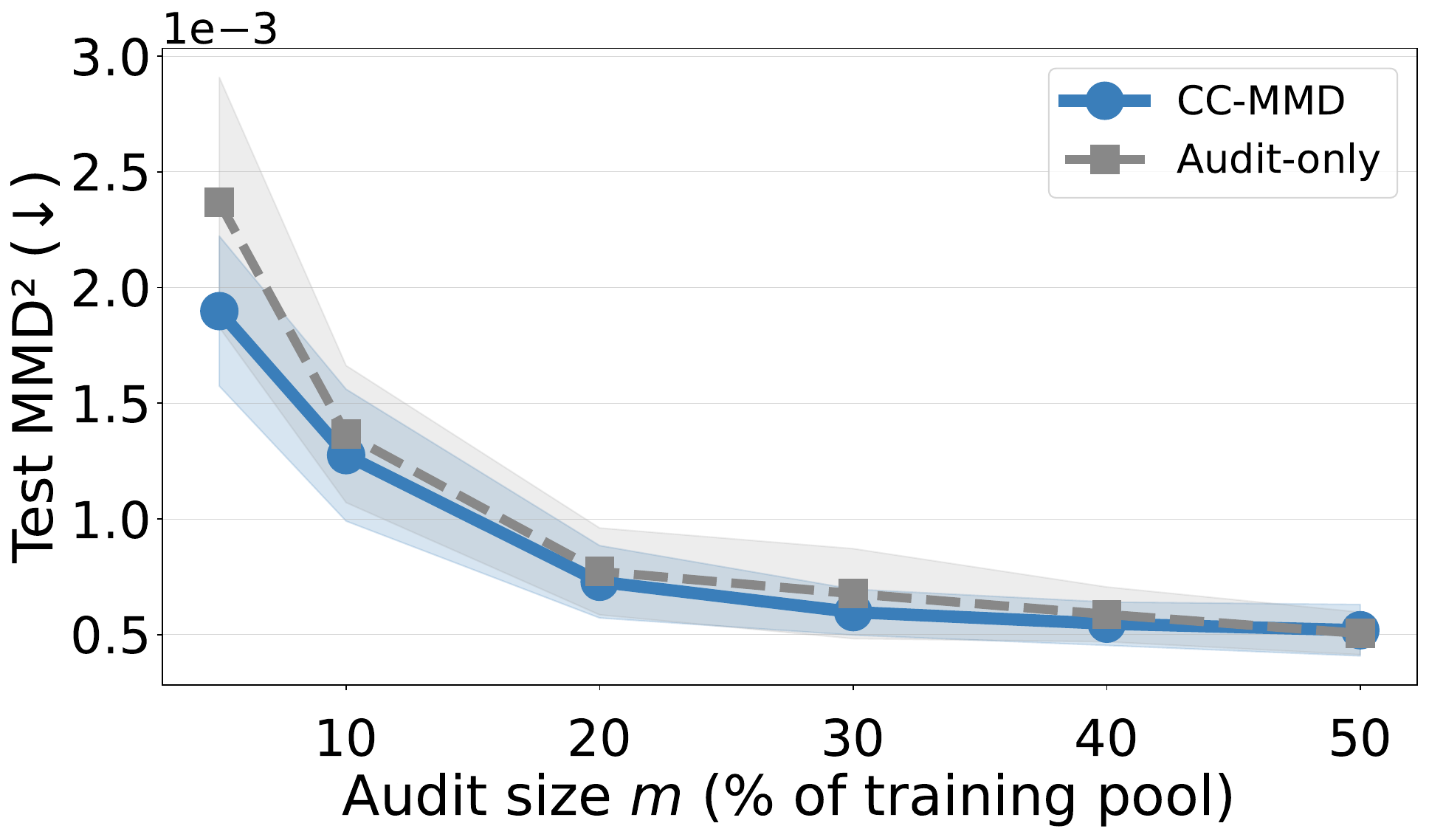}
  \hfill
  \includegraphics[width=0.49\linewidth]{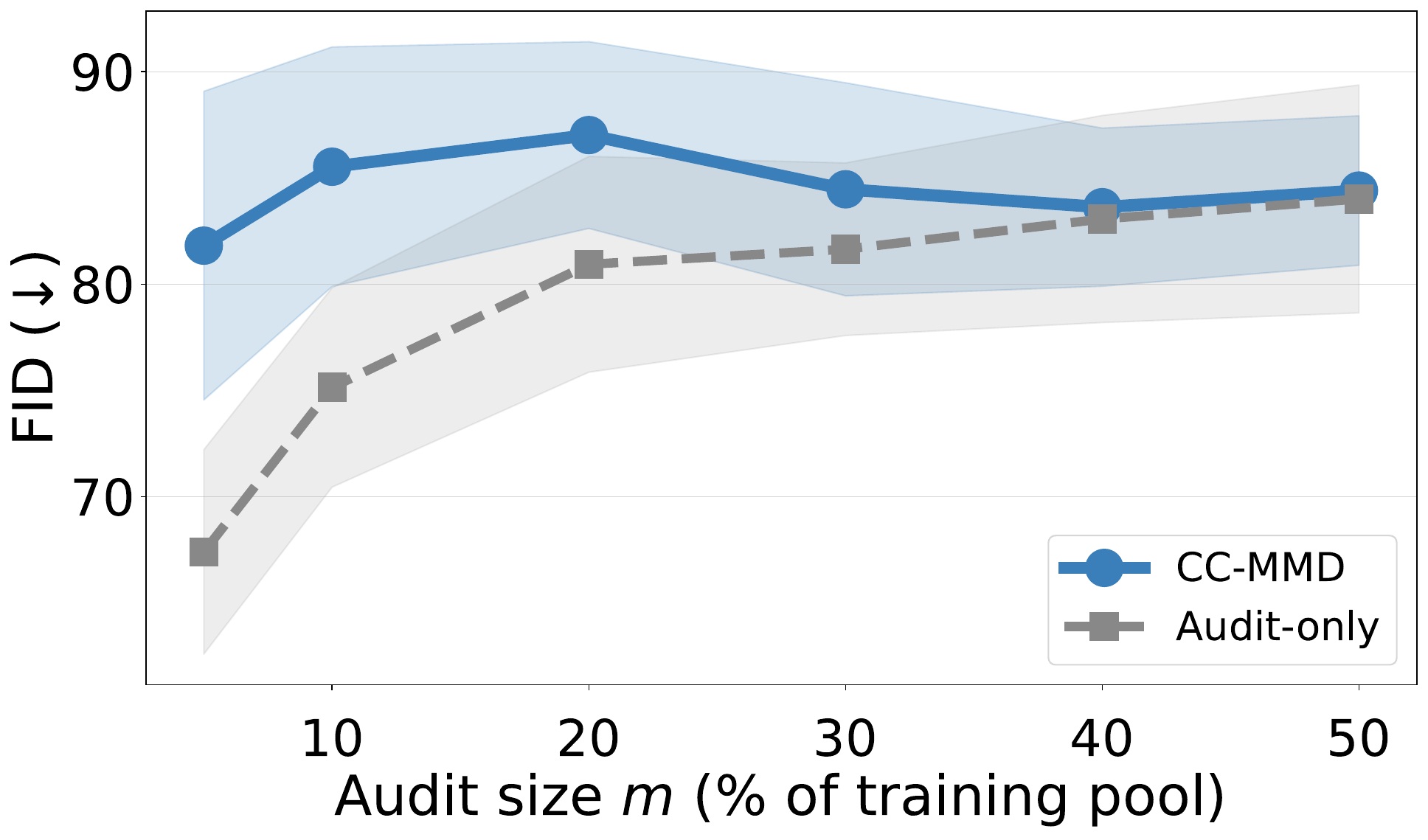}
  \caption{
    \textbf{Generative quality vs. audit size on MNIST.} Test $\mathrm{MMD}^2$ (\textbf{left}) and FID (\textbf{right}) against the clean MNIST test set. Contamination rate is fixed at $\epsilon = 0.60$ and proxy AUROC is $\approx 0.7$ across all the audit sizes. Contamination source is K-MNIST and architecture is IAF. Shaded bands show $\pm 1$ standard deviation over 20 random seeds. For reference, audit-only performance is shown too.}
  \label{fig:ablation-audit-size}
\end{figure}

\section{Additional results and discussion on proxy scores}
\label{app:additional-results}
As specified in Section~\ref{sec:contamination_problem}, the proxy scores, audited subset, and full dataset are treated as fixed inputs produced by a pre-existing data-processing pipeline. The downstream modeler does not train the proxy model to obtain proxy scores. Rather, they are generated by an upstream system, for example, labels or judgments from a pretrained LLM, scores from an automated signal-versus-noise classifier, or domain-specific classification probabilities used in astronomy to distinguish galaxies, stars, and quasars. The proxy scores exist before any audit is drawn. 
The modeler observes proxy scores for the full sample and true membership labels only for the audited subset; illustrative examples are provided in Appendix~\ref{sec:illustrative_exmpl}.

This reflects a distinction between two streams of research: designing the upstream data-processing pipeline to maximize information gain for downstream inference, and performing statistically valid downstream modeling and inference given whatever that pipeline produces. 
While the former is an interesting and important research direction, it is a fundamentally different problem from the one studied here. In astronomy, for instance, such upstream pipelines are developed over many years by large collaborations to support a broad range of scientific objectives rather than a single downstream task, and may be trained on raw data or infrastructure the downstream analyst cannot access. Therefore, the available proxy scores and audited labels are constraints imposed by the data collection and processing pipeline, not quantities the downstream analyst can readily optimize. Our work addresses the complementary problem of performing statistically valid learning given these fixed resources.

In this setting the independence Assumption~\ref{ass:srs} requires holds by construction: the analyst does not fit the proxy, and the audit is a simple random sample drawn afterwards.

Experiments 1 and 2 in Section~\ref{sec:experiments} instantiate this directly. The proxy scores there are produced by fixed score functions of $W$ alone and never touch the audit set, so Assumption~\ref{ass:srs} holds exactly as stated.

Since no such upstream pipeline exists for the benchmark datasets we use in Experiments 3 and 4, we must simulate one. One way to do it is to train a classifier using the audited labels. This draws attention to Assumption~\ref{ass:srs} as it creates dependence between $\tilde I$ and $L$. However, fitting a proxy using the audit sample is an artifact of constructing a proxy for a benchmark rather than a feature of the estimator, which never refits $\tilde I$.
An alternative is to use the audit only to fit a classifier and then either weight or filter the pool by its predictions. 

To confirm that our modeling choice does not drive our conclusions, this appendix compares the two strategies, in the regression setting of Experiment 2 and in the generative setting of Experiment 4.

\subsection{No Upstream Classifier}
\label{sec:additional-no-classifier}

First, we consider the most challenging setting, in which no upstream classifier is available. We assume $\widetilde{I}_j = 1$ for $j = 1, \dots, N$. This represents the scenario in which the downstream method has no prior information regarding which observations may be contaminated and defaults to the assumption that all observations come from the target distribution. Here, no classifier is fit at any stage.

We compare four approaches for estimating the slope of the linear model of Experiment 2 in Section~\ref{sec:experiments}, using the same contamination model:

\begin{itemize}
\item \textbf{Na\"ive MMD.} The regression parameters are estimated by minimizing the MMD objective using all observations, without any correction for contamination.
\item \textbf{CC--MMD.} The audit sample with verified labels is used directly to construct the CC-MMD debiasing correction, while the remaining observations are used in the standard MMD objective.
\item \textbf{Probability-weighted MMD.} The audit sample is used to train the best random forest classifier estimating the probability of each non-audited observation being an inlier. These estimated probabilities are then used as weights in the MMD objective.
\item \textbf{Hard-threshold MMD.} The same random forest classifier is trained on the audit sample, but only observations classified as inliers are retained in the MMD objective after removing the audited observations.
\end{itemize}

Table~\ref{tab:additional_table_1} reports the resulting MSEs for estimating the slope using 50 random realizations of data.

\begin{table}[t]
\centering
\caption{Slope MSE $\pm$ SE over 50 random realizations with no upstream classifier ($\tilde I_j \equiv 1$), under the contamination model of Experiment 2. Lower is better.}
\label{tab:additional_table_1}
\small
\begin{tabular}{llcccc}
\toprule
$N$ & $m$ & Na\"ive MMD & CC--MMD ($\tilde I_j \equiv 1$) & RF prob.-weighted & RF hard-threshold \\
\midrule
200  & 20 & 0.0579 $\pm$ 0.0103 & \textbf{0.0367 $\pm$ 0.0073} & 0.0384 $\pm$ 0.0077 & 0.0882 $\pm$ 0.0169 \\
200  & 40 & 0.0495 $\pm$ 0.0093 & \textbf{0.0120 $\pm$ 0.0024} & 0.0162 $\pm$ 0.0032 & 0.0280 $\pm$ 0.0054 \\
200  & 60 & 0.0517 $\pm$ 0.0097 & \textbf{0.0112 $\pm$ 0.0022} & 0.0160 $\pm$ 0.0032 & 0.0229 $\pm$ 0.0045 \\
300  & 20 & 0.0369 $\pm$ 0.0070 & \textbf{0.0194 $\pm$ 0.0039} & 0.0252 $\pm$ 0.0049 & 0.0923 $\pm$ 0.0168 \\
300  & 40 & 0.0457 $\pm$ 0.0084 & \textbf{0.0118 $\pm$ 0.0024} & 0.0177 $\pm$ 0.0035 & 0.0249 $\pm$ 0.0046 \\
300  & 60 & 0.0579 $\pm$ 0.0104 & \textbf{0.0077 $\pm$ 0.0015} & 0.0138 $\pm$ 0.0025 & 0.0201 $\pm$ 0.0036 \\
500  & 20 & 0.0437 $\pm$ 0.0082 & \textbf{0.0167 $\pm$ 0.0033} & 0.0249 $\pm$ 0.0049 & 0.2686 $\pm$ 0.0537 \\
500  & 40 & 0.0418 $\pm$ 0.0079 & \textbf{0.0105 $\pm$ 0.0021} & 0.0141 $\pm$ 0.0028 & 0.0294 $\pm$ 0.0057 \\
500  & 60 & 0.0548 $\pm$ 0.0097 & \textbf{0.0097 $\pm$ 0.0019} & 0.0117 $\pm$ 0.0022 & 0.0157 $\pm$ 0.0030 \\
1000 & 20 & 0.0421 $\pm$ 0.0079 & 0.0548 $\pm$ 0.0110 & \textbf{0.0177 $\pm$ 0.0035} & 0.0564 $\pm$ 0.0105 \\
1000 & 40 & 0.0809 $\pm$ 0.0161 & \textbf{0.0123 $\pm$ 0.0025} & 0.0148 $\pm$ 0.0027 & 0.0207 $\pm$ 0.0038 \\
1000 & 60 & 0.0514 $\pm$ 0.0093 & 0.0100 $\pm$ 0.0020 & \textbf{0.0085 $\pm$ 0.0017} & 0.0098 $\pm$ 0.0019 \\
\bottomrule
\end{tabular}
\end{table}

As expected, the naive MMD estimator performs worst since it completely ignores contamination. CC-MMD achieves the lowest slope MSE in 10 of the 12 experimental configurations and is consistently among the best-performing methods. The probability-weighted approach performs competitively, occasionally achieving the lowest error, particularly for large sample sizes coupled with very small audit sample. The hard-threshold approach is generally less stable and performs worse because classification errors permanently discard observations that may still contain useful information.

\subsection{Weak Upstream Classifier}
\label{sec:additional-weak-classifier}

We next consider a more realistic setting using a weak upstream classifier as a proxy. Here, the proxy is the fixed, miscalibrated score function of Experiment 2 in Section~\ref{sec:experiments} (with $\sigma = 1.0$). It is a function of the pooled sample alone and never uses the audit set. The weak upstream classifier provides noisy estimates of the inlier probability for each observation, in place of the uninformative $\tilde I_j \equiv 1$ of Appendix~\ref{sec:additional-no-classifier}. We repeat the same set of experiments as in \ref{sec:additional-no-classifier} using these imperfect scores. The results are shown in Table~\ref{tab:additional_table_2}. 

\begin{table}[t]
\centering
\caption{Slope MSE $\pm$ SE over 50 random realizations with a weak upstream
classifier providing noisy inlier probability estimates. Settings otherwise as in
Table~\ref{tab:additional_table_1}. Lower is better.}
\label{tab:additional_table_2}
\small
\begin{tabular}{llcccc}
\toprule
$N$ & $m$ & Na\"ive MMD & CC--MMD (weak proxy) & RF prob.-weighted & RF hard-threshold \\
\midrule
200  & 20 & 0.0624 $\pm$ 0.0110 & 0.0370 $\pm$ 0.0074 & \textbf{0.0343 $\pm$ 0.0068} & 0.0873 $\pm$ 0.0168 \\
200  & 40 & 0.0491 $\pm$ 0.0093 & \textbf{0.0153 $\pm$ 0.0031} & 0.0164 $\pm$ 0.0032 & 0.0297 $\pm$ 0.0058 \\
200  & 60 & 0.0552 $\pm$ 0.0104 & \textbf{0.0105 $\pm$ 0.0021} & 0.0154 $\pm$ 0.0030 & 0.0231 $\pm$ 0.0045 \\
300  & 20 & 0.0355 $\pm$ 0.0068 & \textbf{0.0167 $\pm$ 0.0033} & 0.0257 $\pm$ 0.0050 & 0.0924 $\pm$ 0.0167 \\
300  & 40 & 0.0438 $\pm$ 0.0081 & \textbf{0.0109 $\pm$ 0.0022} & 0.0185 $\pm$ 0.0036 & 0.0267 $\pm$ 0.0049 \\
300  & 60 & 0.0541 $\pm$ 0.0095 & \textbf{0.0097 $\pm$ 0.0019} & 0.0135 $\pm$ 0.0025 & 0.0204 $\pm$ 0.0037 \\
500  & 20 & 0.0445 $\pm$ 0.0084 & \textbf{0.0225 $\pm$ 0.0045} & 0.0273 $\pm$ 0.0054 & 0.5206 $\pm$ 0.1041 \\
500  & 40 & 0.0408 $\pm$ 0.0076 & \textbf{0.0100 $\pm$ 0.0020} & 0.0142 $\pm$ 0.0028 & 0.0305 $\pm$ 0.0059 \\
500  & 60 & 0.0571 $\pm$ 0.0104 & \textbf{0.0111 $\pm$ 0.0022} & 0.0119 $\pm$ 0.0022 & 0.0155 $\pm$ 0.0028 \\
1000 & 20 & 0.0457 $\pm$ 0.0086 & 0.0253 $\pm$ 0.0051 & \textbf{0.0172 $\pm$ 0.0034} & 0.0508 $\pm$ 0.0090 \\
1000 & 40 & 0.0919 $\pm$ 0.0183 & 0.0141 $\pm$ 0.0028 & \textbf{0.0125 $\pm$ 0.0023} & 0.0181 $\pm$ 0.0033 \\
1000 & 60 & 0.0577 $\pm$ 0.0102 & \textbf{0.0077 $\pm$ 0.0015} & 0.0090 $\pm$ 0.0018 & 0.0113 $\pm$ 0.0022 \\
\bottomrule
\end{tabular}
\end{table}

In this setting, even modest predictive ability improves the performance of CC--MMD in low-audit regime since the proxy contributes information that the audited residuals alone cannot supply. The other methods do not gain from the upstream classifier because they do not use the proxy scores. The ordering across methods is otherwise unchanged from Table~\ref{tab:additional_table_1}.

\subsection{High-Dimensional Generative Modeling (MNIST)}
\label{sec:additional-mnist}

To assess whether these conclusions extend beyond the linear regression simulation, we repeat the comparison on the MNIST generative experiment in Section~\ref{sec:experiments} with the same settings. We use MNIST digits as target distribution contaminated at rate $\epsilon = 0.6$ with either K-MNIST or Omniglot images, a 10\% audited subset, and the three generative architectures, each trained under 5 random seeds.

We compare five approaches:
\begin{enumerate}
    \item \textbf{Naive MMD.}
    \item \textbf{CC-MMD ($\tilde I \equiv 1$).} No upstream classifier available; we assume
    \[
    \tilde I_j \equiv 1, \quad j = 1, \dots , N,
    \]
    so that every observation is included in the analysis (as in Appendix~\ref{sec:additional-no-classifier}).
    \item \textbf{CC-MMD (trained proxy).} Proxy scores from a proxy classifier (logistic regression) and audit sample are used to construct the CC-MMD debiasing correction, as in Section~\ref{sec:experiments}, Experiment 4.
    \item \textbf{Logistic regression (LR) probability-weighted MMD.}
    \item \textbf{Logistic regression (LR) hard-threshold MMD.}
\end{enumerate}

The results are presented in Table~\ref{tab:additional_table_3} below.

\begin{table}[t]
\centering
\caption{Test $\mathrm{MMD}^2$ ($\times 10^{-3}$) $\pm$ SE over 5 seeds on clean
MNIST test images, at $\epsilon = 0.60$. Lower is better.}
\label{tab:additional_table_3}
\resizebox{\textwidth}{!}{
\begin{tabular}{lccccc}
\toprule
& Na\"ive MMD & CC--MMD ($\widetilde{I} \equiv 1$) & CC--MMD (proxy) & LR prob.-weighted & LR hard-threshold \\
\midrule
\multicolumn{6}{l}{\emph{K-MNIST contamination}} \\
\midrule
IAF  & 25.9 $\pm$ 0.2 & 2.2 $\pm$ 0.2 & \textbf{1.1 $\pm$ 0.1} & 11.1 $\pm$ 0.5 & 10.5 $\pm$ 0.5 \\
VAE  & 27.9 $\pm$ 0.5 & 3.1 $\pm$ 0.2 & \textbf{2.4 $\pm$ 0.1} & 12.4 $\pm$ 0.3 & 11.7 $\pm$ 0.6 \\
GMMN & 27.1 $\pm$ 0.4 & 2.3 $\pm$ 0.3 & \textbf{1.2 $\pm$ 0.1} & 11.3 $\pm$ 0.5 & 10.5 $\pm$ 0.6 \\
\midrule
\multicolumn{6}{l}{\emph{Omniglot contamination}} \\
\midrule
IAF  & 459.5 $\pm$ 1.6 & \textbf{1.9 $\pm$ 0.8} & \textbf{1.9 $\pm$ 0.5} & 167.6 $\pm$ 6.6 & 132.3 $\pm$ 4.4 \\
VAE  & 452.5 $\pm$ 2.5 & 3.0 $\pm$ 0.9 & \textbf{2.7 $\pm$ 0.4} & 165.7 $\pm$ 6.2 & 127.8 $\pm$ 5.6 \\
GMMN & 458.2 $\pm$ 3.5 & \textbf{1.9 $\pm$ 0.8} & \textbf{1.9 $\pm$ 0.6} & 162.7 $\pm$ 3.7 & 132.3 $\pm$ 6.0 \\
\bottomrule
\end{tabular}
}
\end{table}

Both CC-MMD variants dominate across every architecture and contamination source. Crucially, CC-MMD is nearly unaffected by contamination severity: Omniglot (more severe contamination than K-MNIST) inflates the naive-MMD baseline 18-fold over K-MNIST, yet CC-MMD’s error stays flat ($\sim 0.001 - 0.003$). The classifier-trusting methods improve over naive MMD but remain 10–100× worse than CC-MMD, and this gap grows larger exactly when contamination is worst — the opposite of what a robustness method should do. This matches the additional experiments 1 (\ref{sec:additional-no-classifier}) and 2 (\ref{sec:additional-weak-classifier}): using audited data in the estimator in addition to proxy scores remains the key advantage even in high-dimensional, non-linear settings.

\subsection{Discussion}
\label{sec:additional-discussion}

Although probability-weighted MMD performs well in these experiments, particularly because the simulated classifier is well calibrated, CC-MMD has two important conceptual advantages, in addition to empirical performance shown above.

First, CC-MMD is unbiased regardless of the audit and data sample sizes. The debiasing correction is constructed directly from the audited observations, yielding an unbiased estimator of the target MMD without requiring any assumptions about the quality of an auxiliary classifier. In contrast, the probability-weighted estimator is unbiased only if the estimated probabilities are well calibrated, a property that is often difficult to guarantee in practice. Our simulation intentionally uses a calibrated classifier, which explains the strong performance of the weighting approach.

Second, CC-MMD uses the audit sample more efficiently. Audited observations are typically the most valuable resource because obtaining verified labels is expensive. Using these samples solely to train a classifier and then weighting MMD uses information in a suboptimal way that could instead be used directly to correct the estimation objective. CC-MMD incorporates the audited observations explicitly into the estimator (analogous to the philosophy behind prediction-powered inference \cite{angelopoulos2023prediction}, where limited high-quality labels are used to debias an estimator rather than simply improve a prediction model). This direct use of audited data is particularly advantageous when learning an accurate classifier is itself difficult.

The experimental results above provide direct evidence of advantages of our method, CC-MMD, compared to the classifier-based alternatives. In particular, CC-MMD combines the efficiency of the full proxy-labeled sample with the bias correction provided by the audit.

The results are also consistent with our theory: the correction is design-unbiased for any $\tilde I$, so proxy quality affects the variance of the estimator rather than its validity, and at the audit budget used here the audited residuals alone already suffice. Any advantage over naive MMD, therefore, cannot arise from reusing the audit set.

\section{Computation details and reproducibility.}
\label{sec:computation}
All generative experiments are repeated over 5 random seeds (42--46), while the ablation studies are repeated over 20 random seeds per condition. Reported results show $\mu \pm \sigma$ across seeds. All experiments were implemented in PyTorch~2.6 with CUDA~12.4. Normalizing flows use the \texttt{nflows}~0.14 library, proxy classifiers use scikit-learn~1.7.2. Experiments were run on NVIDIA RTX~A4000 GPUs (16\,GB) and on high-performance computing cluster equipped with NVIDIA H100 GPUs.

\end{document}